\documentclass[accepted]{uai2024} 
                        
\usepackage[american]{babel}

\usepackage{natbib} 
\usepackage{mathtools} 
\usepackage{booktabs} 
\usepackage{tikz} 

\usepackage{amsmath}

\usepackage{graphicx}
\usepackage{epstopdf}
\usepackage{amssymb}
\usepackage{amsthm}
\usepackage{caption}
\usepackage{subcaption}
\usepackage[normalem]{ulem}
\DeclareGraphicsExtensions{.eps}

\usepackage{xcolor}

\usepackage{hyperref}
\hypersetup{
    colorlinks=true,
    linkcolor=blue,
    filecolor=blue,      
    urlcolor=blue,
    citecolor=blue,
}

\usepackage{algorithm}
\usepackage{algorithmic}
\usepackage{caption}

\newtheorem{theorem}{Theorem}
\newtheorem{lemma}[theorem]{Lemma}

\newtheorem{fact}[theorem]{Fact}

\usepackage{bm}

\title{Efficient and Adaptive Posterior Sampling Algorithms for Bandits}

\author[1]{Bingshan Hu}
\author[2]{Zhiming Huang}
\author[3]{Tianyue H. Zhang}
\author[1]{Mathias Lécuyer}
\author[4]{Nidhi Hegde}

\affil[1]{%
   University of British Columbia\\
    Vancouver, BC, Canada
}
\affil[2]{%
    University of Victoria\\
    Victoria, BC, Canada
}
\affil[3]{%
    Mila-Quebec AI Institute\\
    Université de Montréal\\
    Montreal, QC, Canada
  }
\affil[4]{%
    University of Alberta\\
    Alberta Machine Intelligence Institute (Amii)\\
    Edmonton, AB, Canada
  }
  
\begin{document}
\maketitle

\begin{abstract}



 We study Thompson Sampling-based algorithms for stochastic bandits with bounded rewards. As the existing problem-dependent regret bound for Thompson Sampling with Gaussian priors \citep{agrawalnear} is vacuous when $T \le 288 e^{64}$, 
we derive a  more practical   bound that tightens the coefficient of the leading term 
to $1270$. Additionally, motivated by large-scale real-world applications that require scalability, adaptive computational resource allocation, and a balance in  utility and computation, we propose two parameterized Thompson Sampling-based algorithms: Thompson Sampling with Model Aggregation (TS-MA-$\alpha$) and Thompson Sampling with Timestamp Duelling (TS-TD-$\alpha$), where $\alpha \in [0,1]$  controls the trade-off between utility and computation. Both algorithms achieve
$O \left(K\ln^{\alpha+1}(T)/\Delta \right)$ regret bound, where $K$ is the number of arms, $T$ is the finite learning horizon, and $\Delta$ denotes the single round performance loss when pulling  a sub-optimal arm. 


\end{abstract}

\section{Introduction}\label{sec:intro}

We study the learning problem of stochastic multi-armed bandits specified by $\left(K; p_1,  \dotsc, p_K \right)$ where $K$ is the number of arms and $p_i$ is the reward distribution of arm $i$ with its mean denoted by $\mu_i$.  
In this learning problem, a learner chooses an arm to pull in each round $t=1,2,\dotsc,T$ without knowledge of the reward distributions. At the end of the round, the learner obtains and observes a reward drawn from the reward distribution associated with the pulled arm. The goal of the learner is to pull arms sequentially to maximize the cumulative reward over  $T$ rounds. 
Since only the pulled arm is observed at the end of each round, the main challenge for solving bandit problems is to balance exploitation and exploration. Exploitation involves pulling arms expected to obtain high rewards based on past experience, whereas exploration involves pulling arms to help the learner better learn the reward distributions or the means of the reward distributions, to make more informed decisions in the future. 


Many real-world applications that require this exploitation-vs-exploration balance can be framed as bandit learning problems. For example, in healthcare, bandit algorithms can be used in clinical trials where exploitation is to prescribe known effective treatments and exploration is to prescribe new medicine to discover potentially more effective treatments. In inventory management, bandit algorithms can be used to decide whether to re-order products that sold well (exploitation) or to switch to new products  (exploration). In online advertising systems, bandit algorithms can help to decide whether to show users familiar content (exploitation) or to introduce users to new but potentially more interesting content (exploration) to optimize overall satisfaction.

There are two well-studied and empirically successful mechanisms to balance exploitation-vs-exploration in bandit problems.
The first approach relies on using the Upper Confidence Bound (UCB) of the estimate of each arm's mean reward \citep{auer2002finite,audibert2007tuning,garivier2011kl,kaufmann2012bayesian,lattimore2018refining}. In essence, exploration is thus driven by a bonus term added to the empirical estimate of the mean reward. The bonus depends on the number of observations collected for this arm.
The second approach is to add randomness in the estimate of each arm's mean reward \citep{agrawalnear,kaufmann2012thompson,jin2021mots, bian2022maillard,jin2022finite,jin2023thompson}. This can be viewed as injecting controllable data-dependent noise into the empirical estimate to ensure exploration.
In practice this randomness is added by, in each round, sampling each arm's mean reward estimate from a data-dependent distribution. 
When this data-dependent distribution is the posterior distribution of  arm's mean reward (typically given a simple prior and reward likelihood), this approach is called Thompson Sampling \citep{kaufmann2012thompson,agrawalnear}, one of the oldest Bayes-inspired randomized learning algorithms.

Thompson Sampling is appealing for practitioners, due to their good empirical performance,  and their intuitive exploration mechanism. However, they suffer from two shortcomings. First, the theoretical analysis of Thompson Sampling with Gaussian priors (the most widely applicable, as it applies to any  reward distribution with a bounded support) suffers from large constants, making its short learning horizon behavior hard to predict. Second, it requires a new posterior sample {\em for each arm, at each step}, which can be impractical when the number of arms is large ({\em e.g.}, in advertising or online recommendation).
This paper aims to address these shortcomings with two contributions.

\textbf{Overview of Contribution 1.} We revisit Thompson Sampling with Gaussian priors (Algorithm~2 in \cite{agrawalnear}). 
Denoting $\Delta_i$ the mean reward gap between a sub-optimal arm $i$ and the optimal arm, our new  bound is $\sum_{i: \Delta_i >0}  1270 \ln \left(T \Delta_i^2 + 100^{\frac{1}{3}}\right)/\Delta_i + O\left(1/\Delta_i\right)$, which
significantly improves the existing  $\sum_{i: \Delta_i >0}288  \left(e^{64}+6 \right) \ln \left(T\Delta_i^2 + e^{32} \right)/\Delta_i  + O\left(1/\Delta_i\right)$ problem-dependent bound \citep{agrawalnear}. The improved bound relies on a new technical Lemma~\ref{UBC 22}, which provides an upper bound on the concentration speed of the optimal arm's posterior distribution. Additionally, we  justify that Thompson Sampling is also inspired by the philosophy of optimism in the face of uncertainty (OFU). 



\textbf{Overview of Contribution 2.} 
Even when 
all arms have concentrated posterior distributions, Thompson Sampling-based algorithms \citep{kaufmann2012thompson,agrawalnear,bian2022maillard,jin2021mots,jin2022finite} still draw one random sample for each arm in each round. In total, they require $KT$ data-dependent random samples over $T$ rounds. 
This exploration mechanism becomes computationally prohibitive at scale, when the number of arms is very large ({\em e.g.}, in online advertising, recommendation systems). This problem is exacerbated when  sampling from the posterior reward distribution is computationally expensive, for example when it relies on Markov chain Monte Carlo (MCMC) methods. 
This computational cost can prevent the deployment of Thompson Sampling algorithms at scale despite their good empirical performance.




Motivated by the intuition that new samples from the posterior distribution cannot provide much new information when no new data was collected for a given arm, 
\emph{we propose two parameterized Thompson Sampling-based algorithms that simultaneously achieve good problem-dependent regret bounds and draw fewer data-dependent samples.} Both algorithms achieve a $\sum_{i : \Delta_i > 0} O \left(\ln^{\alpha+1}(T)/\Delta_i \right)$ sub-linear regret, but with the potential to draw fewer than $KT$ samples, where parameter $\alpha \in [0,1]$ controls the trade-off between utility (regret) and computation (total number of samples). 

Our first algorithm, Thompson Sampling with Model Aggregation (TS-MA-$\alpha$), is inspired by OFU and motivated by Gaussian anti-concentration bounds. The key idea of TS-MA-$\alpha$  is to draw a batch of $\phi=O \left(T^{0.5(1-\alpha)} \ln^{0.5(3-\alpha)}(T) \right)$ i.i.d. random samples for each arm at once rather than drawing an independent sample in each round. Then,  TS-MA-$\alpha$  commits to the best  sample among this batch  and only draws a new batch when the data-dependent distribution of the given arm changes. 
TS-MA-$\alpha$ only  draws $T \phi$ data-dependent samples and uses $\alpha \in [0,1]$ to control the trade-off between utility (regret) and computation (total number of drawn samples).\footnote{TS-MA-$\alpha$ only needs to draw $T$ random samples if being implemented efficiently. More details can be found in Section~\ref{sec: TS-MA}.}
The key advantage of TS-MA-$\alpha$ is that the total number of drawn samples does not depend on the number of arms $K$.
Since the total amount of data-dependent samples does not depend on $K$, TS-MA-$\alpha$ is extremely efficient when the number of arms is  very large. However,  when the number of arms is small, it may not save on computation.  



Our second algorithm, Thompson Sampling with Timestamp Duelling (TS-TD-$\alpha$), is an adaptive switching algorithm between Thompson Sampling and  TS-MA-$\alpha$. 
Intuitively, we would like to draw data-dependent samples for the optimal arms to preserve good practical performance, and save on data-dependent samples for the sub-optimal arms that are not frequently pulled and hence see little new data.
The key idea of TS-TD-$\alpha$ is to aggregate up to $\phi$ samples for each arm over time instead of drawing them all at once. 
Given arm $i$, there are thus two phases in the rounds between two consecutive pulls (the updates of posterior distributions).
In the first phase, of up to $\phi$ rounds, TS-TD-$\alpha$ draws a data-dependent sample for arm $i$ (running Thompson Sampling). 
The second phase starts after $\phi$ rounds if arm $i$ wasn't pulled, and TS-TD-$\alpha$ uses the maximum of the $\phi$ samples accumulated in the first phase (running TS-MA-$\alpha$).
Similarly to TS-MA-$\alpha$, TS-TD-$\alpha$ also realizes a $\sum_{i : \Delta_i > 0} O \left(\ln^{\alpha+1}(T)/\Delta_i \right)$ regret bound, but only draws $\min \left\{O \left( \phi\ln^{1+\alpha}(T) /\Delta_i^2 \right), T \phi \right\}$ data-dependent samples in expectation for a sub-optimal arm $i$. The key advantage of TS-TD-$\alpha$ is that it is extremely efficient when the learning problem has many sub-optimal arms, as it stops drawing data-dependent samples for the sub-optimal arms.


 
 
  
 



\section{Learning Problem }
As a general rule: in a stochastic multi-arm bandit (MAB) problem, we have an arm set $\mathcal{A}$ of size $K$ and each arm $i \in \mathcal{A}$ is associated with an unknown reward distribution $p_i$ 
with $[0,1]$ support. At the beginning of each round $t$,  a  random reward vector $X(t) := \left(  X_1(t),X_2(t), \dotsc, X_K(t) \right)$ is generated, where each $X_{i} (t) \sim p_i $. Simultaneously, the learner pulls an arm  $i_t \in \mathcal{A}$. At the end of the round, the learner observes and obtains $X_{i_t}(t)$, the reward of the pulled arm. Let $\mu_i $ denote the mean of reward distribution $p_i$. Without loss of generality, we assume the first arm is the 
unique optimal arm, that is, $\mu_1 > \mu_i$ for all $i \ne 1$. Let $\Delta_i:= \mu_1 - \mu_i$ denote the mean reward gap, which indicates the single round performance loss when pulling arm $i$ instead of the optimal arm.
The goal of the learner is to pull arms sequentially to 
minimize \emph{pseudo-regret} $\mathcal{R} \left( T \right) $, expressed as
\[
T  \mu_1 - \mathbb{E} \left[   \sum_{t=1}^{T}  X_{i_t}(t) \right] = \sum_{i}    \sum_{t=1}^{T} \mathbb{E} \left[ \bm{1} \left\{i_t =i \right\} \right] \Delta_i ,
\]
where the expectation is taken over $i_1, i_2, \dotsc, i_T$.

\section{Related Work}

There is a vast literature on  MAB algorithms, which we split based on deterministic versus randomized exploration.

{\bf Deterministic exploration} algorithms (the pulled arms are determined by the history data) are predominantly UCB-based \citep{auer2002finite,audibert2007tuning,garivier2011kl,kaufmann2012bayesian,lattimore2018refining}. They rely on the OFU principle, and leverage well-known concentration inequalities to construct a confidence interval for the empirical estimate of each arm. They then use the upper bound of the confidence interval to decide which arm to pull.
There also exists non-OFU deterministic algorithms, such as 
DMED \citep{honda2010asymptotically}, IMED  \citep{honda2015non}, and an elimination-style algorithm by \citet{auer2010ucb}.

{\bf Randomized exploration} algorithms often rely on Thompson Sampling (TS)
\citep{agrawalnear,kaufmann2012thompson,bian2022maillard,jin2021mots,jin2022finite,jin2023thompson}.
The key idea of TS  algorithms is to maintain a posterior on each arm's mean reward updated each time new data is acquired, and use a random sample drawn from this data-dependent distribution to decide which arm to pull.
There also are non TS-based MAB algorithms, such as Non-parametric TS \citep{riou2020bandit} and Generic Dirichlet Sampling \citep{baudry2021optimality}.\footnote{In some work, all  algorithms  with randomized exploration are categorized into Thompson Sampling-based regardless of whether they use data-dependent distributions to model the mean rewards or not. In this work, we only say a  randomized algorithm is Thompson Sampling-based if data-dependent distributions are used to model the mean rewards. } 
%
All  aforementioned  algorithms achieve (at least) one of three notions of optimality:
asymptotic optimality (i.e., learning horizon $T$ is infinite), worst-case optimality (i.e., minimax optimality), and sub-UCB criteria (problem-dependent optimality with a finite horizon $T$) \citep{lattimore2018refining}. This work focuses on deriving problem-dependent regret bounds with a finite learning horizon $T$ to study the sub-UCB property.
Problem-dependent regret bounds with a finite learning horizon $T$ can take either the $\sum_{i : \Delta_i>0} \frac{C \ln(T)}{\Delta_i}$ or $\sum_{i : \Delta_i>0} \frac{C \ln(T)\Delta_i}{\text{KL}\left(\mu_i, \mu_i+ \Delta_i \right)}$ form for the term involving  $T$, where $C \ge 1$ is a universal constant and $\text{KL}(a,b)$ denotes the KL-divergence between two Bernoulli distributions with parameters $a,b \in (0,1)$. 



For stochastic bandits with $[0,1]$ rewards, there are two original versions of TS using different conjugate pairs for the prior distribution and reward likelihood function. 
The first version is  TS with Beta priors \citep{kaufmann2012thompson,agrawalnear}, using a uniform $\text{Beta}(1,1)$ prior and Bernoulli likelihood function. 
\citet{kaufmann2012thompson,agrawalnear} both present a $\sum_{i : \Delta_i>0} \frac{(1+\varepsilon)\ln(T)\Delta_i}{\text{KL}\left(\mu_i, \mu_i+ \Delta_i \right)} +C(1/\varepsilon, \mu)$ problem-dependent regret bound, where $\varepsilon>0$ can be arbitrarily small and  $C(1/\varepsilon, \mu)$ is a problem-dependent constant and goes to infinity if $\varepsilon \rightarrow 0$. 
The second version is TS with Gaussian priors \citep{agrawalnear}, using an  improper Gaussian prior $\mathcal{N}\left(0, \infty \right)$ and a Gaussian likelihood function. 
\citet{agrawalnear} derive a $ \sum_{i: \Delta_i >0}288  \left(e^{64}+6 \right) \ln \left(T\Delta_i^2 + e^{32}\right) /\Delta_i   + O\left(1/\Delta_i\right)$ problem-dependent regret bound.  Since the coefficient for the $\ln(T)$ leading term is at least  $288 e^{64} \approx 1.8 \times 10^{30} $, this regret bound is vacuous when $T \le 288  e^{64}$. Note that when $T \le 288 e^{64}$, the regret is at most $T$, and thus, the existing bound does not take the bandit instance $\left(K; \mu_1, \mu_2, \dotsc, \mu_K \right)$ into account. 
In this work, we derive a new problem-dependent regret bound with a more acceptable coefficient for the leading term. Our Theorem~\ref{Theorem: TS new 1} improves the coefficient for the leading term to $1270$. Also, we justify that TS is also an OFU-inspired algorithm. 


%
Very recently, \citet{jin2023thompson} proposed $\epsilon$-Exploring Thompson Sampling ($\epsilon$-TS), a hybrid algorithm simply blending Thompson Sampling with pure exploitation for some specific reward distributions in the exponential family including Bernoulli, Gaussian, Poisson, and Gamma distributions to save on sampling computation. The core idea of $\epsilon$-TS
is, for each individual arm, to apply the normal TS procedure---sample from 
posterior 
---with probability $\epsilon \in [1/K,1]$, and use 
the empirical mean 
with probability $1-\epsilon$. It is not hard to see that the expected number of drawn posterior samples is $\epsilon KT$. 
They derive an $O \left(\sqrt{KT\ln(eK\epsilon)} \right)$ worst-case regret bound if we ignore some problem-dependent parameters, and prove that $\epsilon$-TS achieves asymptotic optimality. 



The limitation of $\epsilon$-TS is that the derived bounds may not work for other reward distributions. The theoretical analysis of $\epsilon$-TS relies on knowing the true reward distribution family to use the proper conjugate pair. Hence, contrary to  TS with Gaussian priors \citep{agrawalnear}, $\epsilon$-TS  (at least its performance guarantees) may not apply to other reward distributions.
Further, $\epsilon$-TS is not adaptive and efficient in using sampling resources in the common case of a large number of sub-optimal arms. This is because $\epsilon$-TS simply flips a biased coin to decide whether to draw a posterior sample for each arm or not.  It does not take the underlying bandit problem into account. For example, given the same amount of arms, the expected number of posterior samples stays the same whether the bandit learning problem has one optimal arm or $K-1$ optimal arms. We provide a detailed discussion of $\epsilon$-TS in Section~\ref{sec: ALG}.









To use sampling resources more efficiently and adaptively, we propose two TS-based algorithms, TS-MA-$\alpha$ and TS-TD-$\alpha$, that allocate sampling resources based on whether exploration is necessary. The intuition is that once sub-optimal arm's posterior has concentrated enough, drawing new posterior samples is unnecessary.
Both of our proposed algorithms achieve an $O \left(K \ln^{1+\alpha}(T)/\Delta \right)$ problem-dependent regret bound 
, where $\alpha \in [0,1]$ is a parameter controlling the trade-off between utility (regret) and computation (number of drawn posterior samples).  

\section{Algorithms} \label{sec: ALG}

We first present notations specific to this section. 
Let $n_i(t-1) := \sum_{q=1}^{t-1} \bm{1} \left\{i_q = i \right\}$ denote the number of pulls of arm $i$ by the end of round $t-1$.
Let $\hat{\mu}_{i, n_i(t-1)}(t-1) := \frac{1}{n_i(t-1)}\sum_{q=1}^{t-1} \bm{1} \left\{i_q = i \right\} X_{i}(q)$ denote the empirical mean of arm $i$ by the end of round $t-1$.  We write it as $\hat{\mu}_{i, n_i(t-1)}$ for short. Let $\mathcal{N} \left(\mu,  \sigma^2\right)$ denote a Gaussian distribution with mean $\mu$ and variance $\sigma^2$. 
Let  $c_0 := 1/(2\sqrt{2e\pi})$. 

\subsection{Vanilla Thompson Sampling}\label{sec: TS}


We first revisit Thompson Sampling with Gaussian priors (Algorithm~2 in \cite{agrawalnear}). For ease of presentation, we rename it as Vanilla Thompson Sampling in this work. 
The idea of Vanilla Thompson Sampling is to draw a posterior sample $\theta_{i}(t) \sim \mathcal{N} \left(\hat{\mu}_{i, n_i(t-1)}, \ 1/n_i(t-1)\right)$ for each arm $i$ in each round $t$.  Then, the learner pulls the arm with the highest posterior sample value, that is,  $i_t = \mathop{\arg\max}_{i \in \mathcal{A}}\theta_{i}(t)$. 

Now, we present our new bounds for it.
\begin{theorem}
   The problem-dependent regret of Algorithm~\ref{Alg: TS} is 
$\sum_{i: \Delta_i >0}  1270 \ln \left(T \Delta_i^2 + 100^{\frac{1}{3}}\right)/\Delta_i + 182.5/\Delta_i+ \Delta_i$. The worst-case regret is  $O \left( \sqrt{KT \ln(K)} \right)$, where $O(\cdot)$ only hides a universal constant. 
    \label{Theorem: TS new 1}
\end{theorem}


\textbf{Comparison with bounds in \cite{agrawalnear}.}

Note that Theorems~1.3 and 1.4 in \cite{agrawalnear}  only presented an $O \left(\sqrt{KT\ln(K)} \right)$  worst-case regret upper bound and a  matching $\Omega\left(\sqrt{KT\ln(K)} \right)$  regret lower bound for Vanilla Thompson Sampling. They did not state any problem-dependent results in their theorems. 
We compare our new results with  a problem-dependent regret bound embedded in their proof for Theorem~1.3. In page A.18 in \cite{agrawalnear}, they derived a $ \sum_{i: \Delta_i >0}288  \left(e^{64}+6 \right) \ln \left(T\Delta_i^2 + e^{32}\right) /\Delta_i   + O\left(1/\Delta_i\right)$ regret bound. 
Note that  our improved problem-dependent regret bound  implies an improved worst-case regret bound, i.e., our $O(\cdot)$ hides a smaller constant  than theirs.


\textbf{Comparison with $\epsilon$-TS.} Note that although $\epsilon$-TS reduces to Thompson Sampling when setting $\epsilon=1$,  the authors do not revisit and derive new bounds for Vanilla Thompson Sampling  as $\epsilon$-TS only works for reward distributions that have conjugate priors. As presented in Table~1 in \cite{jin2023thompson}, $\epsilon$-TS using Gaussian priors only works for Gaussian rewards, whereas Vanilla Thompson Sampling works for any bounded reward distributions. Vanilla Thompson Sampling simply uses Gaussian distributions to model the means of the underlying reward distributions.  The true reward distribution and Gaussian prior does not need to form a conjugate pair in \cite{agrawalnear}. 


The improvement of our problem-dependent regret bound mainly comes from Lemma~\ref{UBC 22} below, which improves the results shown in Lemma~2.13 in \cite{agrawalnear} significantly.
Let $\mathcal{F}_{t-1} = \left\{i_q, X_{i_q}(q),q = 1,2,\dotsc,t-1 \right\}$ collect all the history information by the end of round $t-1$. Let $L_{1,i} := 4\left(\sqrt{2}+\sqrt{3.5}\right)^2 \ln \left(T \Delta_i^2 + 100^{\frac{1}{3}}\right)/\Delta_i^2$.

\begin{lemma}
Let $\tau_s^{(1)}$ be the round when the $s$-th pull of the optimal arm $1$ occurs and $\theta_{1,s} \sim \mathcal{N}\left(\hat{\mu}_{1,s}, \frac{1}{s}\right)$. Then, we have 

$\mathbb{E}\left[\frac{1}{\mathbb{P} \left\{ \theta_{1,s} > \mu_1 - \frac{\Delta_i}{2}\mid \mathcal{F}_{\tau_{s}^{(1)}}  \right\} }   \right]-1  = \begin{cases}
 29,  & \forall s \ge 1,\\
 \frac{180}{T\Delta_i^2}, & \forall s \ge L_{1,i}. 
\end{cases}$.
\label{UBC 22}
\end{lemma}
\textbf{Discussion of Lemma~\ref{UBC 22}.} 
If event $\theta_{1,s} > \mu_1 - \Delta_i/2$ occurs, we can view the optimal arm as having a good posterior sample as compared to  sub-optimal arm $i$. Informally, a good posterior sample of the optimal arm indicates the pull of the optimal arm. The pulls of the optimal arm contribute to reducing the variance of the posterior distribution and the width of the confidence intervals.
The value of 
$\mathbb{E}\left[1/\mathbb{P} \left\{ \theta_{1,s} > \mu_1 - \frac{\Delta_i}{2}\mid \mathcal{F}_{\tau_{s}^{(1)}}  \right\} \right]-1$ quantifies the expected number of independent draws needed before the optimal arm has a good posterior sample. The first case in Lemma~\ref{UBC 22} says that even if the optimal arm has not been observed enough, i.e., $s < L_{1,i}$, it takes at most $29$ draws in expectation before the optimal arm has a good posterior sample. The second case says that after the optimal arm has been observed enough, i.e., $s \ge L_{1,i}$, the expected number of draws before the optimal arm has a good posterior sample is in the order of $\frac{1}{T\Delta_i^2}$. That is also to say, after the optimal arm has been observed enough, it will be pulled very frequently, which preserves the excellent practical performance of Thompson Sampling.
If the learning algorithm draws one random sample in each round, this quantity reflects the concentration speed of the optimal arm's posterior distribution. 
In Lemma~2.13 of \cite{agrawalnear}, the derived upper bound on this quantity is $e^{64} +5$ when $s \ge 1$. 
Our new bound is $29$.

Lemma~\ref{UBC 22} also shows that Thompson Sampling follows the principle of OFU, as do  UCB-based algorithms. 
Note that in each round, the learning algorithm either pulls the optimal arm or a sub-optimal arm. If the optimal arm is pulled, we do not suffer any performance loss and gain information for the optimal arm. The pull of the optimal arm contributes to being closer to having enough observations of the optimal arm and reducing the variance of the posterior distribution. If the learning algorithm pulls a sub-optimal arm, we attain information from the following two perspectives. First, the pull of the sub-optimal arm contributes to shrinking the width of the confidence interval and reducing the variance of the posterior distribution of the pulled sub-optimal arm. Second, the draw of the posterior sample for the optimal arm still provides some information. From Lemma~\ref{UBC 22} we know that even if the optimal arm has not been observed enough yet, it takes at most a constant number of i.i.d. draws in expectation before the optimal arm is pulled again.


\begin{algorithm}[ht]
	\caption{Vanilla Thompson Sampling (Algorithm~2 in \cite{agrawalnear})}
	\label{Alg: TS}
	
	\begin{algorithmic}[1]
 	\STATE \textbf{Ini:} 
 Pull each arm $i$ once to initialize $n_i$ and $\hat{\mu}_{i, n_i}$ 

		\FOR {$t = K+1,K+2,\cdots$}
		\STATE Draw $\theta_i(t) \sim \mathcal{N} \left(\hat{\mu}_{i,n_i}, 1/n_i \right)$ for all $i \in \mathcal{A}$ 
		\STATE Pull arm $i_t \in \arg \mathop {\max }_{i \in \mathcal{A}} \theta_i(t)$
	and  observe  $X_{i_t}(t)$	  
  \STATE Update  $n_{i_t}$ and $\hat{\mu}_{i_t, n_{i_t}}$. 
		\ENDFOR
	\end{algorithmic}
\end{algorithm}

\subsection{TS-MA-$\alpha$}\label{sec: TS-MA}

Since Vanilla Thompson Sampling draws a random sample for each arm in each round, the total number of posterior samples is $KT$ and the percentage of posterior samples for the optimal arms is exactly $m/K$ if we have $m$ optimal arms. Inspired by the facts that Gaussian distributions have nice anti-concentration bounds (Fact~\ref{Fact 1}) and exploration is unnecessary once sub-optimal arms have been pulled enough,  
we propose TS-MA-$\alpha$ (Algorithm~\ref{ALPHA: Thompson Sampling with Model Aggregation}) 
to save on some sampling computation and allocate computational resources efficiently.
TS-MA-$\alpha$ can be viewed as a randomized and parameterized version of UCB1 of \cite{auer2002finite} with parameter $\alpha$ controlling the trade-off between utility  (regret) and computation (number of drawn posterior samples). 

TS-MA-$\alpha$ draws   $\phi = O \left(T^{0.5(1-\alpha)} \ln^{0.5(3-\alpha)}(T) \right)$ i.i.d. random samples at once rather than drawing a fresh sample in each round. Let $\theta_i(t) := \mathop{\max}_{h \in [\phi]}\theta_{i, n_i(t-1)}^{h}$, where $\theta_{i, n_i(t-1)}^{h} \sim \mathcal{N} \left(\hat{\mu}_{i,n_i(t-1)}, \ \ln^{\alpha}(T)/n_i(t-1)\right)$, be the maximum among all these $\phi$ random samples. Then, the learner uses the best sample $\theta_i(t) $ in the learning  and 
pulls the arm with the highest $\theta_i(t)$, i.e., $i_t = \mathop{\arg\max}\theta_i(t)$. 
At the end of round $t$, the learner updates the statistics of the pulled arm $i_t$.  Since the posterior distribution of arm $i_t$ changes, now, the learner draws a new batch of $\phi$ i.i.d. random samples for arm $i_t$ according to the updated posterior distribution. Note that all $\theta_i(t)$ are analogous to the upper confidence bounds in UCB-based algorithms. 
The term model aggregation refers to the fact that we aggregate  $\phi$ posterior samples into one by putting all weight over the best posterior sample.

\begin{algorithm}[ht]
	\caption{TS-MA-$\alpha$}
	\label{ALPHA: Thompson Sampling with Model Aggregation}
	
	\begin{algorithmic}[1]
 \STATE \textbf{Input:} Time horizon $T$ and  parameter $\alpha \in [0,1]$
		\STATE \textbf{Initialization:} Set $\phi =  2 T^{0.5(1-\alpha)} \ln^{0.5(3-\alpha)}(T)/c_0 $; For each arm $i$, pull it once to initialize $n_i$ and $ \hat{\mu}_{i, n_i}$, set $\theta_i \leftarrow \mathop{\max}_{h \in [\phi]} \theta_{i,n_i}^{h} $, where each $\theta_{i,n_i}^{h} \sim \mathcal{N}\left(\hat{\mu}_{i,n_i},  \ \ln^{\alpha}(T)/n_i  \right)$ 
  
		\FOR {$t = K+1,K+2,\cdots,T$}
		 
		\STATE Pull arm $i_t \in \arg \mathop {\max }_{i \in \mathcal{A}} \theta_i$ and observe  $X_{i_t}(t)$	  
  \STATE Update $n_{i_t}$, $\hat{\mu}_{i_t, n_{i_t}}$, and $\theta_{i_t} \leftarrow \mathop{\max}_{h \in [\phi]} \theta_{i_t,n_{i_t}}^{h} $, where each $\theta_{i_t,n_{i_t}}^{h} \sim \mathcal{N}\left( \hat{\mu}_{i_t, n_{i_t}}, \ \ln^{\alpha}(T)/n_{i_t} \right)$. \label{ALG: Efficient}
		\ENDFOR
	\end{algorithmic}
\end{algorithm}

Now, we present regret bounds for TS-MA-$\alpha$. 

\begin{theorem}
    The problem-dependent regret  of Algorithm~\ref{ALPHA: Thompson Sampling with Model Aggregation} is  $\sum_{i: \Delta_i > 0} O \left(\ln^{\alpha+1}(T)/\Delta_i \right)$, where $\alpha \in [0,1]$. The worst-case regret  is $O \left(\sqrt{KT\ln^{\alpha+1}(T)} \right)$.
    \label{Theorem: Alpha-TS}
\end{theorem}




\textbf{Discussion of TS-MA-$\alpha$  and comparison with $\epsilon$-TS.}
At first glance, Algorithm~\ref{ALPHA: Thompson Sampling with Model Aggregation} seems to draw $T\phi = O \left( (T \ln(T))^{0.5(3-\alpha)} \right)$ random samples, but it can be implemented  only drawing $T$ random samples. In Line~\ref{ALG: Efficient} in Algorithm~\ref{ALPHA: Thompson Sampling with Model Aggregation}, we can directly draw  $\theta_{i_t}$ from  the distribution of the maximum of $\phi$ i.i.d.  Gaussian random variables. TS-MA-$\alpha$ is scalable as the total number of drawn random samples does not depend on $K$.  
When setting $\alpha=1$, TS-MA-$\alpha$ achieves a sub-linear $O \left(K\ln^{2}(T)/\Delta \right)$ regret, but only draws $O(T\ln(T))$ random samples in total even without using the efficient implementation. TS-MA-$\alpha$ is quite efficient when $K \gg O(\ln (T))$. When setting $\alpha =0$, TS-MA-$\alpha$  achieves the optimal $O\left( K\ln(T)/\Delta\right)$ problem-dependent regret. In addition, TS-MA-$\alpha$ is minimax optimal up to an extra $\sqrt{\ln^{\alpha+1}(T)}$ factor. 
Furthermore, TS-MA-$\alpha$ 
 works for any reward distribution with bounded support, whereas  $\epsilon$-TS only works for some special reward distributions including Gaussian, Bernoulli, Poisson, and Gamma. 
As compared to $\epsilon$-TS, our proposed TS-MA-$\alpha$ saves on drawing random samples from sub-optimal arms as the expected number of random samples drawn for a sub-optimal arm $i$  is at most $O \left( \phi\ln^{1+\alpha}(T) /\Delta_i^2 \right)$, whereas $\epsilon$-TS draws exactly $\epsilon T$ random samples in expectation for any sub-optimal arm.

Now, we compare our theoretical results with $\epsilon$-TS for the Bernoulli  reward case as Bernoulli distributions have bounded support.
 Theorem~2.1 of \cite{jin2023thompson}   states that the worst-case regret of $\epsilon$-TS is  $O\left( \sqrt{KT\log(eK\epsilon)}\right)$ and $\epsilon$-TS is asymptotically optimal. Since they do not present any finite-time problem-dependent results, we compare TS-MA-$\alpha$ to some inferred problem-dependent regret bounds based on their regret analysis. 
From their theoretical analysis, the optimal $O \left(\sum_{i: \Delta_i >0} \frac{\ln(T \Delta_i^2)}{ \Delta_i} \right)$ problem-dependent regret bound can be inferred by tuning $\epsilon$ as a constant.
However, this tuning  sacrifices the worst-case regret guarantee as  $\epsilon$-TS is only worst-case (minimax) optimal when setting $\epsilon = O(1/K)$. In addition, this tuning results in   $O(KT)$ posterior samples in total, which may not be more efficient than TS-MA-$\alpha$.

\subsection{TS-TD-$\alpha$}\label{sec: TS-TD}
\begin{algorithm}[ht]
	\caption{TS-TD-$\alpha$}
	\label{Thompson Sampling with Time-stamp Duelling-Alpha}
	
	\begin{algorithmic}[1]
  \STATE \textbf{Input:}  Time horizon $T$ and  parameter $\alpha \in [0,1]$ 
		\STATE \textbf{Initialization:} 
  Set $\phi = 2 T^{0.5(1-\alpha)} \ln^{0.5(3-\alpha)}(T) /c_0$; For each arm $i $,
pull it once to initialize $n_i$ and $ \hat{\mu}_{i, n_i}$,
set used sample counter $h_i \leftarrow 0$ and
   $\psi_i \leftarrow 0$ 
		\FOR {$t =K+ 1,K+2,\cdots,T$}
\FOR {$i \in \mathcal{A}$}

  \IF { $h_i \le \phi-1$} 
  
  \STATE Draw $\theta_{i}(t)  \sim \mathcal{N}\left(\hat{\mu}_{i,n_i}, \frac{\ln^{\alpha}(T)}{n_{i}} \right)$  {\color{blue}\%\text{Phase I}} \label{phase: TS}
  
  Set  $h_i \leftarrow h_i + 1 $ and $\psi_i \leftarrow \max\{\psi_i, \theta_{i}(t) \}$ 

\ELSE
\STATE Set $\theta_{i}(t) \leftarrow \psi_i$ {\color{blue}\%\text{Phase II}}
\label{phase: commitment}

  \ENDIF

  \ENDFOR
		 
		\STATE Pull arm $i_t \in \arg \mathop {\max }_{i \in \mathcal{A}} \theta_i(t)$
	and observe $X_{i_t}(t)$	  
  \STATE Update $n_{i_t}$, $\hat{\mu}_{i_t,n_{i_t}}$; Reset $h_{i_t} \leftarrow 0 $ and $\psi_{i_t} \leftarrow 0$.
		\ENDFOR
	\end{algorithmic}
\end{algorithm}
Since  TS-MA-$\alpha$ (Algorithm~\ref{ALPHA: Thompson Sampling with Model Aggregation})   draws   $T\phi$ random samples if not being implemented efficiently,
it may not save on posterior sampling computation as compared to Vanilla Thompson Sampling (Algorithm~\ref{Alg: TS}) for some  $\alpha,K$ and $T$. To further reduce computation, we propose TS-TD-$\alpha$ (Algorithm~\ref{Thompson Sampling with Time-stamp Duelling-Alpha}), an adaptive switching algorithm between  Algorithm~\ref{Alg: TS} and Algorithm~\ref{ALPHA: Thompson Sampling with Model Aggregation} with minor modifications.
The intuition behind TS-TD-$\alpha$ is  to keep drawing random samples for optimal arms, whereas saving on sampling for sub-optimal arms. The key reason for introducing Thompson Sampling (drawing posterior samples in each round) is, as already discussed in Section~\ref{sec: TS}, optimal arms need certain number of draws  to have concentrated posterior distributions. On the other hand, we would like to avoid drawing random samples for the sub-optimal arms. 
As already discussed in Section~\ref{sec: TS-MA}, once sub-optimal arms have been observed enough, it is not necessary to draw posterior samples to further explore them. 




TS-TD-$\alpha$ allocates $\phi=O \left(T^{0.5(1-\alpha)} \ln^{0.5(3-\alpha)}(T) \right)$  sampling budget for each pull of arm $i$ and separates all rounds between two consecutive pulls of arm $i$ into two phases.  Note that the posterior distributions for all the rounds between the two pulls stay the same as there is no new observation from arm $i$.
Drawing posterior samples is only allowed in the first phase (Vanilla Thompson Sampling), whereas the best posterior sample in the first phase will be used among all rounds in the second phase (TS-MA-$\alpha$). Note that the second phase does not draw any new posterior samples for the given arm $i$.
 Since an optimal arm is pulled very frequently,  the number of rounds between two consecutive pulls is unlikely to be greater than $\phi$, that is, optimal arms mostly stay in the first phase (Line~\ref{phase: TS}). On the other hand, sub-optimal arms are unlikely to be pulled after they have been pulled sufficiently, that is,  sub-optimal arms mostly stay in the second phase (Line~\ref{phase: commitment}).  

Now, we present  regret bounds for TS-TD-$\alpha$.
\begin{theorem}
    The problem-dependent regret  of Algorithm~\ref{Thompson Sampling with Time-stamp Duelling-Alpha} is
 $\sum_{i: \Delta_i > 0} O \left(\ln^{\alpha+1}(T)/\Delta_i\right)$, where $\alpha \in [0,1]$.  The worst-case regret  is $O \left(\sqrt{KT\ln^{\alpha+1}(T)} \right)$.
\label{Theorem: Main 3}
\end{theorem}

  \textbf{Discussion of TS-TD-$\alpha$ and comparison with $\epsilon$-TS.}
When $T$ is large enough, 
for a sub-optimal arm $i$, TS-TD-$\alpha$  draws $O \left(T^{0.5(1-\alpha)}\ln^{0.5(\alpha+5)}(T)/\Delta_i^2\right)$ posterior samples in expectation. When setting $\alpha =0$, TS-TD-$\alpha$ draws  $O \left(\sqrt{T}\ln^{2.5}(T)/\Delta_i^2\right)$ posterior samples and achieves the optimal $O\left(K\ln(T)/\Delta \right)$ regret.
When setting $\alpha = 1$, TS-TD-$\alpha$  draws  $O \left(\ln^{3}(T)/\Delta_i^2\right)$ posterior samples and achieves an $O\left(K\ln^2(T)/\Delta \right)$ regret.  Still, parameter $\alpha$ is used to control the trade-off between utility  and sampling computation. The expected number of posterior samples in TS-TD-$\alpha$ is $\sum_{i: \Delta_i >0} O \left(T^{0.5(1-\alpha)}\ln^{0.5(\alpha+5)}(T)/\Delta_i^2\right) + mT$, where $m$ is the total number of optimal arms. When setting a reasonable $\alpha$, for example, $\alpha > 0.5$, the total number of drawn posterior samples in TS-TD-$\alpha$ for the sub-optimal arms is negligible. In contrast, $\epsilon$-TS draws $\epsilon KT$ posterior samples in expectation.

Although  TS-TD-$\alpha$ and $\epsilon$-TS are both hybrid learning algorithms, fundamentally, they are very different in allocating sampling resources and TS-TD-$\alpha$ is extremely efficient when the number of sub-optimal arms is large. For each arm regardless of being sub-optimal or optimal, $\epsilon$-TS simply flips a biased coin with parameter $\epsilon$ to  decide  whether to draw a posterior sample.
It does not consider the situation whether the given arm needs exploration or not. 
In a sharp contrast, TS-TD-$\alpha$ decides whether to draw posterior samples based on the need for exploration.
It draws posterior samples more frequently  for good arms, whereas avoiding drawing posterior samples for sub-optimal arms. 
The advantage of this sophisticated hybrid approach is  more sampling resources are allocated to the optimal arms. 









    

\section{Experimental Results} \label{sec: experiments}


In this section, we conduct experiments to compare our TS-MA-$\alpha$ and TS-TD-$\alpha$ with the following state-of-the-art algorithms which pull arm $i_t \in \arg \max _{i \in \mathcal{A}} a_i(t)$, where $a_i(t)$ is the index defined as follows for those algorithms. Note that except Vanilla TS (Gaussian priors) and $\epsilon$-TS (Gaussian priors), all the other compared algorithms are proved to be asymptotically optimal, i.e., achieving a $\sum_{i: \Delta_i >0}  \frac{\Delta_i\ln (T)}{\operatorname{KL}\left(\mu_i, \mu_1\right)} + o(\ln(T))$ regret  when $T$ is sufficiently large.
We would like to emphasize that the goal of this work is not to show that our proposed algorithms can empirically perform  better than the state-of-the-art algorithms. Instead, we would like to show that  designing computationally efficient algorithms from another angles is possible.

\begin{itemize}
 
    \item Vanilla TS (Gaussian priors) (Algorithm~\ref{Alg: TS}):  $a_i(t) \sim \mathcal{N} \left(\hat{\mu}_{i, n_i}, 1/n_i \right)$.
     \item Vanilla TS (Beta priors) \citep{agrawalnear}:  $a_i(t) \sim \text{Beta}\left(\hat{\mu}_{i, n_i} \cdot n_{i} + 1, (1 - \hat{\mu}_{i, n_i}) \cdot n_{i} + 1\right)$.
    \item $\epsilon$-TS (Gaussian priors)~\citep{jin2023thompson}: with probability $\epsilon$, draw $a_i(t) \sim \mathcal{N} \left(\hat{\mu}_{i, n_i}, 1/n_i \right)$, whereas  with probability $1-\epsilon$, use $a_i(t) = \hat{\mu}_{i, n_i}$. Since $\epsilon $ can be any value in $[1/K, 1]$, we choose $\epsilon = \left\{1/K, 1/\sqrt{K},1 \right\}$. Note that for the choice of $\epsilon = 1$, $\epsilon$-TS (Gaussian priors) boils down to Vanilla TS (Gaussian priors).
       \item $\epsilon$-TS (Beta priors)~\citep{jin2023thompson}: with probability $\epsilon$, 
       
       draw $a_i(t) \sim  \text{Beta}\left(\hat{\mu}_{i, n_i} \cdot n_{i} + 1, (1 - \hat{\mu}_{i, n_i}) \cdot  n_{i} + 1\right)$, whereas  with probability $1-\epsilon$, use $a_i(t) = \hat{\mu}_{i, n_i}$. 
    \item $\text {ExpTS }^{+}$~\citep{jin2022finite}: with probability $1/K$, draw $a_i(t) \sim P\left(\hat{\mu}_{i, n_i}, n_i\right)$, where $P\left(\mu, n_i\right)$ is a sampling distribution with  CDF defined as:$$F(x)= \begin{cases}1-1 / 2 e^{-(n_i-1) \cdot \mathrm{KL}(a, x)}, & x \geq a, \\ 1 / 2 e^{-(n_i-1) \cdot \mathrm{KL}(a, x)}, & x \leq a ,\end{cases}$$ 
    whereas with probability $1-1/K$, use $a_i(t) = \hat{\mu}_{i, n_i}$.

    \item KL-UCB++~\citep{menard2017minimax}: $a_i(t)=  \begin{aligned}
 &\sup \left\{a:  \mathrm{KL}\left(\hat{\mu}_{i, n_i}, a\right) \leq \frac{\overline{\ln}\left(\frac{T\left(\overline{\ln}^2\left(\frac{T}{K n_i}\right)+1\right)}{K n_i}\right)}{n_i}\right\},
\end{aligned}$ where  $\overline{\ln}(\cdot) := \max\{0, \cdot\}$.
\item Similar to \cite{jin2023thompson}, we also plot the asymptotic 
$\sum_{i: \Delta_i >0}  \frac{\Delta_i\ln (T)}{\operatorname{KL}\left(\mu_i, \mu_1\right)} + o(\ln(T)) $ regret lower bound to study the optimal regret growth rate.
\end{itemize}

We have in total $K=20$ arms and each arm is associated with a Bernoulli reward distribution. 
The mean rewards of the optimal arms are set to $0.9$ and the mean rewards of the sub-optimal arms are set to $0.8$, i.e., $\Delta_i = 0.1$ for all the sub-optimal arms. We set $\alpha = \left\{0.0, 0.8, 0.9,  1.0\right\}$.

\textbf{Discussions for the empirical performances.} We set $T = 10^5$.  Fig.~\ref{fig:regret} reports the results of the regret for all the compared algorithms. Note that when $\alpha = 0$, our proposed TS-TD-$\alpha$ boils down to Vanilla TS (Gaussian priors). 
It is not surprising that our proposed algorithms cannot beat the state-of-the-art algorithms due to the following reasons. 
First, all  the state-of-the-art algorithms except Vanilla TS (Gaussian priors) and $\epsilon$-TS (Gaussian priors)\footnote{$\epsilon$-TS (Gaussian priors) for Bernoulli rewards may not have theoretical guarantees as \cite{jin2023thompson} only provide theoretical guarantee for Gaussian priors for Gaussian rewards.} have been proved to achieve asymptotic optimality, whereas we only prove that our proposed algorithms are problem-dependent optimal. They may not be asymptotically optimal. 
Second, our proposed algorithms are developed based on Gaussian priors instead of Beta priors, but for bounded reward distributions.  Empirically, less exploration leads to better performances. Beta posteriors are in favour of reducing exploration as they  have bounded support, and on some occasions, Beta posterior has a much smaller variance than a  Gaussian posterior given arm $i$'s statistics $\hat{\mu}_{i,n_i}$ and $n_i$.  Smaller variance makes the learning algorithm explore less. 




\begin{figure*}[h]
    \centering
    \begin{subfigure}[b]{0.405\textwidth}
        \includegraphics[width=\textwidth]{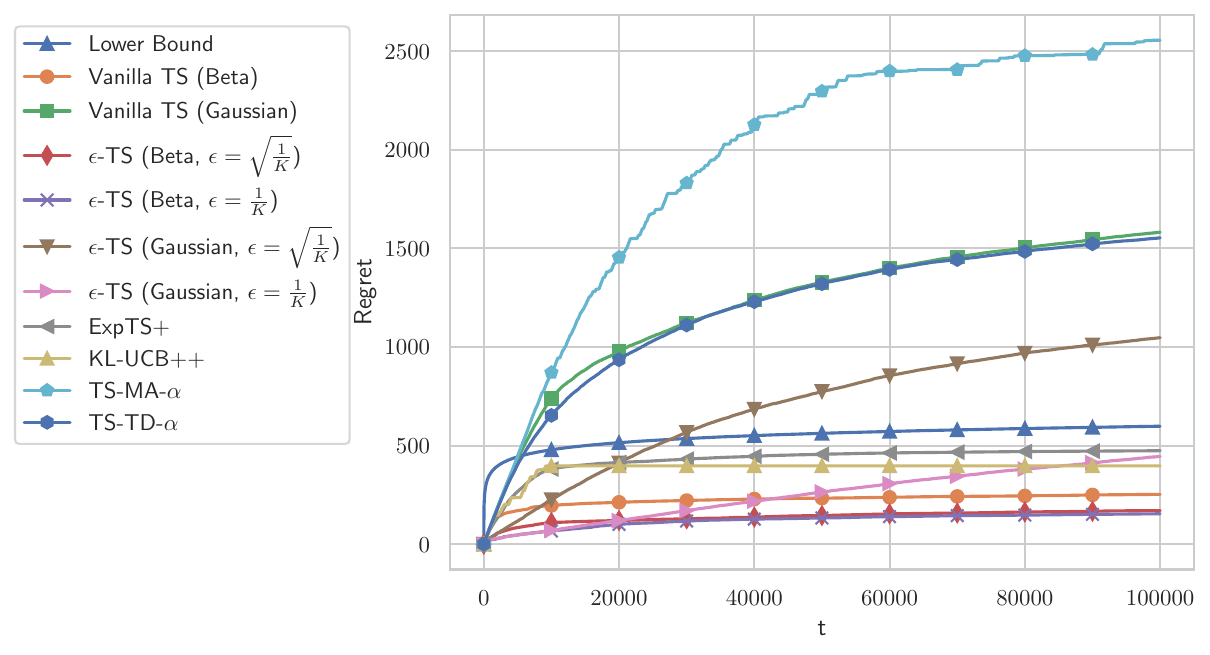}
        \caption{$\alpha = 0$}
        \label{fig:alpha0arm10}
    \end{subfigure}
    \hfill 
    \begin{subfigure}[b]{0.29\textwidth}
        \includegraphics[width=\textwidth]{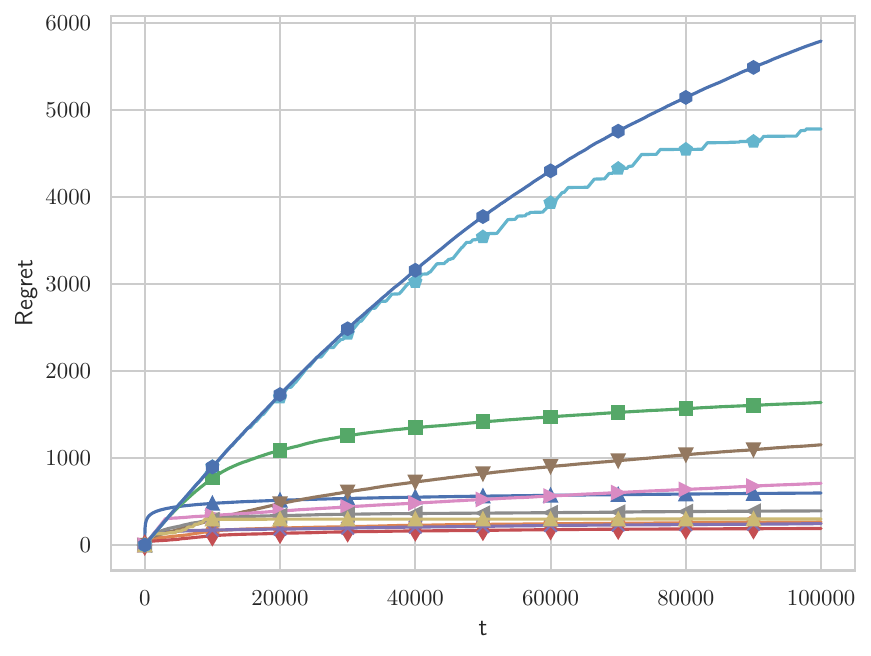}
        \caption{$\alpha = 0.8$}
        \label{fig:alpha08arm10}
    \end{subfigure}
    \hfill 
    \begin{subfigure}[b]{0.29\textwidth}
        \includegraphics[width=\textwidth]{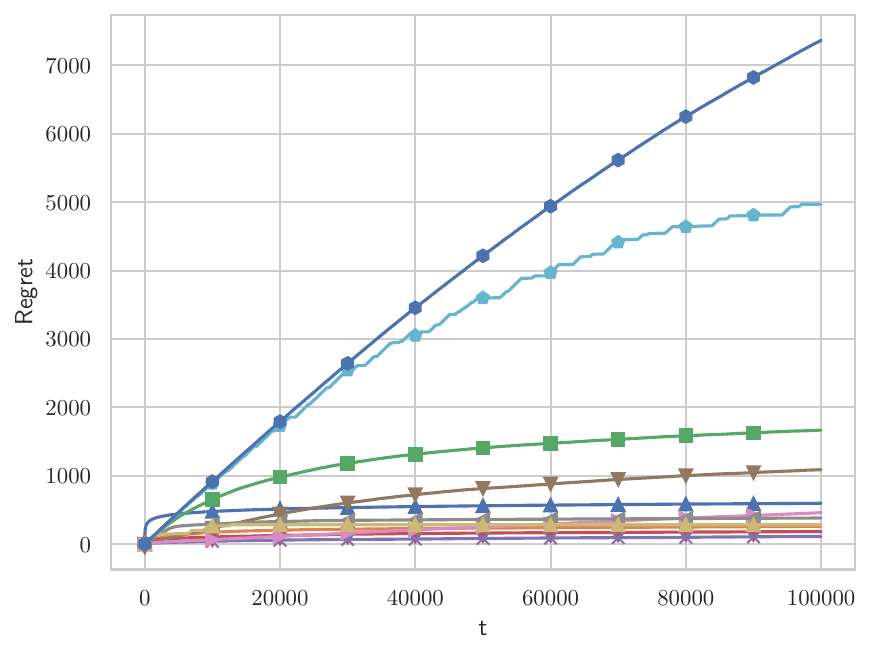}
        \caption{$\alpha = 1$}
        \label{fig:alpha09arm10}
    \end{subfigure}
    \caption{Comparison between different algorithms for $20$ arms with one optimal arm.}
    \label{fig:regret}
\end{figure*}




Fig.~\ref{fig:tstd_alpha_sample_num} reports the total number of drawn posterior samples in TS-TD-$\alpha$ with different values of $\alpha$ to show the adaptivity and efficiency in drawing posterior samples. Note that $T = 10^6$ and $K=20$.  From the results, we can see that when $\alpha = 0$, TS-TD-$\alpha$ basically is Vanilla TS (Gaussian priors). Regardless of the number of optimal arms, the total number of drawn posterior samples is $KT = 2 \times 10^7$.
However, for other  $\alpha$ value, e.g., $\alpha = 0.8$,  given  $K$, the total number of drawn posterior samples decreases with the decrease of the number of the optimal arms.  For example, when we have $K/2=10$ optimal arms, the total number of posterior samples drawn is around $1.0 \times 10^7 $ (the blue bars). In contrast, if we only have one optimal arm, the total number of posterior samples drawn is around $0.25 \times 10^7$ (the green bars). These experimental results demonstrate that TS-TD-$\alpha$ itself is able to figure out how to use sampling resources wisely and is extremely efficient when the number of sub-optimal arms is very large.  

Fig.~\ref{fig:tsma_alpha_pull_num} reports the percentage of data-dependent samples drawn in TS-MA-$\alpha$ for the optimal arms with different numbers of optimal arms and different values of $\alpha$. This is also to study the adaptive property in drawing data-dependent samples.  From the results we can see that  almost all the data-dependent samples are for optimal arms (more than $90\%$) and given $K$ and $\alpha$, the percentage increases with the increase of the number of the optimal arms. For example, when $\alpha = 0.8$, if there is only one optimal arm, the percentage is $94\%$ (the blue bar), whereas the percentage is around $97\%$ (the green bar) when the number of optimal arms is $K/2 = 10$.


\begin{figure}
    \centering
    \includegraphics[width = 0.9\columnwidth]{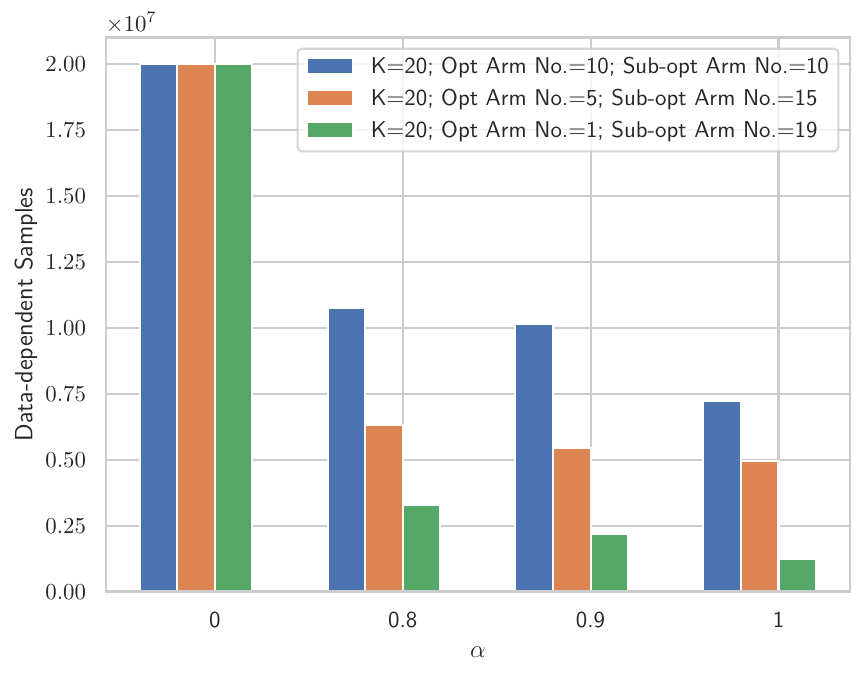}
    \caption{Total number of drawn posterior samples in TS-TD-$\alpha$.}
    \label{fig:tstd_alpha_sample_num}
\end{figure}

\begin{figure}
    \centering
    \includegraphics[width = 0.9\columnwidth]{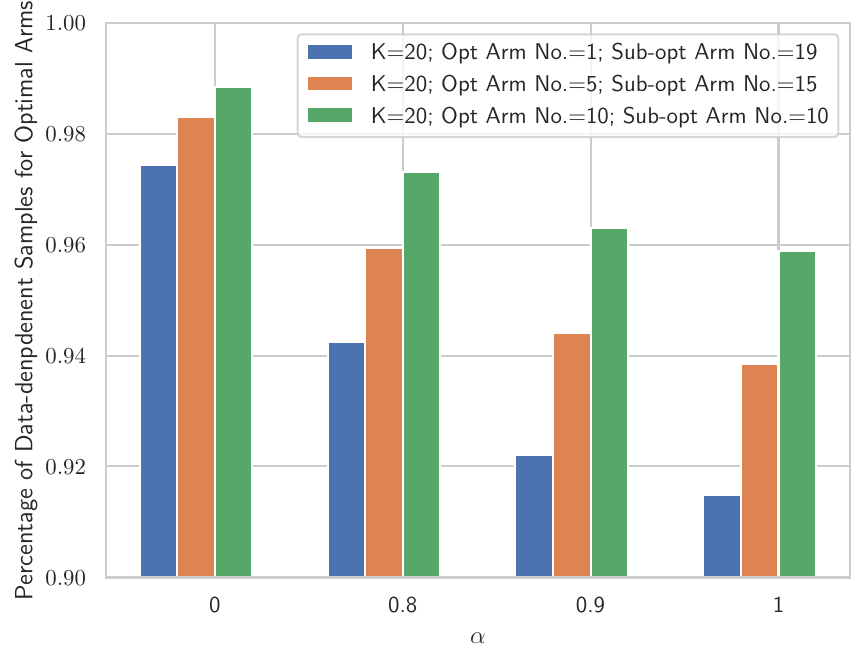}
    \caption{Percentage of data-dependent samples drawn for optimal arms in TS-Ma-$\alpha$.}
    \label{fig:tsma_alpha_pull_num}
\end{figure}






\section{Discussion and future work}


In this work, we have tackled Thompson Sampling-based algorithms from balancing utility  and computation. When $\alpha=0$, both of our proposed  algorithms have the optimal $O \left(K\ln(T)/\Delta\right)$  regret bound. 
 The key advantage of TS-MA-$\alpha$ is the scalability, whereas 
the key advantage of  TS-TD-$\alpha$ is  the expected number of drawn posterior samples for a sub-optimal arm is $\widetilde{O} \left( \sqrt{T}/\Delta^2\right)$, a rate in the order of $\sqrt{T}$ instead of the classical linear rate in $T$.  When $\alpha = 1$, although TS-TD-$\alpha$ is sub-optimal, the expected number of drawn posterior samples for a sub-optimal arm is further reduced to $O\left(\ln^3(T)/\Delta^2 \right)$, a rate only polylogarithmic in $T$. 
We believe it is worthwhile to study the fundamental reasons  resulting in the savings on posterior samples. We conjecture that it is at the cost of sacrificing learning algorithms' anytime property and the worst-case regret bounds. Note that Vanilla Thompson Sampling is anytime and has an $O \left(\sqrt{KT\ln(K)} \right)$ worst-case regret bound, whereas our proposed algorithms are not anytime and only have an $O \left(\sqrt{KT\ln(T)} \right)$ worst-case regret bound when setting $\alpha = 0$. 
The conjecture motivates 
 making our proposed algorithms  anytime as one of the future work. 
Since $\epsilon$-TS from \cite{jin2023thompson} has demonstrated excellent practical performance for some reward distributions in exponential family, another future work is 
 to study the possibility to generalize $\epsilon$-TS to other reward distributions, for example, bounded reward distributions or Sub-Gaussian reward distributions. 
\cite{jin2023thompson}  proved that $\epsilon$-TS (setting $\epsilon = 1$) with Gaussian priors and rewards is simultaneously minimax optimal and asymptotically optimal, whereas $\epsilon$-TS (setting $\epsilon$ as a constant)  is  simultaneously Sub-UCB,  asymptotically optimal, and minimax optimal up to an extra $\sqrt{\ln (K)}$ factor. 
So far, Gaussian priors for bounded rewards can only be Sub-UCB and minimax optimal up to an extra $\sqrt{\ln (K)}$ factor (our Theorem~\ref{Theorem: TS new 1}). It would be interesting  to see whether Gaussian priors for bounded rewards can   achieve  Sub-UCB and  asymptotic optimality. We are not optimistic about it and  conjecture that Gaussian priors for bounded rewards can never achieve asymptotic optimality. 










\newpage

\onecolumn

\title{Efficient and Adaptive Posterior Sampling Algorithms for Bandits\\(Supplementary Material)}
\maketitle


\appendix
The appendix is organized as follows.
\begin{enumerate}
    \item Motivations for our proposed learning algorithms are provided in Appendix~\ref{app: application};
    \item Useful facts related to this work are provided in Appendix~\ref{app: fact};
    \item Proofs for Theorem~\ref{Theorem: TS new 1} are presented in Appendix~\ref{App: Vanilla TS};
    \item Proofs for Theorem~\ref{Theorem: Alpha-TS} are presented in Appendix~\ref{app: TS-MA};
    \item Proofs for Theorem~\ref{Theorem: Main 3} are presented in Appendix~\ref{app: TS-TD};

\item Proofs for the worst-case regret bounds are presented in Appendix~\ref{app: worst-case} .

    
\end{enumerate}
\section{Motivations} \label{app: application}
Most bandit learning algorithms only focus on sample efficiency, but for large-scale applications, algorithms' \emph{computational efficiency} and \emph{algorithmic scalability} also play important roles.
A learning algorithm that uses fewer posterior samples is computationally efficient because drawing posterior samples and processing the generated samples both consume computational resources.
Algorithmic scalability is crucial because scalable algorithms can function steadily when the volume of data is increasing. Take online advertising systems for instance. A platform often receives a changing volume of interactions from users and needs to adjust the decisions frequently to capture users' dynamic preferences to maximize the overall click rate. An algorithm with good scalability can maintain a steady response speed when the volume of user interactions grows. Typically,  more computation brings better performance. 
A good learning algorithm should also be able to balance the trade-off between performance and computational cost. We show,  both theoretically and empirically, that our proposed algorithms do not sacrifice a lot in terms of accumulating rewards but consume fewer computational resources. We introduce parameter $\alpha \in [0,1]$, which allows practitioners to control this trade-off and personalize their algorithms based on their priorities and preferences.
We would like to emphasize that although our proposed learning algorithms use fewer posterior samples, provably, the exploration within the learning algorithms is still sufficient. This is because our proposed algorithms adaptively allocate sufficient computation to the optimal arms over time and save on computation from the sub-optimal arms. 
In the example of  online advertising systems, 
 our proposed algorithms can allocate more computational resources to process the portion of data believed to bring good performance to the platform instead of distributing all the available computational resources evenly.

\section{Useful Facts} \label{app: fact}
\begin{fact}
  Let $X_1, X_2, \dotsc, X_n$ be independent random variables with support $[0,1]$. Let ${\mu}_{1:n} = \frac{1}{n}\sum\limits_{i=1}^{n}X_i$. Then, for any $z >0$, we have
    \begin{equation}
        \mathbb{P} \left\{\left|{\mu}_{1:n} -\mathbb{E}\left[{\mu}_{1:n}\right]\right|  \ge z\sqrt{\frac{1}{n}}  \right\} \le 2e^{-2z^2}\quad.
    \end{equation}
    \label{Hoeffding}
\end{fact}
\begin{fact}
    (Concentration and anti-concentration bounds of Gaussian distributions). For a Gaussian distributed random variable $Z$ with mean $\mu$ and variance $\sigma^2$, for any $z > 0$, we have
\begin{equation}
    \mathbb{P} \left\{Z > \mu + z \sigma \right\} \le \frac{1}{2}e^{- \frac{z^2}{2}}, \quad \mathbb{P} \left\{Z < \mu - z \sigma \right\} \le \frac{1}{2}e^{- \frac{z^2}{2}}\quad,
\end{equation}
and 
\begin{equation}
     \mathbb{P} \left\{Z > \mu + z \sigma \right\} \ge \frac{1}{\sqrt{2\pi}} \frac{z}{z^2+1} e^{- \frac{z^2}{2}} \quad.
\end{equation}
\label{Fact 1}
\end{fact}
\begin{fact}
    For any fixed $T > e^3$, for any $\alpha \in [0,1]$, we have
    \begin{equation}
        \ln^{1-\alpha} (T) \le (1-\alpha)\ln(T) +1\quad.
    \end{equation}
    \label{Fact: calculus}
\end{fact}
\begin{proof}
    Let $f(\alpha) = (1-\alpha)\ln(T) +1-\ln^{1-\alpha} (T)$. Then, we have
    \begin{equation}
        f'(\alpha) = -\ln(T)+\ln^{1-\alpha}(T)\ln(\ln(T))\quad.
    \end{equation}
    It is not hard to verify that $f'\left( \frac{\ln(\ln(\ln(T)))}{\ln(\ln(T))} \right)=0$, and, when $\alpha \in \left[0, \frac{\ln(\ln(\ln(T)))}{\ln(\ln(T))} \right]$, we have $f'(\alpha) \ge 0$, and when $\alpha \in \left[ \frac{\ln(\ln(\ln(T)))}{\ln(\ln(T))},1 \right]$, we have $f'(\alpha) \le 0$. We also have $f(0) = 1$ and $f(1) = 0$. Therefore, for any $\alpha \in \left[0, \frac{\ln(\ln(\ln(T)))}{\ln(\ln(T))} \right]$, we have $f(\alpha) \ge 1$, and, any $\alpha \in \left[ \frac{\ln(\ln(\ln(T)))}{\ln(\ln(T))},1 \right]$, we have $f(\alpha) \ge 0$.
\end{proof}

\section{Proofs for Theorem~\ref{Theorem: TS new 1}}\label{App: Vanilla TS}

In this appendix, we present proofs for Lemma~\ref{UBC 22} and the full proofs for Theorem~\ref{Theorem: TS new 1}.

 \subsection{Proofs for Lemma~\ref{UBC 22}}
\paragraph{Re-statement of Lemma~\ref{UBC 22}.} Let $L_{1,i} = 4\left(\sqrt{2}+\sqrt{3.5}\right)^2 \ln \left(T \Delta_i^2 + 100^{\frac{1}{3}}\right)/\Delta_i^2$. Let $\tau_s^{(1)}$ be the round when the $s$-th pull of the optimal arm $1$ occurs and $\theta_{1,s} \sim \mathcal{N}\left(\hat{\mu}_{1,s}, \frac{1}{s}\right)$. Then, we have 
\begin{equation*}
    \mathbb{E}\left[\frac{1}{\mathbb{P} \left\{ \theta_{1,s} > \mu_1 - \frac{\Delta_i}{2}\mid \mathcal{F}_{\tau_{s}^{(1)}}  \right\} }-1   \right]  = \begin{cases}
 29,  & \forall s \ge 1,\\
 \frac{180}{T\Delta_i^2}, & \forall s \ge L_{1,i}. 
\end{cases}
\end{equation*}

\begin{proof}{of Lemma~\ref{UBC 22}:}
   We have two results stated in Lemma~\ref{UBC 22}. The purpose of the first result is to show that even if the optimal arm has not been pulled enough, the regret caused by the underestimation of the optimal arm can still be well controlled.  The second result states that after the optimal arm $1$ has been pulled enough, the regret caused by the underestimation of the optimal arm is very small. 

         Similar to the proof of Lemma~2.13 in \cite{agrawalnear}, we introduce a geometric random variable in the proof. For any integer $s \ge 1$, we let $\mathcal{G}_{1,s}$ be a geometric random variable denoting the number of consecutive independent trials until event  $\theta_{1,s} > \mu_1 - \frac{\Delta_i}{2}$  occurs. Then, we have
         \begin{equation}
          \begin{array}{l}
    \mathbb{E}\left[\frac{1}{\mathbb{P} \left\{ \theta_{1,s} > \mu_1 - \frac{\Delta_i}{2}\mid \mathcal{F}_{\tau_{s}^{(1)}}  \right\} }\right] =             \mathbb{E}\left[    \frac{1}{ \mathbb{P} \left\{\theta_{1,s}>\mu_1 - \frac{\Delta_i}{2} \mid \hat{\mu}_{1,s}  \right\} }  \right] 
           =    \mathbb{E} \left[\mathbb{E} \left[\mathcal{G}_{1,s} \mid \hat{\mu}_{1,s}\right] \right] 
             =    \mathbb{E} \left[\mathcal{G}_{1,s} \right].
              \end{array} 
         \end{equation}
  To upper bound $\mathbb{E}\left[\frac{1}{\mathbb{P} \left\{ \theta_{1,s} > \mu_1 - \frac{\Delta_i}{2}\mid \mathcal{F}_{\tau_{s}^{(1)}}  \right\} }   \right]$, it is sufficient to upper bound $ \mathbb{E} \left[\mathcal{G}_{1,s} \right]$. To upper bound $ \mathbb{E} \left[\mathcal{G}_{1,s} \right]$, we use the definition of expectation when the random variable has non-negative integers as support. We have
             \begin{equation}
    \begin{array}{lllll}
            \mathbb{E} \left[\mathcal{G}_{1,s} \right] & = & 
        \sum\limits_{r = 0}^{\infty} \mathbb{P} \left\{\mathcal{G}_{1,s} > r  \right\} &= & \sum\limits_{r = 0}^{\infty} \mathbb{E} \left[\mathbb{P} \left\{\mathcal{G}_{1,s} > r \mid \hat{\mu}_{1,s} \right\}  \right]\quad. 
         \end{array}
\end{equation}
For any $s \ge 1$, we claim 
         \begin{equation}
\begin{array}{lll}   
   \mathbb{E} \left[\mathbb{P} \left\{\mathcal{G}_{1,s} > r \mid \hat{\mu}_{1,s} \right\}  \right] & \le
&
\begin{cases}
   1,& r \in [0,12] \quad,\\
     e^{-\sqrt{\frac{r  }{\pi}} \cdot \frac{\ln(13)}{\ln(13)+2} } + \frac{1}{r},& r \in [13,100] \quad,\\
e^{- \sqrt{\frac{r}{3\pi}} }  + r^{-\frac{4}{3}}   ,              & \text{$r \ge 101$}\quad.
\end{cases}
\end{array}
\label{UBC 333}
\end{equation}
      For any $s \ge L_{1,i}=\frac{4\left(\sqrt{2}+\sqrt{3.5}\right)^2 \ln \left(T \Delta_i^2 + 100^{\frac{1}{3}}\right)}{\Delta_i^2}$, we claim
        \begin{equation}
 \begin{array}{lll}
   \mathbb{E} \left[ \mathbb{P} \left\{\mathcal{G}_{1,s} > r \mid \hat{\mu}_{1,s}\right\} \right] &\le&
\begin{cases}
1, & r=0\quad,\\
    \frac{1}{r^2 \left(T\Delta_i^2\right)} +\frac{0.5^r}{ \left(T\Delta_i^2\right)^r} ,&  r \in \left[1, \left\lfloor \left( T\Delta_i^2 + 100^{\frac{1}{3}} \right)^3 \right\rfloor \right]\quad,\\
e^{- \sqrt{\frac{r}{3\pi}} } + r^{-\frac{4}{3}}   ,              & \text{$r \ge \left\lfloor \left( T\Delta_i^2 + 100^{\frac{1}{3}} \right)^3 \right\rfloor + 1$} \quad.
\end{cases}
\end{array}
\label{UBC 4}
\end{equation}
The proofs for the results shown in (\ref{UBC 333}) and   (\ref{UBC 4}) are deferred to the end of this session.

With (\ref{UBC 333}) in hand, for any fixed integer $s \ge 1$, we have
    \begin{equation}
        \begin{array}{ll}
&\mathbb{E}\left[\frac{1}{\mathbb{P} \left\{ \theta_{1,s} > \mu_1 - \frac{\Delta_i}{2}\mid \mathcal{F}_{\tau_{s}^{(1)}}  \right\} }-1   \right] \\
=& 
\mathbb{E} \left[\mathcal{G}_{1,s} \right]-1 \\
         = & \sum\limits_{r = 0}^{\infty} \mathbb{E} \left[\mathbb{P} \left\{\mathcal{G}_{1,s} > r \mid \hat{\mu}_{1,s} \right\}  \right] -1\\
               = & \sum\limits_{r = 0}^{12} \mathbb{E} \left[\mathbb{P} \left\{\mathcal{G}_{1,s} > r \mid \hat{\mu}_{1,s} \right\}  \right] + \sum\limits_{r = 13}^{100} \mathbb{E} \left[\mathbb{P} \left\{\mathcal{G}_{1,s} > r \mid \hat{\mu}_{1,s} \right\}  \right] + \sum\limits_{r = 101}^{\infty} \mathbb{E} \left[\mathbb{P} \left\{\mathcal{G}_{1,s} > r \mid \hat{\mu}_{1,s} \right\}  \right]-1 \\
            \le & 13 + \sum\limits_{r=13}^{100} \left(e^{-\sqrt{\frac{r  }{\pi}} \cdot \frac{\ln(13)}{\ln(13)+2} } + \frac{1}{r} \right) + \sum\limits_{r=101}^{\infty} \left(e^{- \sqrt{\frac{r}{3\pi}} } + \frac{1}{r^{\frac{4}{3}}} \right) -1\\
             \le & 12 + \int_{12}^{100} \left(e^{-\sqrt{\frac{r  }{\pi}} \cdot \frac{\ln(13)}{\ln(13)+2} } + \frac{1}{r}\right) dr+ \int_{100}^{\infty} \left(e^{- \sqrt{\frac{r}{3\pi}} }  + \frac{1}{r^{\frac{4}{3}}} \right)dr \\

\le & 12 + 10.44 +2.13 + 3.1 + 0.65 \\
\le & 29 \quad,    
        \end{array}  
        \label{Yoga 3}
        \end{equation}
        which concludes the proof for the first stated result in Lemma~\ref{UBC 22}.

With (\ref{UBC 4}), for any fixed integer $s \ge L_{1,i}$, we have
    \begin{equation}
        \begin{array}{ll}
       & \mathbb{E}\left[\frac{1}{\mathbb{P} \left\{ \theta_{1,s} > \mu_1 - \frac{\Delta_i}{2}\mid \mathcal{F}_{\tau_{s}^{(1)}}  \right\} }-1   \right] \\
         = &  \sum\limits_{r=0}^{\infty}\mathbb{E} \left[ \mathbb{P} \left\{\mathcal{G}_{1,s} > r \mid \hat{\mu}_{1,s}\right\} \right] -1\\
           
               \le & 1+\sum\limits_{r = 1}^{\left\lfloor \left( T\Delta_i^2 + 100^{\frac{1}{3}} \right)^3 \right\rfloor}   \mathbb{E} \left[ \mathbb{P} \left\{\mathcal{G}_{1,s} > r \mid \hat{\mu}_{1,s}\right\} \right] + \sum\limits_{r = \left\lfloor \left( T\Delta_i^2 + 100^{\frac{1}{3}} \right)^3 \right\rfloor + 1}^{\infty}  \mathbb{E} \left[ \mathbb{P} \left\{\mathcal{G}_{1,s} > r \mid \hat{\mu}_{1,s}\right\} \right]-1\\
                \le & \sum\limits_{r = 1}^{\left\lfloor \left( T\Delta_i^2 + 100^{\frac{1}{3}} \right)^3 \right\rfloor}  \left(\frac{1}{r^2 \left(T\Delta_i^2\right)} +\frac{0.5^r}{ \left(T\Delta_i^2\right)^r}  \right) +  \sum\limits_{r = \left\lfloor \left( T\Delta_i^2 + 100^{\frac{1}{3}} \right)^3 \right\rfloor + 1}^{\infty}  \left(e^{- \sqrt{\frac{r}{3\pi}} } + r^{-\frac{4}{3}} \right) \\
                        \\
                         \le & \frac{1}{T\Delta_i^2} + \frac{0.5}{T\Delta_i^2} +   \int_{1}^{+\infty } \left( \frac{1}{r^2 \left(T\Delta_i^2\right)} +    \frac{1}{  \left(2T\Delta_i^2\right)^r} \right) dr +\int_{   (T\Delta_i^2)^{2}  }^{\infty} \ e^{- \frac{\sqrt{r}}{\sqrt{3\pi}} }  dr+ \int_{    (T\Delta_i^2)^{3}}^{\infty}  \ r^{-\frac{4}{3}} dr
                        \\
                        \le^{(a)} &  \frac{1}{ T\Delta_i^2  } +\frac{0.5}{ T\Delta_i^2   } +\frac{1}{ T\Delta_i^2  } +\frac{0.5}{ T\Delta_i^2   } +  \frac{6 \cdot 3\pi \cdot \sqrt{3\pi}}{ T\Delta_i^2   } + \frac{3}{ T\Delta_i^2 } \\
                       \le &  \frac{180}{ T\Delta_i^2  }  \quad,
        \end{array}  
        \end{equation}
        which concludes the proof for the second stated result in Lemma~\ref{UBC 22}.
        
        Inequality (a) uses the fact that, for any $a, b >0$, we have
       $\int_{b}^{+\infty}e^{-a\sqrt{x}}dx  
        =  \frac{2\sqrt{b}}{ae^{a\sqrt{b}}} + \frac{2}{a^2 e^{a \sqrt{b}}} 
        \le  \frac{2\sqrt{b}}{a \cdot \left(1+ a\sqrt{b}+\frac{1}{2}a^2b  \right)} + \frac{2}{a^2 \cdot \left(1+ a\sqrt{b}+\frac{1}{2}a^2b \right)}  
        \le  \frac{4}{a^3 \sqrt{b}} + \frac{2}{a^3\sqrt{b}} 
        =  \frac{6}{a^3\sqrt{b}}$,
where the first inequality uses $e^x \ge 1+x+ \frac{1}{2}x^2$.

Now, we present the proofs for the results shown in (\ref{UBC 333}) and   (\ref{UBC 4}).
\paragraph{Proofs for (\ref{UBC 333}).}

We express the LHS in (\ref{UBC 333}) as 
        \begin{equation}
            \begin{array}{l}
         \mathbb{E} \left[ \mathbb{P} \left\{\mathcal{G}_{1,s} > r \mid \hat{\mu}_{1,s}\right\} \right] 
          =1- \mathbb{E} \left[\mathbb{P} \left\{\mathcal{G}_{1,s} \le r \mid \hat{\mu}_{1,s} \right\}\right]    

          =1- \underbrace{\mathbb{E} \left[\mathbb{E} \left[ \bm{1} \left\{ \mathcal{G}_{1,s} \le r  \right\} \mid \hat{\mu}_{1,s} \right]\right] }_{=:\gamma} \quad.
          \end{array}
          \label{ddd 4}
          \end{equation}
Our goal is to construct a lower bound for $\gamma$. Let $\theta_{1,s}^{h}$ for all $h \in [r]$ be i.i.d. random variables according to  $\mathcal{N}\left(\hat{\mu}_{1,s}, \frac{1}{s} \right)$. 

\

\paragraph{When $r \in [0,12]$,} the proof is trivial as $\gamma \ge 0 $. Then, we have $\mathbb{E} \left[ \mathbb{P} \left\{\mathcal{G}_{1,s} > r \mid \hat{\mu}_{1,s}\right\} \right] \le 1$.

\paragraph{When $r \in [13,100]$,}  
         we  introduce  $ z = \sqrt{\frac{1}{2} \ln(r)} >0$ and have
 \begin{equation}
        \begin{array}{lll}
 \gamma &
  = & \mathbb{E} \left[\mathbb{E} \left[ \bm{1} \left\{ \mathcal{G}_{1,s} \le r  \right\} \mid \hat{\mu}_{1,s} \right]\right]  \\
& \ge & 
\mathbb{E} \left[\mathbb{E} \left[ \bm{1}\left\{ \mathop{\max}\limits_{h \in [r]} \theta_{1,s}^{h} > \mu_1 - \frac{\Delta_i}{2}  \right\} \mid \hat{\mu}_{1,s}  \right]\right] \\
& \ge &
 \mathbb{E} \left[
\mathbb{E} \left[\bm{1} \left\{\hat{\mu}_{1,s} + z\sqrt{\frac{1}{s}} \ge \mu_1 - \frac{\Delta_i}{2} \right\} \bm{1}\left\{ \mathop{\max}\limits_{h \in [r]} \theta_{1,s}^{h} >\hat{\mu}_{1,s} +z \sqrt{\frac{1}{s}} \right\} \mid \hat{\mu}_{1,s}  \right]\right] \\
& = &
 \mathbb{E} \left[\bm{1} \left\{\hat{\mu}_{1,s} +z \sqrt{\frac{1}{s}} \ge \mu_1 - \frac{\Delta_i}{2} \right\} \cdot
\underbrace{\mathbb{P} \left\{ \mathop{\max}\limits_{h \in [r]} \theta_{1,s}^{h} >\hat{\mu}_{1,s} +z\sqrt{\frac{1}{s}}  \mid \hat{\mu}_{1,s} \right\} }_{=:\beta}\right] \quad.
\end{array}
\end{equation}

 We construct a lower bound for $\beta$ and have
        \begin{equation}
            \begin{array}{lll}
            \beta & = &  \mathbb{P} \left\{ \mathop{\max}\limits_{h \in [r]} \theta_{1,s}^{h} >\hat{\mu}_{1,s} +z\sqrt{\frac{ 1}{s}}  \mid \hat{\mu}_{1,s} \right\} \\

               &            = & 1- \prod\limits_{h \in [r]} \left( 1- \mathbb{P} \left\{ \theta_{1,s}^{h} > \hat{\mu}_{1,s} +z \sqrt{\frac{ 1}{s}}  \mid \hat{\mu}_{1,s} \right\}   \right)  \\
                &           = & 1- \left( 1- \mathbb{P} \left\{ \theta_{1,s} > \hat{\mu}_{1,s} +z\sqrt{\frac{ 1}{s}}  \mid \hat{\mu}_{1,s} \right\}\right)^r  \\
                 &          \ge^{(a)} &    1- \left( 1- \frac{1}{\sqrt{2\pi}} \frac{\sqrt{\frac{1}{2}\ln(r)}}{\frac{1}{2}\ln(r)+1} e^{- \frac{1}{4}\ln(r)} \right)^r  \\
                  & \ge^{(b)} & 1- e^{-r \cdot \frac{1}{\sqrt{2\pi}} \frac{\sqrt{\frac{1}{2}\ln(r)}}{\frac{1}{2}\ln(r)+1} \cdot r^{-\frac{1}{4}}} \\
                         & \ge^{(c)} & 1- e^{-r^{\frac{1}{2}} \cdot \frac{1}{\sqrt{2\pi}} \frac{\sqrt{\frac{1}{2}\ln(r)}}{\frac{1}{2}\ln(r)+\frac{1}{\ln(13)}\ln(r)} \cdot r^{\frac{1}{4}}} \\
                            & = & 1- e^{-r^{\frac{1}{2}} \cdot \frac{1}{\sqrt{\pi}} \cdot \frac{\ln(13)}{\ln(13)+2}\frac{\sqrt{r^{\frac{1}{2}}}}{\sqrt{\ln(r)}} } \\
                            & \ge^{(d)} & 1-e^{-r^{\frac{1}{2}} \cdot \frac{1}{\sqrt{\pi}} \cdot \frac{\ln(13)}{\ln(13)+2} }\quad.
            \end{array}
        \end{equation}

        Inequalities (a) and (b) use anti-concentration bounds of Gaussian distributions (Fact~\ref{Fact 1}) and the fact that $e^{-x} \ge 1-x$. Inequalities (c) and (d)
        use the facts that when $r \ge 13$, we have $1 \le \frac{\ln(r)}{\ln(13)}$ and $\sqrt{r^{\frac{1}{2}}} \ge \sqrt{\ln(r)}$.
        
           Now, we have
        \begin{equation}
            \begin{array}{lll}
                 \gamma & \ge & \left(1-e^{-r^{\frac{1}{2}} \cdot \frac{1}{\sqrt{\pi}} \cdot \frac{\ln(13)}{\ln(13)+2} } \right)\cdot \mathbb{P} \left\{ \hat{\mu}_{1,s} +\sqrt{\frac{1}{2}\ln(r)}\sqrt{\frac{1}{s}} \ge \mu_1 - \frac{\Delta_i}{2}\right\}\\
                  & \ge & \left(1-e^{-r^{\frac{1}{2}} \cdot \frac{1}{\sqrt{\pi}} \cdot \frac{\ln(13)}{\ln(13)+2} } \right)\cdot \mathbb{P} \left\{ \hat{\mu}_{1,s} +\sqrt{\frac{1}{2}\ln(r)}\sqrt{\frac{1}{s}} \ge \mu_1 \right\}\\
                   & \ge^{(a)} & \left(1-e^{-r^{\frac{1}{2}} \cdot \frac{1}{\sqrt{\pi}} \cdot \frac{\ln(13)}{\ln(13)+2} } \right)\cdot \left(1-\frac{1}{r} \right)\\
                   & \ge & 1-\left( e^{-\sqrt{\frac{r}{\pi}}  \cdot \frac{\ln(13)}{\ln(13)+2} } + \frac{1}{r}\right)\quad,
            \end{array}
        \end{equation}
        where $(a)$ uses Hoeffding's inequality (Fact~\ref{Hoeffding}).

          Plugging the lower bound of $\gamma$ into (\ref{ddd 4}) gives $ \mathbb{E} \left[ \mathbb{P} \left\{\mathcal{G}_{1,s} > r \mid \hat{\mu}_{1,s}\right\} \right]   \le e^{-\sqrt{\frac{r}{\pi}}  \frac{\ln(13)}{\ln(13)+2} }+ \frac{1}{r}$.
\paragraph{When  $r \ge 101$,} we  introduce  $ z = \sqrt{\frac{2}{3} \ln(r)} >0$. We still construct the lower bound of $\gamma$ as
\begin{equation}
    \begin{array}{lll}
         \gamma & \ge & \mathbb{E} \left[\bm{1} \left\{\hat{\mu}_{1,s} +z \sqrt{\frac{1}{s}} \ge \mu_1 - \frac{\Delta_i}{2} \right\} \cdot
\mathbb{P} \left\{ \mathop{\max}\limits_{h \in [r]} \theta_{1,s}^{h} >\hat{\mu}_{1,s} +z\sqrt{\frac{1}{s}}  \mid \hat{\mu}_{1,s} \right\} \right] \\
& \ge &  \left( 1- \left( 1- \frac{1}{\sqrt{2\pi}} \frac{z}{z^2+1} e^{-0.5  z^2} \right)^r  \right) \cdot \mathbb{P}\left\{\hat{\mu}_{1,s} +z \sqrt{\frac{1}{s}} \ge \mu_1 - \frac{\Delta_i}{2} \right\}  \\
& \ge & \left(1- \underbrace{e^{-r \cdot \frac{1}{\sqrt{2\pi}} \frac{z}{z^2+1} e^{-0.5  z^2}}}_{=:f(r)} \right) \cdot  \mathbb{P}\left\{\hat{\mu}_{1,s} +\sqrt{\frac{2\ln(r)}{3}} \sqrt{\frac{1}{s}} \ge \mu_1 \right\} \\
& \ge & \left(1- e^{-\sqrt{\frac{r}{3\pi} }}\right) \cdot \left( 1- \frac{1}{r^{\frac{4}{3}}}\right) \\
& \ge & 1-\left(e^{-\sqrt{\frac{r}{3\pi} }} +\frac{1}{r^{\frac{4}{3}}}\right)\quad.
    \end{array}\label{plate}
\end{equation}
The third inequality in (\ref{plate}) uses the fact that
\begin{equation}
            \begin{array}{lll}
                 f(r) & = & e^{-r \cdot \frac{1}{\sqrt{2\pi}} \frac{z}{z^2+1} e^{-0.5  z^2}} \\
                 & = & e^{-r \cdot \frac{1}{\sqrt{2\pi}} \frac{\sqrt{\frac{\ln r}{1.5}}}{\frac{\ln r}{1.5}+1} e^{-0.5 \cdot \frac{\ln r}{1.5}}} \\
                            & \le^{(a)} & e^{- \frac{1}{\sqrt{2\pi}} \frac{\sqrt{1.5\ln r}}{\ln r+0.5 \ln r} \cdot r^{\frac{2}{3}}}  \\
                         & = & e^{-r^{\frac{1}{2}} \cdot \frac{1}{\sqrt{3\pi}} \sqrt{\frac{r^{\frac{1}{3}} }{\ln r}} } \\
                         & \le^{(b)} & e^{-\sqrt{\frac{r}{3\pi} }} \quad,      
            \end{array}
            \label{ddd}
        \end{equation}
        where inequality (a) and (b) use the facts that when $r \ge 100$, we have $1.5 < 0.5 \ln r$ and  $\sqrt{\frac{r^{\frac{1}{3}} }{\ln r}} \ge 1$.
                    
                 Plugging the lower bound of $\gamma$ into (\ref{ddd 4}) gives
$ \mathbb{E} \left[ \mathbb{P} \left\{\mathcal{G}_{1,s} > r \mid \hat{\mu}_{1,s}\right\} \right]  \le e^{-\sqrt{\frac{r}{3\pi} }} + \frac{1}{r^{\frac{4}{3}}}$.

    \paragraph{Proofs for (\ref{UBC 4}).}  We still express the LHS in (\ref{UBC 4}) as 
        \begin{equation}
            \begin{array}{l}
         \mathbb{E} \left[ \mathbb{P} \left\{\mathcal{G}_{1,s} > r \mid \hat{\mu}_{1,s}\right\} \right] 
          = 1- \mathbb{E} \left[\mathbb{P} \left\{\mathcal{G}_{1,s} \le r \mid \hat{\mu}_{1,s} \right\}\right]    

          =1- \underbrace{\mathbb{E} \left[\mathbb{E} \left[ \bm{1} \left\{ \mathcal{G}_{1,s} \le r  \right\} \mid \hat{\mu}_{1,s} \right]\right] }_{=:\gamma} \quad.
          \end{array}
          \label{ddd 444}
          \end{equation}  
          \paragraph{When $r =0$,} we have $\gamma \ge 0$, which means $\mathbb{E} \left[ \mathbb{P} \left\{\mathcal{G}_{1,s} > r \mid \hat{\mu}_{1,s}\right\} \right] \le 1$.

          \paragraph{When $r \in \left[1, \left\lfloor \left( T\Delta_i^2 + {100}^{\frac{1}{3}} \right)^3 \right\rfloor \right]$,} we define event $\mathcal{E}_{1,s}^{\mu} := \left\{ \mu_1 \le  \hat{\mu}_{1,s}  + \sqrt{\frac{0.5 \ln \left(r^2  \left(T\Delta_i^2 + 100^{\frac{1}{3}} \right) \right)}{s}}\right\}$. Then, we have
    \begin{equation}
        \begin{array}{lll}
       \gamma 
      & = &\mathbb{E} \left[\mathbb{E} \left[ \bm{1} \left\{ \mathcal{G}_{1,s} \le r  \right\} \mid \hat{\mu}_{1,s} \right]\right] \\
           & \ge  & \mathbb{E} \left[\mathbb{E} \left[ \bm{1}\left\{ \mathop{\max}\limits_{h \in [r]} \theta_{1,s}^{h} > \mu_1 - \frac{\Delta_i}{2}  \right\} \mid \hat{\mu}_{1,s}  \right]\right]  \\
           & = & \mathbb{E} \left[ \bm{1}\left\{ \mathop{\max}\limits_{h \in [r]} \theta_{1,s}^{h} > \mu_1 - \frac{\Delta_i}{2}  \right\} \right]\\
            & = & \mathbb{E} \left[\bm{1}\left\{\mathcal{E}_{1,s}^{\mu} \right\}\bm{1}\left\{ \mathop{\max}\limits_{h \in [r]} \theta_{1,s}^{h} > \mu_1 - \frac{\Delta_i}{2}  \right\} \right] + \mathbb{E} \left[\bm{1}\left\{\overline{\mathcal{E}_{1,s}^{\mu}} \right\}\bm{1}\left\{ \mathop{\max}\limits_{h \in [r]} \theta_{1,s}^{h} > \mu_1 - \frac{\Delta_i}{2}  \right\} \right] \\
              & \ge^{(a)} & \mathbb{E} \left[\bm{1}\left\{\mathcal{E}_{1,s}^{\mu} \right\}\bm{1}\left\{ \mathop{\max}\limits_{h \in [r]} \theta_{1,s}^{h} >  \hat{\mu}_{1,s} + \sqrt{\frac{0.5\ln \left(r^2  \left( T\Delta_i^2 + {100}^{\frac{1}{3}} \right)\right)}{s}} - \frac{\Delta_i}{2}  \right\} \right]  \\
       & \ge^{(b)} & \mathbb{E} \left[\bm{1}\left\{\mathcal{E}_{1,s}^{\mu} \right\}\bm{1}\left\{ \mathop{\max}\limits_{h \in [r]} \theta_{1,s}^{h} >  \hat{\mu}_{1,s} + \sqrt{\frac{0.5\ln \left(r^2  \left( T\Delta_i^2 + {100}^{\frac{1}{3}} \right)\right)}{s}} - \sqrt{\frac{\left(\sqrt{2}+\sqrt{3.5}\right)^2 \ln \left(T\Delta_i^2 + 100^{\frac{1}{3}}\right)}{s} }  \right\} \right]  \\
& \ge^{(c)} & \mathbb{E} \left[\bm{1}\left\{\mathcal{E}_{1,s}^{\mu} \right\}\bm{1}\left\{ \mathop{\max}\limits_{h \in [r]} \theta_{1,s}^{h} >  \hat{\mu}_{1,s} + \sqrt{\frac{0.5\ln \left(  \left( T\Delta_i^2 + {100}^{\frac{1}{3}} \right)^7\right)}{s}} - \sqrt{\frac{\left(\sqrt{2}+\sqrt{3.5}\right)^2 \ln \left(T\Delta_i^2 + 100^{\frac{1}{3}}\right)}{s} }   \right\} \right]  \\
& = & \mathbb{E} \left[\bm{1}\left\{\mathcal{E}_{1,s}^{\mu} \right\}\bm{1}\left\{ \mathop{\max}\limits_{h \in [r]} \theta_{1,s}^{h} >  \hat{\mu}_{1,s}  - \sqrt{\frac{2 \ln \left(T\Delta_i^2 + 100^{\frac{1}{3}}\right)}{s} }  \right\} \right]  \\

& = & \mathbb{E} \left[\bm{1}\left\{\mathcal{E}_{1,s}^{\mu} \right\} \cdot \mathbb{P} \left\{ \mathop{\max}\limits_{h \in [r]} \theta_{1,s}^{h} >  \hat{\mu}_{1,s}  - \sqrt{\frac{2 \ln \left(T\Delta_i^2 + 100^{\frac{1}{3}}\right)}{s} } \mid \hat{\mu}_{1,s} \right\} \right]  \\

         & =& \mathbb{E} \left[\bm{1}\left\{\mathcal{E}_{1,s}^{\mu} \right\} \cdot   \left(1-\mathbb{P} \left\{ \mathop{\max}\limits_{h \in [r]} \theta_{1,s}^{h} \le \hat{\mu}_{1,s}  -\sqrt{\frac{2 \ln \left(T \Delta_i^2  + 100^{\frac{1}{3}}\right)}{s} }\mid \hat{\mu}_{1,s} \right\} \right) \right]\\
            & =& \mathbb{E} \left[\bm{1}\left\{\mathcal{E}_{1,s}^{\mu} \right\} \cdot   \left(1-\prod\limits_{h \in [r]} \mathbb{P} \left\{ \theta_{1,s}^{h} \le \hat{\mu}_{1,s}  -\sqrt{\frac{2 \ln \left(T \Delta_i^2  + 100^{\frac{1}{3}}\right)}{s} }\mid \hat{\mu}_{1,s}\right\} \right) \right]\\
              & \ge^{(d)} & \mathbb{E} \left[\bm{1}\left\{\mathcal{E}_{1,s}^{\mu} \right\} \cdot   \left(1-\frac{0.5^r}{\left(T\Delta_i^2+100^{\frac{1}{3}}\right)^r}\right) \right]\\
                & = & \left(1-\frac{0.5^r}{\left(T\Delta_i^2+100^{\frac{1}{3}}\right)^r}\right)  \cdot  \mathbb{E} \left[\bm{1}\left\{\mathcal{E}_{1,s}^{\mu} \right\} \right]\\
      
          & \ge^{(e)} & \left(1-\frac{0.5^r}{ \left(T\Delta_i^2+100^{\frac{1}{3}}\right)^r}\right) \cdot \left(1-\frac{1}{r^2\left(T\Delta_i^2+100^{\frac{1}{3}}\right)}\right) \\
         &  \ge &1- \frac{1}{r^2 \left(T\Delta_i^2+100^{\frac{1}{3}}\right)} -\frac{0.5^r}{ \left(T\Delta_i^2+100^{\frac{1}{3}}\right)^r} \\
        &   \ge &1- \frac{1}{r^2 \left(T\Delta_i^2\right)} -\frac{0.5^r}{ \left(T\Delta_i^2\right)^r} \quad.
                   \end{array}
                   \label{ddd 6}
    \end{equation}
    
        We now provide a detailed explanation for some key steps in (\ref{ddd 6}). Inequality (a) uses the fact that if event $\mathcal{E}_{1,s}^{\mu}$ is true, we have $\mu_1 \le \hat{\mu}_{1,s} +  \sqrt{\frac{0.5 \ln \left(r^2  \left(T\Delta_i^2 + 100^{\frac{1}{3}} \right)\right)}{s}}$. Recall $L_{1,i} = \left\lceil \frac{4\left(\sqrt{2}+\sqrt{3.5}\right)^2 \ln \left(T \Delta_i^2 + 100^{\frac{1}{3}}\right)}{\Delta_i^2} \right\rceil$. Inequality (b) uses the fact that   $\frac{\Delta_i}{2} \ge \sqrt{\frac{4 \left(\sqrt{2}+\sqrt{3.5}\right)^2 \ln \left(T \Delta_i^2 + 100^{\frac{1}{3}}\right)}{4L_1} }\ge \sqrt{\frac{\left(\sqrt{2}+\sqrt{3.5}\right)^2 \ln \left(T \Delta_i^2 + 100^{\frac{1}{3}}\right)}{s} }$, when $s \ge L_{1,i}$. Recall $1 \le r \le \left\lfloor \left( T\Delta_i^2 + 100^{\frac{1}{3}} \right)^3 \right\rfloor$. Inequality (c) uses the fact that $r^2 \le \left( T\Delta_i^2 + 100^{\frac{1}{3}} \right)^6$. Inequality (d) uses Gaussian concentration bounds ( Fact~\ref{Fact 1}) and inequality (e) uses Hoeffding's inequality (Fact~\ref{Hoeffding}) giving   $\mathbb{P} \left\{\mathcal{E}_{1,s}^{\mu} \right\} \ge 1- \frac{1}{r^2 \left(T\Delta_i^2 + 100^{\frac{1}{3}} \right)}$.

   Plugging the lower bound of $\gamma$ into (\ref{ddd 444}) gives  $    \mathbb{E} \left[ \mathbb{P} \left\{\mathcal{G}_{1,s} > r \mid \hat{\mu}_{1,s}\right\} \right]       \le   \frac{1}{r^2 \left(T\Delta_i^2\right)} + \frac{0.5^r}{ \left(T\Delta_i^2\right)^r}$.

  \paragraph{When $r \ge \left\lfloor \left( T\Delta_i^2 + 100^{\frac{1}{3}} \right)^3 \right\rfloor + 1$,} we reuse the result shown in (\ref{UBC 333}) directly. Note that we have $r \ge \left\lfloor \left( T\Delta_i^2 + 100^{\frac{1}{3}} \right)^3 \right\rfloor + 1 \ge  \left( T\Delta_i^2 + 100^{\frac{1}{3}} \right)^3 \ge 101$.
\end{proof}

\subsection{Proofs for Theorem~\ref{Theorem: TS new 1}}

For a  sub-optimal arm $i$ such that  $\Delta_i \le \sqrt{\frac{1}{T}}$, we have the total regret of pulling this sub-optimal arm $i$ over $T$ rounds is at most $T\Delta_i \le \sqrt{T} \le \frac{1}{\Delta_i}$.

For a sub-optimal arm $i$ such that  $\Delta_i > \sqrt{\frac{1}{T}}$, we upper bound its expected number of pulls $\mathbb{E}\left[n_i(T) \right]$ by the end of round $T$.

Let $\bar{\mu}_{i, n_i(t-1)} := \hat{\mu}_{i, n_i(t-1)} +\sqrt{\frac{2 \ln \left(T \Delta_i^2 \right)}{n_i(t-1)}}$ be the upper confidence bound. 

Let $L_i := \left\lceil \frac{4 \left(\sqrt{0.5}+\sqrt{2} \right)^2 \ln \left(T\Delta_i^2 \right)}{\Delta_i^2} \right\rceil$. 

To decompose the regret, we define 
$\mathcal{E}_{i}^{\mu}(t-1) := \left\{ \left|\hat{\mu}_{i, n_i(t-1)} -\mu_i \right| \le \sqrt{\frac{0.5 \ln \left(T \Delta_i^2 \right)}{n_i(t-1)}} \right\}$ and   
$\mathcal{E}_i^{\theta}(t) := \left\{\theta_i(t) \le \bar{\mu}_{i, n_i(t-1)}\right\}$. 
Then, we  have
\begin{equation}
    \begin{array}{lll}
         \mathbb{E} \left[n_i(T) \right] 
       &   \le & L_i + \sum\limits_{t=1}^{T} \mathbb{E} \left[\bm{1} \left\{i_t = i, n_i(t-1) \ge L_i \right\} \right]  \\
        &   \le & L_i + \underbrace{\sum\limits_{t=1}^{T} \mathbb{E} \left[\bm{1} \left\{i_t = i,\mathcal{E}^{\theta}_i(t), \mathcal{E}_i^{\mu}(t-1),  n_i(t-1) \ge L_i \right\} \right] }_{=:\omega_1} \\
         &  + &\underbrace{\sum\limits_{t=1}^{T} \mathbb{E} \left[\bm{1} \left\{i_t = i, \overline{\mathcal{E}^{\theta}_i(t)},n_i(t-1) \ge L_i \right\} \right]}_{=:\omega_2} + \underbrace{  \sum\limits_{t=1}^{T} \mathbb{E} \left[\bm{1} \left\{i_t = i, \overline{\mathcal{E}_i^{\mu}(t-1)},n_i(t-1) \ge L_i \right\} \right]}_{=:\omega_3} 
           \quad.
           \end{array}
           \label{Eric 1}
           \end{equation}

           Term $\omega_1$ is very similar to  Lemma~2.14 in \cite{agrawalnear}, which is challenging to derive a tighter upper bound. Terms $\omega_2$ (similar to Lemma~2.16 in \cite{agrawalnear}) and  $\omega_3$ (similar to Lemma~2.15 in \cite{agrawalnear}) are not difficult to bound. 
We use Gaussian concentration inequality (Fact~\ref{Fact 1}) for upper bounding term $\omega_2$ and Hoeffding's inequality (Fact~\ref{Hoeffding})  for upper bounding terms  $\omega_3$.  From Lemma~\ref{Eric 3} and Lemma~\ref{Eric 2} we have $\omega_2 \le \frac{0.5}{\Delta_i^2}$ and  $\omega_3 \le \frac{2}{\Delta_i^2}$.

           For $\omega_1$, similar to Lemma~2.8 in \cite{agrawalnear}, we have our Lemma~\ref{bandit: lemma 1}, a  lemma that links the probability of pulling a sub-optimal arm $i$ to the probability of pulling the optimal arm $1$ in round $t$ by introducing the upper confidence bound $\bar{\mu}_{i, n_i(t-1)}$.  Note that both events $ \mathcal{E}_i^{\mu}(t-1)$ and $  n_i(t-1) \ge L_i$ are determined by the history information. The value of $\bar{\mu}_{i, n_i(t-1)}$ is  determined by the history information. Also, the distributions for $\theta_j(t)$ for all $j \in \mathcal{A}$ are determined by the history information. 

 \begin{lemma}
             For any instantiation $F_{t-1}$ of $\mathcal{F}_{t-1}$, we have
             \begin{equation}
                 \begin{array}{lll}
                       \mathbb{E} \left[\bm{1} \left\{i_t = i, \mathcal{E}^{\theta}_i(t)\right\} \mid \mathcal{F}_{t-1}= F_{t-1}\right]
               &    \le & \frac{\mathbb{P} \left\{ \theta_1(t) \le \bar{\mu}_{i, n_i(t-1)}\mid \mathcal{F}_{t-1}= F_{t-1} \right\}}{\mathbb{P} \left\{ \theta_1(t) > \bar{\mu}_{i, n_i(t-1)}\mid \mathcal{F}_{t-1}= F_{t-1} \right\} } \cdot \mathbb{E} \left[\bm{1} \left\{i_t = 1, \mathcal{E}_i^{\theta}(t)\right\} \mid \mathcal{F}_{t-1}= F_{t-1}\right].
                 \end{array}
             \end{equation}
             \label{bandit: lemma 1}
         \end{lemma}

With Lemma~\ref{bandit: lemma 1} in hand, now, we are ready to upper bound term $\omega_1$. We have
     \begin{equation}
             \begin{array}{lll}
                  \omega_1 & = & \sum\limits_{t=1}^{T} \mathbb{E} \left[\bm{1} \left\{i_t = i, \mathcal{E}^{\theta}_i(t), \mathcal{E}_i^{\mu}(t-1),  n_i(t-1) \ge L_i \right\} \right]  \\ 
                  & =& \sum\limits_{t=1}^{T} \mathbb{E} \left[\bm{1} \left\{ \mathcal{E}_i^{\mu}(t-1),  n_i(t-1) \ge L_i \right\} \cdot \underbrace{\mathbb{E} \left[\bm{1} \left\{i_t = i, \mathcal{E}^{\theta}_i(t) \right\} \mid \mathcal{F}_{t-1}\right]}_{\text{LHS in Lemma~\ref{bandit: lemma 1}}} \right] \\
   
                     & \le &  \sum\limits_{t=1}^{T}\mathbb{E} \left[ \underbrace{\bm{1} \left\{ \mathcal{E}_i^{\mu}(t-1),  n_i(t-1) \ge L_i \right\}  
 \underbrace{\frac{\mathbb{P} \left\{ \theta_1(t) \le \bar{\mu}_{i, n_i(t-1)}\mid \mathcal{F}_{t-1} \right\}}{\mathbb{P} \left\{ \theta_1(t) > \bar{\mu}_{i, n_i(t-1)}\mid \mathcal{F}_{t-1} \right\} }  \mathbb{E} \left[\bm{1} \left\{i_t = 1, \mathcal{E}_i^{\theta}(t)\right\} \mid \mathcal{F}_{t-1}\right]}_{\text{RHS in Lemma~\ref{bandit: lemma 1}}}}_{\eta} \right]\\
 &&\\
& \le^{(a)} &  \sum\limits_{t=1}^{T}\mathbb{E} \left[  \mathbb{E} \left[ \frac{\mathbb{P} \left\{ \theta_1(t) \le \mu_1 - \frac{\Delta_i}{2} \mid \mathcal{F}_{t-1} \right\}}{\mathbb{P} \left\{ \theta_1(t) > \mu_1 - \frac{\Delta_i}{2}\mid \mathcal{F}_{t-1} \right\} } \cdot \bm{1} \left\{i_t = 1 \right\} \mid \mathcal{F}_{t-1}\right]  \right]\\
& = &  \sum\limits_{t=1}^{T}\mathbb{E} \left[  \frac{\mathbb{P} \left\{ \theta_1(t) \le \mu_1 - \frac{\Delta_i}{2} \mid \mathcal{F}_{t-1} \right\}}{\mathbb{P} \left\{ \theta_1(t) > \mu_1 - \frac{\Delta_i}{2}\mid \mathcal{F}_{t-1} \right\} } \cdot \bm{1} \left\{i_t = 1\right\}  \right]\quad.
             \end{array}
             \label{Beta 1}
         \end{equation}
         Inequality (a) in (\ref{Beta 1}) uses the argument that for a specific $F_{t-1}$ of $\mathcal{F}_{t-1}$ such that either event $ \mathcal{E}_i^{\mu}(t-1)$ or $   n_i(t-1) \ge L_i$ does not occur, term $\eta$ in (\ref{Beta 1}) will be $0$. Note that for any $F_{t-1}$, we have $1 < \frac{1}{\mathbb{P} \left\{ \theta_1(t) > \bar{\mu}_{i, n_i(t-1)}\mid \mathcal{F}_{t-1}= F_{t-1} \right\}} < +\infty$.
                  Recall $L_i = \left\lceil \frac{4 \left(\sqrt{0.5}+\sqrt{2} \right)^2 \ln \left(T\Delta_i^2 \right)}{\Delta_i^2} \right\rceil$. For any specific $F_{t-1}$ such that both events $ \mathcal{E}_i^{\mu}(t-1)$ and $ n_i(t-1) \ge L_i$ occur, we have 
         $\bar{\mu}_{i, n_i(t-1)} = \hat{\mu}_{i, n_i(t-1)} + \sqrt{\frac{2 \ln \left(T \Delta_i^2 \right)}{n_i(t-1)}} \le \mu_i +\sqrt{\frac{0.5 \ln \left(T \Delta_i^2 \right)}{n_i(t-1)}}+ \sqrt{\frac{2 \ln \left(T \Delta_i^2 \right)}{n_i(t-1)}} \le \mu_i +\sqrt{\frac{0.5 \ln \left(T \Delta_i^2 \right)}{L_i}}+ \sqrt{\frac{2 \ln \left(T \Delta_i^2 \right)}{L_i}} \le \mu_i + \frac{\Delta_i}{2}= \mu_1 - \frac{\Delta_i}{2}$, which implies $\frac{\mathbb{P} \left\{ \theta_1(t) \le \bar{\mu}_{i, n_i(t-1)}\mid \mathcal{F}_{t-1}= F_{t-1} \right\}}{\mathbb{P} \left\{ \theta_1(t) > \bar{\mu}_{i, n_i(t-1)}\mid \mathcal{F}_{t-1}=F_{t-1} \right\} } \le \frac{\mathbb{P} \left\{ \theta_1(t) \le \mu_1 -\frac{\Delta_i}{2}\mid \mathcal{F}_{t-1}= F_{t-1} \right\}}{\mathbb{P} \left\{ \theta_1(t) > \mu_1 - \frac{\Delta_i}{2}\mid \mathcal{F}_{t-1}=F_{t-1} \right\} }$.

     Now, we partition all $T$ rounds into multiple intervals based on the arrivals of the observations of the optimal arm $1$. Let $\tau_s^{(1)}$ be the round when arm $1$ is pulled for the $s$-th time.  Note that this partition in time horizon ensures that the posterior distribution of arm $1$ stays the same among all the rounds when $t \in \left\{ \tau_s^{(1)}+1, \dotsc, \tau_{s+1}^{(1)}\right\}$. 
Recall $L_{1,i} = \frac{4\left(\sqrt{2}+\sqrt{3.5}\right)^2 \ln \left(T \Delta_i^2 \right)}{\Delta_i^2}$ and $\theta_{1,s} \sim \mathcal{N} \left(\hat{\mu}_{1,s}, \ \frac{1}{s} \right)$.
Then, we have
         \begin{equation}
             \begin{array}{lll}
                \omega_1 & \le  & \sum\limits_{t=1}^{T} 
 \mathbb{E} \left[  \frac{\mathbb{P} \left\{ \theta_1(t) \le \mu_1 - \frac{\Delta_i}{2} \mid \mathcal{F}_{t-1} \right\}}{\mathbb{P} \left\{ \theta_1(t) > \mu_1 - \frac{\Delta_i}{2}\mid \mathcal{F}_{t-1} \right\} } \cdot \bm{1} \left\{i_t = 1 \right\}  \right]  \\

             &  \le   & \mathbb{E} \left[ \sum\limits_{s=1}^{T} \sum\limits_{t=\tau_s^{(1)}+1}^{\tau_{s+1}^{(1)}}   
 \frac{\mathbb{P} \left\{ \theta_1(t) \le \mu_1 - \frac{\Delta_i}{2} \mid \mathcal{F}_{t-1} \right\}}{\mathbb{P} \left\{ \theta_1(t) > \mu_1 - \frac{\Delta_i}{2}\mid \mathcal{F}_{t-1} \right\} } \cdot \bm{1} \left\{i_t = 1 \right\}  \right]  \\
 
              & \le   & \sum\limits_{s=1}^{T}\mathbb{E} \left[  
 \frac{\mathbb{P} \left\{ \theta_1\left(\tau_{s+1}^{(1)}\right) \le \mu_1 - \frac{\Delta_i}{2} \mid \mathcal{F}_{\tau_{s+1}^{(1)}-1} \right\}}{\mathbb{P} \left\{ \theta_1\left(\tau_{s+1}^{(1)}\right) > \mu_1 - \frac{\Delta_i}{2}\mid \mathcal{F}_{\tau_{s+1}^{(1)}-1} \right\} }   \right]\\
             &      \le   & \sum\limits_{s=1}^{L_{1,i}}\mathbb{E} \left[\frac{1}{\mathbb{P} \left\{ \theta_{1,s} > \mu_1 - \frac{\Delta_i}{2}\mid \mathcal{F}_{\tau_{s}^{(1)}} \right\} }  -1 \right] + \sum\limits_{s=L_{1,i}+1}^{T}\mathbb{E} \left[  \frac{1}{\mathbb{P} \left\{\theta_{1,s} > \mu_1 - \frac{\Delta_i}{2}\mid \mathcal{F}_{\tau_{s}^{(1)}} \right\} }  -1 \right]  \quad. 
                    \end{array}
                    \label{UBC 1}
                    \end{equation}

With Lemma~\ref{UBC 22} in hand, we have
\begin{equation}
    \begin{array}{lll}
         \omega_1 & \le & \sum\limits_{s=1}^{L_{1,i}}\mathbb{E} \left[\frac{1}{\mathbb{P} \left\{ \theta_{1,s} > \mu_1 - \frac{\Delta_i}{2}\mid \mathcal{F}_{\tau_{s}^{(1)}} \right\} }  -1 \right] + \sum\limits_{s=L_{1,i}+1}^{T}\mathbb{E} \left[ \frac{1}{\mathbb{P} \left\{ \theta_{1,s} > \mu_1 - \frac{\Delta_i}{2}\mid \mathcal{F}_{\tau_{s}^{(1)}}\right\} }  -1 \right]  \\

         & \le & 29 \cdot L_{1,i} +  \sum\limits_{s=1}^{T} \frac{180}{T\Delta_i^2}\\
         & \le & \frac{116\left(\sqrt{2}+\sqrt{3.5}\right)^2 \ln \left(T \Delta_i^2 + 100^{\frac{1}{3}}\right)}{\Delta_i^2} + \frac{180}{\Delta_i^2}\quad.  
    \end{array}
    \label{Beta 3}
\end{equation}

Plugging the upper bounds of $\omega_1, \omega_2$, and  $\omega_3$ together in (\ref{Eric 1}), we have
\begin{equation}
    \begin{array}{lll}
          \begin{array}{lll}
        \mathbb{E} \left[n_i(T) \right]  &  
       \le   & L_i + \omega_1 + \omega_2 + \omega_3 \\
       & \le & \frac{4 \left(\sqrt{0.5}+\sqrt{2} \right)^2 \ln \left(T\Delta_i^2 \right)}{\Delta_i^2} + 1 +\frac{116\left(\sqrt{2}+\sqrt{3.5}\right)^2 \ln \left(T \Delta_i^2 + 100^{\frac{1}{3}}\right)}{\Delta_i^2} + \frac{180}{\Delta_i^2} + \frac{0.5}{\Delta_i^2} + \frac{2}{\Delta_i^2} \\
       & \le & \frac{1270 \ln \left(T \Delta_i^2 + 100^{\frac{1}{3}}\right)}{\Delta_i^2} + \frac{182.5}{\Delta_i^2} + 1 \quad.
    \end{array}
    \end{array}
\end{equation}    
which concludes the proof.

\paragraph{Proofs for the worst-case regret bound.} We reuse the arguments shown in \cite{agrawalnear}. Let $\Delta_* := e \sqrt{\frac{K \ln K}{T}}$ be the critical gap. 
The total regret of pulling any sub-optimal arm $i$ such that  $\Delta_i < \Delta_*$ is at most $T \cdot \Delta_* = e \sqrt{KT\ln(K)}$.
Now, we consider all the remaining sub-optimal arms $i$ with  $\Delta_i > \Delta_*$. For  such $i$, the regret  is upper bounded by 
\[\Delta_i \mathbb{E} \left[n_i(T) \right] \leq \frac{1270 \ln \left(T \Delta_i^2 + 100^{\frac{1}{3}}\right)}{\Delta_i} + \frac{182.5}{\Delta_i} + \Delta_i, \]
which is decreasing in $\Delta_i \in (\frac{e}{\sqrt{T}},1]$. Therefore, for every such
$i$, the regret is bounded by
$O\left(\sqrt{\frac{T \ln K}{K}}\right)$. Taking a sum over all sub-optimal arms concludes the proofs.


  \begin{proof}[Proof of Lemma~\ref{bandit: lemma 1}]
The proof is very similar to the proof of Lemma~2.8 in \cite{agrawalnear}.     Note that the value of $ \bar{\mu}_{i, n_i(t-1)}$ is determined by the history information. 
 
       We have the following two pieces of argument. 

         The first piece is
         \begin{equation}
             \begin{array}{ll}
                  &  \mathbb{E} \left[\bm{1} \left\{i_t = i, \mathcal{E}^{\theta}_i(t) \right\} \mid \mathcal{F}_{t-1}= F_{t-1}\right] \\

\le   &\mathbb{E} \left[\bm{1} \left\{  \theta_j(t) \le  \bar{\mu}_{i, n_i(t-1)}, \forall j \in \mathcal{A}
 \right\} \mid \mathcal{F}_{t-1}= F_{t-1}\right]  \\
=   & \mathbb{P} \left\{  \theta_1(t) \le  \bar{\mu}_{i, n_i(t-1)} \mid \mathcal{F}_{t-1}= F_{t-1}
 \right\}  
\underbrace{\mathbb{P} \left\{  \theta_j(t) \le  \bar{\mu}_{i, n_i(t-1)}, \forall j \in \mathcal{A} \setminus \{1\}\mid \mathcal{F}_{t-1}= F_{t-1}
 \right\} }_{> 0}  \quad.        
             \end{array}
             \label{taco 1}
         \end{equation}

         The second piece is
         \begin{equation}
             \begin{array}{ll}
                  & \mathbb{E} \left[\bm{1} \left\{i_t = 1, \mathcal{E}_i^{\theta}(t) \right\} \mid \mathcal{F}_{t-1}= F_{t-1}\right]  \\
                   \ge & \mathbb{E} \left[\bm{1} \left\{\theta_1(t) >  \bar{\mu}_{i, n_i(t-1)} \ge \theta_j(t), \forall j \in \mathcal{A} \setminus \{1\}\right\} \mid \mathcal{F}_{t-1}= F_{t-1}\right]  \\
                   = & \underbrace{\mathbb{P} \left\{  \theta_1(t) >  \bar{\mu}_{i, n_i(t-1)}\mid \mathcal{F}_{t-1}= F_{t-1}
 \right\} }_{> 0} 
 \underbrace{\mathbb{P} \left\{  \theta_j(t) \le  \bar{\mu}_{i, n_i(t-1)}, \forall j \in \mathcal{A} \setminus \{1\}\mid \mathcal{F}_{t-1}= F_{t-1}
 \right\} }_{> 0}  \quad.
             \end{array}
             \label{taco 2}
         \end{equation}
         Combining (\ref{taco 1}) and (\ref{taco 2}) concludes the proof.
             
         \end{proof}

\begin{lemma}
    We have
    \begin{equation}
        \begin{array}{lll}
             \sum\limits_{t=1}^{T} \mathbb{E} \left[\bm{1} \left\{i_t = i, \overline{\mathcal{E}^{\theta}_i(t)},   n_i(t-1) \ge L_i \right\} \right]  & \le & \frac{0.5}{\Delta_i^2}\quad.
        \end{array}
    \end{equation}
 \label{Eric 3}
\end{lemma}

\begin{proof}[Proof of Lemma~\ref{Eric 3}]
This lemma is very similar to Lemma~2.16 in \cite{agrawalnear}.
Recall event $\mathcal{E}_i^{\theta}(t) = \left\{\theta_i(t) \le \hat{\mu}_{i, n_i(t-1)} + \sqrt{\frac{2 \ln \left(T \Delta_i^2 \right)}{n_i(t-1)}}\right\}$.
Let $\tau_s^{(i)}$ be the round when arm $i$ is pulled for the $s$-th time. Now, we partition all $T$ rounds based on the pulls of arm $i$. We have 
        \begin{equation}
        \begin{array}{ll}
          &   \sum\limits_{t=1}^{T} \mathbb{E} \left[\bm{1} \left\{i_t = i, \overline{\mathcal{E}^{\theta}_i(t)},  n_i(t-1) \ge L_i \right\} \right]  \\
           \le & \sum\limits_{s=L_i}^{T} \mathbb{E} \left[\sum\limits_{t= \tau_s^{(i)}+1}^{\tau_{s+1}^{(i)}}\bm{1} \left\{i_t = i, \overline{\mathcal{E}^{\theta}_i(t)} \right\} \right] \\
            \le & \sum\limits_{s=L_i}^{T} \mathbb{E} \left[\bm{1} \left\{ \overline{\mathcal{E}^{\theta}_i\left(\tau_{s+1}^{(i)}\right)} \right\} \right] \\
             = & \sum\limits_{s=L_i}^{T} \mathbb{E} \left[\bm{1} \left\{ \theta_{i,s} > \hat{\mu}_{i,s} + \sqrt{\frac{2 \ln (T\Delta_i^2)}{s}} \right\} \right] \\
              = & \sum\limits_{s=L_i}^{T}\mathbb{E} \left[\underbrace{ \mathbb{P}  \left\{ \theta_{i,s} > \hat{\mu}_{i,s} + \sqrt{\frac{2 \ln (T\Delta_i^2)}{s}} \mid \hat{\mu}_{1,s} \right\}}_{\text{Fact~\ref{Fact 1}}} \right] \\
             \le & \sum\limits_{s=L_i}^{T} \frac{1}{2} e^{-0.5 \cdot 2 T\Delta_i^2} \\
             \le & \frac{0.5}{\Delta_i^2}\quad,
        \end{array}
    \end{equation}
    which concludes the proof.
\end{proof}

\begin{lemma}
    We have
    \begin{equation}
        \begin{array}{lll}
               \sum\limits_{t=1}^{T} \mathbb{E} \left[\bm{1} \left\{i_t = i, \overline{\mathcal{E}_i^{\mu}(t-1)}, n_i(t-1) \ge L_i \right\} \right] & \le & \frac{2}{\Delta_i^2}\quad.
              
        \end{array}
    \end{equation}
    \label{Eric 2}
\end{lemma}

\begin{proof}[Proof of Lemma~\ref{Eric 2}]

Recall event $\mathcal{E}_{i}^{\mu}(t-1) = \left\{ \left|\hat{\mu}_{i, n_i(t-1)} -\mu_i \right| \le \sqrt{\frac{0.5 \ln \left(T \Delta_i^2 \right)}{n_i(t-1)}} \right\}$.
Let $\tau_s^{(i)}$ be the round when arm $i$ is pulled for the $s$-th time. Now, we partition all $T$ rounds based on the pulls of arm $i$. We have
    \begin{equation}
        \begin{array}{ll}
            &   \sum\limits_{t=1}^{T} \mathbb{E} \left[\bm{1} \left\{i_t = i, \overline{\mathcal{E}_i^{\mu}(t-1)}, n_i(t-1) \ge L_i \right\} \right] \\
             \le &  \sum\limits_{s=L_i}^{T} \mathbb{E} \left[\sum\limits_{t= \tau_s^{(i)}+1}^{\tau_{s+1}^{(i)}}\bm{1} \left\{i_t = i, \overline{\mathcal{E}_i^{\mu}(t-1)} \right\} \right] \\
            \le &  \sum\limits_{s=L_i}^{T} \mathbb{E} \left[\bm{1} \left\{ \overline{\mathcal{E}_i^{\mu}\left(\tau_{s+1}^{(i)}-1 \right)} \right\} \right] \\ 
            = &  \sum\limits_{s=L_i}^{T} \underbrace{ \mathbb{P} \left\{ \left|\hat{\mu}_{i, s} -\mu_i \right| > \sqrt{\frac{0.5 \ln \left(T \Delta_i^2 \right)}{s}} \right\}}_{\text{Hoeffding's inequality}}  \\
            \le &  \sum\limits_{s=L_i}^{T}  \frac{2}{(T\Delta_i^2)} \\
            \le & \frac{2}{\Delta_i^2}\quad,
        \end{array}
    \end{equation}
    which concludes the proof.\end{proof}

\section{Proofs for Theorem~\ref{Theorem: Alpha-TS}} \label{app: TS-MA}

\begin{proof}[Proof of Theorem~\ref{Theorem: Alpha-TS}]

For a fixed sub-optimal arm $i \in \mathcal{A}$, we upper bound the expected number of pulls $\mathbb{E} \left[n_i(T) \right]$ by the end of round $T$. 
Let $L_i := \left\lceil \frac{\left(\sqrt{0.5}+\sqrt{5-\alpha}\right)^2 \ln^{1+\alpha} (T)}{\Delta_i^2} \right\rceil$. 
Let $\theta_i(t) := \mathop{\max}\limits_{h \in [\phi]} \theta_{i, n_i(t-1)}^{h}$ be the maximum of $\phi$ i.i.d. posterior samples, where each $\theta_{i, n_i(t-1)}^{h} \sim \mathcal{N} \left(\hat{\mu}_{i, n_i(t-1)}, \frac{\ln^{\alpha}(T)}{n_i(t-1)} \right)$.
We decompose the regret and have 
\begin{equation}
    \begin{array}{lll}
         \mathbb{E} \left[n_i(T) \right] & = & \sum\limits_{t=1}^{T} \mathbb{E} \left[\bm{1} \left\{i_t = i \right\} \right]  \\
         & \le & L_i + \sum\limits_{t=1}^{T} \mathbb{E} \left[\bm{1} \left\{i_t = i, n_i(t-1) \ge L_i \right\} \right]  \\
          & \le & L_i + \sum\limits_{t=1}^{T} \mathbb{E} \left[\bm{1} \left\{i_t = i, \theta_i(t) \ge \theta_1(t), n_i(t-1) \ge L_i \right\} \right] \\
          & \le & L_i + \underbrace{\sum\limits_{t=1}^{T} \mathbb{E} \left[\bm{1} \left\{i_t = i, \theta_i(t) \ge \mu_i+\Delta_i, n_i(t-1) \ge L_i \right\} \right]  }_{=:\omega_1, \text{over-estimation of sub-optimal arm $i$}} 
          + \underbrace{\sum\limits_{t=1}^{T} \mathbb{E} \left[\bm{1} \left\{\theta_1(t) < \mu_1\right\} \right]}_{=:\omega_2, \text{under-estimation of  optimal arm $1$}}.
          \end{array}
          \end{equation}

    \textbf{Over-estimation of the sub-optimal arm $i$.} Define  $\mathcal{E}_{i}^{\mu}(t-1)$ as the event that $\left|\hat{\mu}_{i, n_i(t-1)} - \mu_i \right| \le \sqrt{\frac{0.5 \ln (T)}{n_i(t-1)}}$.
Define $\tau_s^{(i)}$ as the round when arm $i$ is pulled for the $s$-th time. Based on the  definition, we know $\hat{\mu}_{i, n_i(t-1)}$ stays the same for all rounds $t \in \left\{\tau_s^{(i)} +1, \dotsc, \tau_{s+1}^{(i)} \right\}$. 
We have
        \begin{equation}
          \begin{array}{lll}
            
          \omega_1 & = &  \sum\limits_{t=1}^{T} \mathbb{E} \left[\bm{1} \left\{i_t = i, \theta_i(t) \ge \mu_i+\Delta_i, n_i(t-1) \ge L_i \right\} \right] \\
          & \le &  \sum\limits_{t=1}^{T} \mathbb{E} \left[\bm{1} \left\{i_t = i, \mathcal{E}_{i}^{\mu}(t-1) , \theta_i(t) \ge \mu_i+\Delta_i, n_i(t-1) \ge L_i \right\} \right] + \sum\limits_{t=1}^{T} \mathbb{E} \left[\bm{1} \left\{i_t = i, \overline{ \mathcal{E}_{i}^{\mu}(t-1)}, n_i(t-1) \ge L_i  \right\} \right]  \\
          &\le &\sum\limits_{s=L_i}^{T}\mathbb{E} \left[\sum\limits_{t=\tau_s^{(i)} +1}^{\tau_{s+1}^{(i)}} \bm{1} \left\{i_t = i, \mathcal{E}_{i}^{\mu}(t-1) , \theta_i(t) \ge \mu_i+\Delta_i\right\} \right] + \sum\limits_{s=L_i}^{T}\mathbb{E} \left[\sum\limits_{t=\tau_s^{(i)} +1}^{\tau_{s+1}^{(i)}} \bm{1} \left\{i_t = i, \overline{ \mathcal{E}_{i}^{\mu}(t-1)}\right\} \right]  \\
          & \le &\sum\limits_{s=L_i}^{T} \mathbb{E} \left[\bm{1} \left\{\theta_i\left(\tau_{s+1}^{(i)} \right) \ge \mu_i + \Delta_i, \mathcal{E}_{i}^{\mu}\left(\tau_{s+1}^{(i)}-1\right) \right\} \right] + \sum\limits_{s=L_i}^{T}\mathbb{E} \left[ \bm{1} \left\{ \overline{ \mathcal{E}_{i}^{\mu}\left(\tau_{s+1}^{(i)}-1\right)}\right\} \right]    \\ 
           & = &\sum\limits_{s=L_i}^{T} \mathbb{E} \left[\bm{1} \left\{\mathop{\max}\limits_{h \in [\phi]}\theta_{i,s}^{h} \ge \mu_i + \Delta_i,  \left| \hat{\mu}_{1,s} - \mu_i \right| \le \sqrt{\frac{0.5 \ln (T)}{s}}  \right\}\right]+ \sum\limits_{s=L_i}^{T} \underbrace{\mathbb{P}  \left\{ \left| \hat{\mu}_{1,s} - \mu_i \right| > \sqrt{\frac{0.5 \ln (T)}{s}} \right\} }_{\le 2e^{-\ln(T)}, \text{ Fact~\ref{Hoeffding}}}  \\
           & \le^{(a)} & \sum\limits_{s=L_i}^{T}\mathbb{P}  \left\{\mathop{\max}\limits_{h \in [\phi]}\theta_{i,s}^{h} \ge \hat{\mu}_{i,s} - \sqrt{\frac{0.5 \ln (T)}{s}} + \Delta_i \right\} + \sum\limits_{t=1}^{T} 2e^{-\ln(T)} \\
           & \le &\sum\limits_{s=L_i}^{T} \underbrace{\mathbb{P}  \left\{\mathop{\max}\limits_{h \in [\phi]}\theta_{i,s}^{h} \ge \hat{\mu}_{i,s} - \sqrt{\frac{0.5 \ln^{1+\alpha} (T)}{s}} + \Delta_i \right\}}_{I_1} + 2
           \quad.
    \end{array}
    \label{Banana 1}
\end{equation}

To upper bound term $I_1$, we recall $L_i = \left\lceil \frac{\left(\sqrt{0.5}+\sqrt{5-\alpha} \right)^2 \ln^{1+\alpha} (T)}{\Delta_i^2} \right\rceil$. Then, we have
$ \Delta_i \ge \sqrt{ \frac{\left(\sqrt{0.5}+\sqrt{5-\alpha}\right)^2  \ln^{1+\alpha}(T)}{L_i} } \ge \sqrt{ \frac{\left(\sqrt{0.5}+\sqrt{5-\alpha} \right)^2  \ln^{1+\alpha}(T)}{s} }$ for any $s \ge L_i$. 

Now, we have
\begin{equation}
    \begin{array}{lll}
      I_1 & =   & \mathbb{P}  \left\{\mathop{\max}\limits_{h \in [\phi]}\theta_{i,s}^{h} \ge \hat{\mu}_{1,s} - \sqrt{\frac{0.5 \ln^{1+\alpha} (T)}{s}}+ \Delta_j \right\} \\
       &  \le & \mathbb{P}  \left\{\mathop{\max}\limits_{h \in [\phi]}\theta_{i,s}^{h} \ge \hat{\mu}_{1,s} - \sqrt{\frac{0.5 \ln^{1+\alpha} (T)}{s}} + \sqrt{ \frac{ \left(\sqrt{0.5}+\sqrt{5-\alpha} \right)^2  \ln^{1+\alpha}(T)}{s} } \right\} \\
        
        & \le & \sum\limits_{h=1}^{\phi} \underbrace{ \mathbb{P}  \left\{\theta^{h}_{i,s} \ge \hat{\mu}_{1,s} + \sqrt{ (5-\alpha)\ln(T)  }\sqrt{ \frac{ \ln^{\alpha} (T)}{s} } \right\}}_{\text{Fact~\ref{Fact 1}}} \\ 
      &   \le & \phi \cdot \frac{1}{2T^{2.5-0.5\alpha}}\\
      &&\\
      & = & \frac{2}{c_0} \cdot \ln^{1.5-0.5\alpha}(T)  \cdot T^{ 0.5(1-\alpha)}   \cdot \frac{1}{2T^{2.5-0.5\alpha}} \\
      &&\\

            & \le & \frac{1}{c_0} \cdot \frac{1}{T} \cdot \frac{\ln^{1.5}(T)}{T} \\
            &&\\
      & \le^{(a)} & \frac{1}{c_0 \cdot T} \quad,
    \end{array}
\end{equation}
where inequality (a) uses the fact that $T \ge \ln^{1.5}(T)$, when $T \ge 1$.

By plugging the upper bound of $I_1$  into (\ref{Banana 1}), we have $\omega_1 \le     \frac{1}{c_0}+2$.

\paragraph{Under-estimation of the optimal arm $1$.} Define  $\mathcal{E}_{1}^{\mu}(t-1)$ as the event that $\left|\hat{\mu}_{1, n_1(t-1)} - \mu_1 \right| \le \sqrt{\frac{ \ln (T)}{n_1(t-1)}}$. Define $\tau_s^{(1)}$ as the round when arm $1$ is pulled for the $s$-th time. Based on the  definition, we know $\hat{\mu}_{1, n_1(t-1)}$ stays the same for all rounds $t \in \left\{\tau_s^{(1)} +1, \dotsc, \tau_{s+1}^{(1)} \right\}$.  Still, 
Recall $\phi =  \frac{2}{c_0} T^{\frac{1-\alpha}{2}} \ln^{\frac{3-\alpha}{2}}(T) $. We have
\begin{equation}
    \begin{array}{lll}
       \omega_2 &=  & \sum\limits_{t=1}^{T} \mathbb{E} \left[\bm{1} \left\{\theta_1(t) < \mu_1 \right\} \right]  \\
       &\le   & \sum\limits_{s=1}^{T} \mathbb{E} \left[\sum\limits_{t =\tau_s^{(1)}+1}^{\tau_{s+1}^{(1)}}\bm{1} \left\{\theta_1(t) < \mu_1 \right\} \right]  \\
  &   \le    & \sum\limits_{s=1}^{T} \mathbb{E} \left[\sum\limits_{t =\tau_s^{(1)}+1}^{\tau_{s+1}^{(1)}}\bm{1} \left\{\theta_1(t) < \mu_1, \mathcal{E}_1^{\mu}(t-1) \right\} \right] + \sum\limits_{s=1}^{T} \mathbb{E} \left[\sum\limits_{t =\tau_s^{(1)}+1}^{\tau_{s+1}^{(1)}}\bm{1} \left\{ \overline{\mathcal{E}_1^{\mu}(t-1)} \right\} \right] \\
   &   \le   & \sum\limits_{s=1}^{T} \mathbb{E} \left[\sum\limits_{t =\tau_s^{(1)}+1}^{\tau_{s+1}^{(1)}}\bm{1} \left\{\theta_1(t) < \hat{\mu}_{1, s} + \sqrt{\frac{\ln(T)}{s}}\right\} \right] + \sum\limits_{s=1}^{T} \mathbb{E} \left[\sum\limits_{t =\tau_s^{(1)}+1}^{\tau_{s+1}^{(1)}}\bm{1} \left\{\left|\hat{\mu}_{1, s} - \mu_1 \right| > \sqrt{\frac{ \ln (T)}{s}}\right\} \right] \\
    &   \le   & \sum\limits_{s=1}^{T} \mathbb{E} \left[\sum\limits_{t =\tau_s^{(1)}+1}^{\tau_{s+1}^{(1)}}  \bm{1} \left\{\mathop{\max}\limits_{h \in [\phi]} \theta_{1,s}^{h} < \hat{\mu}_{1, s} + \sqrt{\frac{\ln(T)}{s}} \right\}   \right] + \sum\limits_{s=1}^{T}\sum\limits_{t=1}^{T}2 e^{-2\ln (T)}  \\
        &   \le   & \sum\limits_{s=1}^{T} \mathbb{E} \left[\sum\limits_{t =1}^{T}  \mathbb{P} \left\{\mathop{\max}\limits_{h \in [\phi]} \theta_{1,s}^{h} < \hat{\mu}_{1, s} + \sqrt{\frac{\ln(T)}{s}} \mid \hat{\mu}_{1, s} \right\}   \right] + 2 \\
            &   \le   & \sum\limits_{s=1}^{T} \sum\limits_{t =1}^{T} \mathbb{E} \left[\prod\limits_{h \in [\phi]} \underbrace{ \mathbb{P} \left\{\theta_{1,s}^{h} < \hat{\mu}_{1, s} + \sqrt{\frac{\ln(T)}{s}} \mid \hat{\mu}_{1, s} \right\}  }_{\eta} \right] + 2 \\
&&\\
        & \le^{(a)} & \sum\limits_{s=1}^{T}\sum\limits_{t=1}^{T} \left(1-c_0 \cdot\frac{1}{\sqrt{\ln^{(1-\alpha)}(T) \cdot T^{ (1-\alpha)}  } }\right)^{\phi} +2 \\

        &&\\
        & \le^{(b)} &T^2 \cdot  e^{-c_0 \cdot \frac{1}{\sqrt{\ln^{(1-\alpha)}(T) \cdot T^{ (1-\alpha)}  } } \cdot \phi} + 2\\
        & \le & 3\quad,
  
    \end{array}
    \label{Apple 1}
\end{equation}
where inequality (b) uses the fact $1-x \le e^{-x}$, where $x = c_0 \cdot\frac{1}{\sqrt{\ln^{(1-\alpha)}(T) \cdot T^{ (1-\alpha)}  } }$. Inequality (a) in (\ref{Apple 1}) uses the following result, which is
\begin{equation}
    \begin{array}{lll}
       \eta & =  &  \mathbb{P} \left\{\theta_{1,s}^{h} < \hat{\mu}_{1, s} + \sqrt{\frac{\ln(T)}{s}} \mid \hat{\mu}_{1,s} \right\}\\
       &=  & 1-\mathbb{P} \left\{\theta_{1,s}^{h} \ge \hat{\mu}_{1, s} + \sqrt{\ln^{(1-\alpha)}(T) } \cdot \sqrt{\frac{\ln^{\alpha}(T)}{s}} \mid \hat{\mu}_{1, s}\right\} \\
     &  \le & 1 - \frac{1}{\sqrt{2\pi}} \cdot \frac{\sqrt{\ln^{(1-\alpha)}(T)}}{\ln^{(1-\alpha)}(T) + 1} e^{-0.5 \ln^{(1-\alpha)}(T)} \\
    &   &\\
      & \le^{(a)} & 1 - \frac{1}{2\sqrt{2\pi }} \cdot \frac{\sqrt{\ln^{(1-\alpha)}(T)}}{\ln^{(1-\alpha)}(T) } e^{-0.5 \left( (1-\alpha)\ln(T) +1 \right)} \\

      &&\\
         & = & 1 - \frac{1}{2\sqrt{2e\pi}} \cdot \frac{1}{\sqrt{\ln^{(1-\alpha)}(T) \cdot T^{ (1-\alpha)}  } } \\
       &  = & 1- c_0 \cdot  \frac{1}{\sqrt{\ln^{(1-\alpha)}(T) \cdot T^{ (1-\alpha)}  } }\quad.

    \end{array}
    \label{Apple 2}
\end{equation}
Inequality (a) in (\ref{Apple 2}) uses Fact~\ref{Fact: calculus} and $\ln^{(1-\alpha)}(T) \ge 1$.

By plugging the upper bounds of $\omega_1$ and $\omega_2$ into (\ref{Apple 1}), we have
\begin{equation}
    \begin{array}{lll}
        \mathbb{E} \left[n_i(T) \right] &  \le 
         & L_i + \omega_1 + \omega_2 \\
         & \le &  \frac{\left(\sqrt{0.5}+\sqrt{5-\alpha}\right)^2 \ln^{1+\alpha} (T)}{\Delta_i^2} +1 + \frac{1}{c_0} +2 + 3 \\
         & \le & \frac{\left(\sqrt{0.5}+\sqrt{5-\alpha}\right)^2 \ln^{1+\alpha} (T)}{\Delta_i^2} + O(1)\\
         & \le & \frac{\left(\sqrt{0.5}+\sqrt{5}\right)^2 \ln^{1+\alpha} (T)}{\Delta_i^2} + O(1)\quad, 
    \end{array}
\end{equation}
which concludes the proof. 
\end{proof}

\section{Proofs for Theorem~\ref{Theorem: Main 3}} \label{app: TS-TD}
For a fixed sub-optimal arm $i$, we upper bound the expected number of pulls $\mathbb{E} \left[ n_i(T)\right]$ by the of round $T$. Let $\bar{\mu}_i(t-1) := \hat{\mu}_{i,n_i(t-1)} + \sqrt{\frac{\left(5-\alpha\right) \ln^{1+\alpha}(T)}{n_i(t-1)}} $ be the upper confidence bound. Define $\mathcal{E}_i^{\mu}(t-1) := \left\{\left| \hat{\mu}_{i,n_i(t-1)} - \mu_i \right| \le \sqrt{\frac{2\ln(T)}{n_i(t-1)}} \right\}$  and $\mathcal{E}_i^{\theta}(t) := \left\{\theta_i(t) \le \bar{\mu}_i(t-1)  \right\}$. Note that  $\theta_i(t)$ can be either a fresh posterior sample (arm $i$ is in the first phase in round $t$) or the best posterior sample based on the previously drawn samples (arm $i$ is in the second phase in round $t$).

Define $\mathcal{E}_1^{\mu}(t-1) := \left\{\left| \hat{\mu}_{1,n_1(t-1)} - \mu_1 \right| \le \sqrt{\frac{2\ln(T)}{n_1(t-1)}} \right\}$.  

Let $L_i := \frac{4 \left(\sqrt{2} + \sqrt{5-\alpha} \right)^2 \ln^{1+\alpha} (T)}{\Delta_i^2}$. 
Now, we are ready to decompose the regret.  We have
\begin{equation}
    \begin{array}{ll}
         &\mathbb{E} \left[n_i(T) \right] \\
         
          \le & L_i + \sum\limits_{t=1}^{T} \mathbb{E} \left[\bm{1} \left\{i_t = i, n_i(t-1) \ge L_i \right\} \right]  \\
           \le & L_i + \underbrace{\sum\limits_{t=1}^{T} \mathbb{E} \left[\bm{1} \left\{i_t = i, \mathcal{E}_i^{\theta}(t), \mathcal{E}_i^{\mu}(t-1), \mathcal{E}_1^{\mu}(t-1),  n_i(t-1) \ge L_i \right\} \right] }_{\omega_1}  
         
           + \underbrace{ \sum\limits_{t=1}^{T} \mathbb{E} \left[\bm{1} \left\{i_t = i,  \overline{\mathcal{E}_i^{\theta}(t)}, n_i(t-1) \ge L_i \right\} \right]}_{\omega_2} \\
           +& \underbrace{  \sum\limits_{t=1}^{T} \mathbb{E} \left[\bm{1} \left\{i_t = i, \overline{\mathcal{E}_i^{\mu}(t-1)}, n_i(t-1) \ge L_i \right\} \right]}_{\omega_3}+ \underbrace{  \sum\limits_{t=1}^{T} \mathbb{E} \left[\bm{1} \left\{i_t = i, \overline{\mathcal{E}_1^{\mu}(t-1)}, n_i(t-1) \ge L_i \right\} \right]}_{\omega_4}.
           \end{array}
           \label{Friday 1}
           \end{equation}
Upper bounding terms $\omega_2, \omega_3$ and $\omega_4$ are not difficult as only concentration bounds are required. 

        The challenging part is to upper bound term $\omega_1$.  We decompose term $\omega_1$  based on whether the optimal arm $1$ is running Vanilla Thompson Sampling or not in round $t$. Define $\mathcal{T}_1(t)$ as the event that the optimal arm $1$ in round $t$ is using a fresh posterior sample in the learning. In other words, if event $\mathcal{T}_1(t)$ is true, we have $\theta_1(t) \sim \mathcal{N} \left(\hat{\mu}_{1, n_1(t-1)}, \ \frac{\ln^{\alpha}(T)}{n_1(t-1)} \right)$. So, $\overline{\mathcal{T}_1(t)}$ denotes the event that the optimal arm $1$ is using the best posterior sample among all the previously drawn samples in the learning in round $t$.  We have      
           \begin{equation}
    \begin{array}{lll}
         \omega_1 
      & = &
           \underbrace{\sum\limits_{t=1}^{T} \mathbb{E} \left[\bm{1} \left\{i_t = i, \mathcal{E}_i^{\theta}(t), \mathcal{T}_1(t), \mathcal{E}_i^{\mu}(t-1),\mathcal{E}_1^{\mu}(t-1),   n_i(t-1) \ge L_i \right\} \right] }_{I_1}  
        \\
        & +  &   \underbrace{\sum\limits_{t=1}^{T} \mathbb{E} \left[\bm{1} \left\{i_t = i,\mathcal{E}_i^{\theta}(t), \overline{\mathcal{T}_1(t)}, \mathcal{E}_i^{\mu}(t-1), \mathcal{E}_1^{\mu}(t-1), n_i(t-1) \ge L_i \right\} \right]  }_{I_2}.  
         \end{array}
         \label{Friday 3}
         \end{equation}

Let $\mathcal{F}_{t-1} = \left\{h_1(\tau), h_2(\tau), \dotsc, h_K(\tau), i_{\tau}, X_{i_{\tau}}(\tau), \forall \tau =1,2,\dotsc,t-1 \right\}$ collect all the history information by the end of round $t$. It collects the number of used posterior samples $h_i(\tau)$ for all $i \in \mathcal{A}$, the index of the pulled arm $i_{\tau}$, and the observed reward $X_{i_{\tau}}(\tau)$ for all $\tau = 1, 2,\dotsc,t-1$. Note that detailed information about posterior samples is not collected.

\paragraph{Upper bound $I_1$.}        Term $I_1$ in (\ref{Friday 3}) will use a similar analysis to  Vanilla Thompson Sampling, which links the probability of pulling a sub-optimal arm $i$ to the probability of pulling the optimal arm $1$. Note that if event $\mathcal{T}_1(t)$ is true, we know for sure the optimal arm $1$ is running  Vanilla Thompson Sampling in round $t$. This means the optimal arm $1$ will have some chance to be pulled. We formalize this claim in our technical Lemma~\ref{martingale lemma} below.
    
        \begin{lemma}
        For any instantiation $F_{t-1}$ of $\mathcal{F}_{t-1}$, we have 
           \begin{equation}
        \begin{array}{l}
             \mathbb{E} \left[   \bm{1} \left\{i_t = i,\mathcal{T}_1(t) , \mathcal{E}_i^{\theta}(t) \right\} \mid \mathcal{F}_{t-1} = F_{t-1}\right] 
             \le  \frac{ \mathbb{P} \left\{\theta_{1, n_1(t-1)} \le \bar{\mu}_i(t-1)\mid \mathcal{F}_{t-1} = F_{t-1} \right\} }{ \mathbb{P} \left\{\theta_{1,n_1(t-1)} > \bar{\mu}_i(t-1)\mid \mathcal{F}_{t-1} = F_{t-1} \right\} }   \mathbb{E} \left[   \bm{1} \left\{i_t = 1, \mathcal{T}_1(t), \mathcal{E}_i^{\theta}(t) \right\} \mid \mathcal{F}_{t-1} = F_{t-1}\right],
        \end{array}
       \end{equation} 
       where $\theta_{1,n_1(t-1)} \sim \mathcal{N} \left(\hat{\mu}_{1,n_1(t-1)}, \ \frac{\ln^{\alpha}(T)}{n_1(t-1)} \right)$.
       \label{martingale lemma}
    \end{lemma}
    With Lemma~\ref{martingale lemma} in hand, we now ready to upper bound term $I_1$.  We have
             \begin{equation}
              \begin{array}{lll}
               
          I_1 & = & \sum\limits_{t=1}^{T} \mathbb{E} \left[\bm{1} \left\{i_t = i, \mathcal{T}_1(t),\mathcal{E}_1^{\mu}(t-1), \mathcal{E}_i^{\mu}(t-1),\mathcal{E}_i^{\theta}(t), n_i(t-1) \ge L_i \right\} \right]\\

          & =  & \sum\limits_{t=1}^{T}  \mathbb{E} \left[\bm{1} \left\{ \mathcal{E}_i^{\mu}(t-1),\mathcal{E}_1^{\mu}(t-1), n_i(t-1) \ge L_i \right\} \cdot  \underbrace{\mathbb{E} \left[   \bm{1} \left\{i_t = i,  \mathcal{T}_1(t),\mathcal{E}_i^{\theta}(t) \right\} \mid \mathcal{F}_{t-1}\right] }_{\text{LHS of Lemma~\ref{martingale lemma}}} \right] \\

& \le   & \sum\limits_{t=1}^{T}  \mathbb{E} \left[\underbrace{\bm{1} \left\{ \mathcal{E}_i^{\mu}(t-1),\mathcal{E}_1^{\mu}(t-1),n_i(t-1) \ge L_i \right\}  \underbrace{  \frac{ \mathbb{P} \left\{\theta_{1,n_1(t-1)} \le \bar{\mu}_i(t-1) \mid \mathcal{F}_{t-1}  \right\} }{ \mathbb{P} \left\{\theta_{1,n_1(t-1)} > \bar{\mu}_i(t-1) \mid \mathcal{F}_{t-1}  \right\} }   \mathbb{E} \left[   \bm{1} \left\{i_t = 1, \mathcal{T}_1(t), \mathcal{E}_i^{\theta}(t) \right\} \mid \mathcal{F}_{t-1} \right] }_{\text{RHS of Lemma~\ref{martingale lemma}}} }_{\eta}\right] \\
& \le^{(a)}   & \sum\limits_{t=1}^{T}  \mathbb{E} \left[ \bm{1} \left\{\mathcal{E}_1^{\mu}(t-1) \right\} \cdot \frac{ \mathbb{P} \left\{\theta_{1,n_1(t-1)} \le \mu_1 - \frac{\Delta_i}{2} \mid \mathcal{F}_{t-1}  \right\} }{ \mathbb{P} \left\{\theta_{1,n_1(t-1)} > \mu_1 - \frac{\Delta_i}{2} \mid \mathcal{F}_{t-1}  \right\} } \cdot  \mathbb{E} \left[   \bm{1} \left\{i_t = 1\right\} \mid \mathcal{F}_{t-1} \right]  \right] \\
& =   & \underbrace{\sum\limits_{t=1}^{T}  \mathbb{E} \left[ \bm{1} \left\{\mathcal{E}_1^{\mu}(t-1) \right\} \cdot  \frac{ \mathbb{P} \left\{\theta_{1,n_1(t-1)} \le \mu_1 - \frac{\Delta_i}{2} \mid \mathcal{F}_{t-1}  \right\} }{ \mathbb{P} \left\{\theta_{1,n_1(t-1)} > \mu_1 - \frac{\Delta_i}{2} \mid \mathcal{F}_{t-1}  \right\} } \cdot     \bm{1} \left\{i_t = 1 \right\} \right] }_{\zeta}\quad,
\end{array}
\end{equation}
where inequality (a) uses the following arguments. For any specific $F_{t-1}$ such that either event $ \mathcal{E}_i^{\mu}(t-1)$ or $ n_i(t-1) \ge L_i$ is false, we have $\eta = 0$. Note that we have $  0 < \frac{ \mathbb{P} \left\{\theta_{1,n_1(t-1)} \le \bar{\mu}_i(t-1) \mid \mathcal{F}_{t-1} = F_{t-1} \right\} }{ \mathbb{P} \left\{\theta_{1,n_1(t-1)} > \bar{\mu}_i(t-1) \mid \mathcal{F}_{t-1} = F_{t-1} \right\} } < + \infty$. Recall $L_i = \frac{ 4\left(\sqrt{2} + \sqrt{5-\alpha} \right)^2 \ln^{1+\alpha} (T)}{\Delta_i^2}$. For any specific $F_{t-1}$ such that both events $\mathcal{E}_i^{\mu}(t-1)$ and $n_i(t-1) \ge L_i $ are true, we have $\bar{\mu}_i(t-1) = \hat{\mu}_{i,n_i(t-1)} + \sqrt{\frac{\left(5-\alpha \right)\ln^{1+\alpha}(T)}{n_i(t-1)}} \le \mu_i + \sqrt{\frac{2\ln(T)}{n_i(t-1)}} + \sqrt{\frac{\left(5-\alpha \right)\ln^{1+\alpha}(T)}{n_i(t-1)}} \le  \mu_i + \sqrt{\frac{ \left(\sqrt{2} + \sqrt{5-\alpha} \right)^2\ln^{1+\alpha}(T)}{n_i(t-1)}} \le \mu_i +  \sqrt{\frac{ \left(\sqrt{2} + \sqrt{5-\alpha} \right)^2\ln^{1+\alpha}(T)}{L_i}}  \le \mu_i + \frac{\Delta_i}{2} = \mu_1 - \frac{\Delta_i}{2}$, which implies 
$ \mathbb{P} \left\{\theta_{1,n_1(t-1)} > \bar{\mu}_i(t-1) \mid \mathcal{F}_{t-1}= F_{t-1}  \right\} \ge  \mathbb{P} \left\{\theta_{1,n_1(t-1)} > \mu_1 - \frac{\Delta_i}{2}\mid \mathcal{F}_{t-1}= F_{t-1}  \right\} $.

Let $L_{1,i} := \frac{32 \ln^{1+\alpha}(T)}{\Delta_i^2}$. Let $\tau_s^{(1)}$ be the round when the optimal arm $1$ is pulled for the $s$-th time. Now, we continue upper bounding $\zeta$ by partitioning all $T$ rounds into multiple intervals based on the pulls of the optimal arm $1$. We have
\begin{equation}
    \begin{array}{lll}
\zeta & \le   & \sum\limits_{s=1}^{T}  \mathbb{E} \left[ \sum\limits_{t= \tau^{(1)}_s +1}^{\tau^{(1)}_{s+1}} \bm{1} \left\{\mathcal{E}_1^{\mu}(t-1) \right\} \cdot   \frac{ \mathbb{P} \left\{\theta_{1,n_1(t-1)} \le \mu_1 - \frac{\Delta_i}{2} \mid \mathcal{F}_{t-1}  \right\} }{ \mathbb{P} \left\{\theta_{1,n_1(t-1)} > \mu_1 - \frac{\Delta_i}{2} \mid \mathcal{F}_{t-1}  \right\} } \cdot  \bm{1} \left\{i_t = 1 \right\}  \right] \\

& \le   & \sum\limits_{s=1}^{T}  \mathbb{E} \left[ \bm{1} \left\{\mathcal{E}_1^{\mu}\left(\tau_{s+1}^{(1)} -1\right) \right\} \cdot   \frac{ \mathbb{P} \left\{\theta_{1,s} \le \mu_1 - \frac{\Delta_i}{2} \mid \mathcal{F}_{\tau^{(1)}_{s+1}-1}  \right\} }{ \mathbb{P} \left\{\theta_{1,s} > \mu_1 - \frac{\Delta_i}{2} \mid \mathcal{F}_{\tau^{(1)}_{s+1}-1}  \right\} }   \right] \\
& \le & \sum\limits_{s=1}^{L_{1,i}}  \mathbb{E} \left[     \frac{ \mathbb{P} \left\{\theta_{1,s} \le \mu_1 - \frac{\Delta_i}{2} \mid \mathcal{F}_{\tau^{(1)}_{s}}  \right\} }{ \mathbb{P} \left\{\theta_{1,s} > \mu_1 - \frac{\Delta_i}{2} \mid \mathcal{F}_{\tau^{(1)}_{s}}  \right\} }   \right] + \sum\limits_{s=L_{1,i}+1}^{T}  \mathbb{E} \left[\underbrace{ \bm{1} \left\{\mathcal{E}_1^{\mu}\left(\tau_s^{(1)}\right) \right\} \cdot   \frac{ \mathbb{P} \left\{\theta_{1,s} \le \mu_1 - \frac{\Delta_i}{2} \mid \mathcal{F}_{\tau^{(1)}_{s}}  \right\} }{ \mathbb{P} \left\{\theta_{1,s} > \mu_1 - \frac{\Delta_i}{2} \mid \mathcal{F}_{\tau^{(1)}_{s}}  \right\} } }_{\lambda}  \right]\\
&& \\
& \le^{(b)} & \sum\limits_{s=1}^{L_{1,i}}  \mathbb{E} \left[     \frac{ \mathbb{P} \left\{\theta_{1,s} \le \mu_1 - \frac{\Delta_i}{2} \mid \mathcal{F}_{\tau^{(1)}_{s}}  \right\} }{ \mathbb{P} \left\{\theta_{1,s} > \mu_1 - \frac{\Delta_i}{2} \mid \mathcal{F}_{\tau^{(1)}_{s}}  \right\} }   \right] + \sum\limits_{s=L_{1,i}+1}^{T}  \mathbb{E} \left[    \frac{ \mathbb{P} \left\{\theta_{1,s} \le \hat{\mu}_{1,s} - \sqrt{\frac{2\ln^{1+\alpha}(T)}{s}} \mid \mathcal{F}_{\tau^{(1)}_{s}}  \right\} }{ \mathbb{P} \left\{\theta_{1,s} > \hat{\mu}_{1,s} - \sqrt{\frac{2\ln^{1+\alpha}(T)}{s}} \mid \mathcal{F}_{\tau^{(1)}_{s}}  \right\} }   \right]\\
&&\\
&= & \sum\limits_{s=1}^{L_{1,i}}  \mathbb{E} \left[     \frac{ \mathbb{P} \left\{\theta_{1,s} \le \mu_1 - \frac{\Delta_i}{2} \mid \mathcal{F}_{\tau^{(1)}_{s}}  \right\} }{ \mathbb{P} \left\{\theta_{1,s} > \mu_1 - \frac{\Delta_i}{2} \mid \mathcal{F}_{\tau^{(1)}_{s}}  \right\} }   \right] + \sum\limits_{s=L_{1,i}+1}^{T}  \mathbb{E} \left[    \frac{1 }{ \mathbb{P} \left\{\theta_{1,s} > \hat{\mu}_{1,s} - \sqrt{2\ln(T)}\sqrt{\frac{\ln^{\alpha}(T)}{s}} \mid \mathcal{F}_{\tau^{(1)}_{s}}  \right\} }  -1 \right]\\
&&\\
& \le^{(c)} & \sum\limits_{s=1}^{L_{1,i}}  \mathbb{E} \left[     \frac{ \mathbb{P} \left\{\theta_{1,s} \le \mu_1 - \frac{\Delta_i}{2} \mid \mathcal{F}_{\tau^{(1)}_{s}}  \right\} }{ \mathbb{P} \left\{\theta_{1,s} > \mu_1 - \frac{\Delta_i}{2} \mid \mathcal{F}_{\tau^{(1)}_{s}}  \right\} }   \right] + \sum\limits_{s=L_{1,i}+1}^{T}  \mathbb{E} \left[    \frac{1 }{1-\frac{0.5}{T} }  -1 \right]\\
&&\\
& \le& \sum\limits_{s=1}^{L_{1,i}}  \mathbb{E} \left[     \frac{ \mathbb{P} \left\{\theta_{1,s} \le \mu_1 - \frac{\Delta_i}{2} \mid \mathcal{F}_{\tau^{(1)}_{s}}  \right\} }{ \mathbb{P} \left\{\theta_{1,s} > \mu_1 - \frac{\Delta_i}{2} \mid \mathcal{F}_{\tau^{(1)}_{s}}  \right\} }   \right] + 1\quad.

    \end{array}
    \end{equation}

Inequality (b) uses the following arguments. For any specific $F_{\tau_s^{(1)}}$ such that event $\mathcal{E}_1^{\mu}\left(\tau_s^{(1)}\right)$ is false, we have $\lambda = 0$. Note that we have $0 < \frac{ \mathbb{P} \left\{\theta_{1,s} \le \mu_1 - \frac{\Delta_i}{2} \mid \mathcal{F}_{\tau^{(1)}_{s}}=F_{\tau_s^{(1)}}  \right\} }{ \mathbb{P} \left\{\theta_{1,s} > \mu_1 - \frac{\Delta_i}{2} \mid \mathcal{F}_{\tau^{(1)}_{s}} =F_{\tau_s^{(1)}} \right\} }  < +\infty$. For any specific $F_{\tau_s^{(1)}}$ such that event $\mathcal{E}_1^{\mu}\left(\tau_s^{(1)}\right)$ is true, we have $\mu_1 - \frac{\Delta_i}{2} \le \hat{\mu}_{1,s} + \sqrt{\frac{2\ln(T)}{s}} - \frac{\Delta_i}{2} \le  \hat{\mu}_{1,s} + \sqrt{\frac{2\ln^{1+\alpha}(T)}{s}} -  \sqrt{\frac{8 \ln^{1+\alpha}(T)}{ s}} =  \hat{\mu}_{1,s} - \sqrt{\frac{2\ln^{1+\alpha}(T)}{s}}$. Note that
from $L_{1,i} = \frac{32 \ln^{1+\alpha}(T)}{\Delta_i^2}$, we have $\frac{\Delta_i}{2} = \sqrt{\frac{32 \ln^{1+\alpha}(T)}{4\cdot L_{1,i}}} \ge \sqrt{\frac{32 \ln^{1+\alpha}(T)}{4\cdot s}} = \sqrt{\frac{8\ln^{1+\alpha}(T)}{ s}} $.
Inequality (c) uses Gaussian concentration bound, which gives $ \mathbb{P} \left\{\theta_{1,s} > \hat{\mu}_{1,s} - \sqrt{2\ln(T)}\sqrt{\frac{\ln^{\alpha}(T)}{s}} \mid \mathcal{F}_{\tau^{(1)}_{s}}  \right\} \ge 1-\frac{0.5}{T}$.

To complete the proof for $\zeta$, we claim the following result stated in Lemma~\ref{Banana 2}. 

\begin{lemma}
\label{Banana 2}
    Let $\tau_s^{(1)}$ be the round when the $s$-th pull of the optimal arm $1$ occurs and $\theta_{1,s} \sim \mathcal{N}\left(\hat{\mu}_{1,s}, \frac{\ln^{\alpha}(T)}{s}\right)$. Then, for any integer $s \ge 1$, we have
    \begin{equation}
    \begin{array}{lll}
          \mathbb{E} \left[     \frac{ \mathbb{P} \left\{\theta_{1,s} \le \mu_1 - \frac{\Delta_i}{2} \mid \mathcal{F}_{\tau^{(1)}_{s}}  \right\} }{ \mathbb{P} \left\{\theta_{1,s} > \mu_1 - \frac{\Delta_i}{2} \mid \mathcal{F}_{\tau^{(1)}_{s}}  \right\} }   \right] & \le & 29\quad. 
    \end{array}
\end{equation}
\end{lemma}

Then, we have
\begin{equation}
    \begin{array}{l}
      I_1  \le \zeta \le  \sum\limits_{s=1}^{L_{1,i}}  \mathbb{E} \left[     \frac{ \mathbb{P} \left\{\theta_{1,s} \le \mu_1 - \frac{\Delta_i}{2} \mid \mathcal{F}_{\tau^{(1)}_{s}}  \right\} }{ \mathbb{P} \left\{\theta_{1,s} > \mu_1 - \frac{\Delta_i}{2} \mid \mathcal{F}_{\tau^{(1)}_{s}}  \right\} }   \right] + 1 
       \le  L_{1,i} \cdot 29 + 1 
      
       \le O \left( \frac{ \ln^{1+\alpha}(T)}{\Delta_i^2}\right) \quad.
    \end{array}
\end{equation}

    \textbf{Upper bound $I_2$.} The regret decomposition for term $I_2$ in (\ref{Friday 3}) is very similar to the one for TS-MA-$\alpha$, as the optimal arm $1$ is using the best posterior sample in the learning in round $t$. Here, for the sub-optimal arm $i$,  we do not define an event to indicate whether the sub-optimal arm $i$ is running Vanilla Thompson Sampling in round $t$ or is using the best posterior sample in the learning in round $t$. This is because once the sub-optimal arm $i$ has been observed enough, it is unlikely that $\theta_i(t)$ will beat $\mu_i + \frac{\Delta_i}{2}$ regardless of whether $\theta_i(t)$ is a fresh posterior sample or the best posterior sample based on the previously drawn samples.
          We have
\begin{equation}
    \begin{array}{lll}
      I_2   & = & \sum\limits_{t=1}^{T} \mathbb{E} \left[\bm{1} \left\{i_t = i,\mathcal{E}_i^{\theta}(t), \overline{\mathcal{T}_1(t)}, \mathcal{E}_i^{\mu}(t-1), \mathcal{E}_1^{\mu}(t-1), n_i(t-1) \ge L_i \right\} \right]    \\
      & = & \sum\limits_{t=1}^{T} \mathbb{E} \left[\bm{1} \left\{i_t = i,\mathcal{E}_i^{\theta}(t), \theta_i(t) \ge \theta_1(t), \overline{\mathcal{T}_1(t)}, \mathcal{E}_i^{\mu}(t-1), \mathcal{E}_1^{\mu}(t-1), n_i(t-1) \ge L_i \right\} \right]  \\
       & \le^{(a)}  & \underbrace{ \sum\limits_{t=1}^{T} \mathbb{E} \left[\bm{1} \left\{i_t = i,\mathcal{E}_i^{\theta}(t), \theta_i(t) \ge \mu_i + \Delta_i,  \mathcal{E}_i^{\mu}(t-1) ,  n_i(t-1) \ge L_i \right\} \right]  }_{=0}
       + \underbrace{\sum\limits_{t=1}^{T} \mathbb{E} \left[\bm{1} \left\{\theta_1(t) < \mu_1, \overline{\mathcal{T}_1(t)}, \mathcal{E}_1^{\mu}(t-1) \right\} \right]}_{\le 1} \\
       & \le & 1\quad.
  \end{array}
\end{equation}
We use contradiction to prove that the first term in inequality (a) is $0$. If both events $\mathcal{E}_i^{\theta}(t)$, $\mathcal{E}_i^{\mu}(t-1)$ are true and $n_i(t-1) \ge L_i$, we have
$\theta_i(t) \le  \hat{\mu}_{i,n_i(t-1)} + \sqrt{\frac{\left(5-\alpha\right) \ln^{1+\alpha}(T)}{n_i(t-1)}} \le \mu_i + \sqrt{\frac{2\ln(T)}{n_i(t-1)}} + \sqrt{\frac{\left(5-\alpha\right) \ln^{1+\alpha}(T)}{n_i(t-1)}} \le \mu_i + \sqrt{\frac{2\ln(T)}{L_i}} + \sqrt{\frac{\left(5-\alpha\right) \ln^{1+\alpha}(T)}{L_i}} \le \mu_i + \frac{\Delta_i}{2}$,
which yields a contradiction with $\theta_i(t) \ge \mu_i+ \Delta_i$.

The second term in (a) reuses a similar analysis as (\ref{Apple 1}). We have
\begin{equation}
    \begin{array}{ll}
        & \sum\limits_{t=1}^{T} \mathbb{E} \left[\bm{1} \left\{\theta_1(t) < \mu_1, \overline{\mathcal{T}_1(t)}, \mathcal{E}_1^{\mu}(t-1) \right\} \right]  \\
       \le   & \sum\limits_{s=1}^{T} \mathbb{E} \left[\sum\limits_{t =\tau_s^{(1)}+1}^{\tau_{s+1}^{(1)}}\bm{1} \left\{\theta_1(t) < \mu_1, \overline{\mathcal{T}_1(t)}, \mathcal{E}_1^{\mu}(t-1) \right\} \right]  \\
   \le^{(a)}   & \sum\limits_{s=1}^{T} \mathbb{E} \left[\underbrace{\sum\limits_{t =\tau_s^{(1)}+1}^{\tau_{s}^{(1)}+\phi}\bm{1} \left\{\theta_1(t) < \mu_1, \overline{\mathcal{T}_1(t)} , \mathcal{E}_1^{\mu}(t-1)\right\}}_{=0}+\sum\limits_{t =\tau_s^{(1)}+\phi+1}^{\tau_{s+1}^{(1)}}\bm{1} \left\{\theta_1(t) < \mu_1, \overline{\mathcal{T}_1(t)},\mathcal{E}_1^{\mu} (t-1)\right\} \right]  \\
      \le   & \sum\limits_{s=1}^{T} \mathbb{E} \left[\sum\limits_{t =\tau_s^{(1)}+\phi+1}^{\tau_{s+1}^{(1)}}\bm{1} \left\{\mathop{\max}\limits_{h \in [\phi]} \theta_{1,s}^{h} < \hat{\mu}_{1, s} + \sqrt{\frac{\ln(T)}{s}}\right\} \right]  \\
       \le   & \sum\limits_{s=1}^{T} \mathbb{E} \left[\sum\limits_{t =1}^{T}  \mathbb{P} \left\{\mathop{\max}\limits_{h \in [\phi]} \theta_{1,s}^{h} < \hat{\mu}_{1, s} + \sqrt{\frac{\ln(T)}{s}} \mid \hat{\mu}_{1,s} \right\}   \right]   \\

               \le   & \sum\limits_{s=1}^{T} \mathbb{E} \left[\sum\limits_{t =1}^{T} \prod\limits_{h \in [\phi]} \underbrace{ \mathbb{P} \left\{\theta_{1,s}^{h} < \hat{\mu}_{1, s} + \sqrt{\frac{\ln(T)}{s}} \mid \hat{\mu}_{1, s} \right\}  }_{\eta} \right]  \\
&\\
        \le & \sum\limits_{s=1}^{T}\sum\limits_{t=1}^{T} \left(1-c_0 \cdot\frac{1}{\sqrt{\ln^{(1-\alpha)}(T) \cdot T^{ (1-\alpha)}  } }\right)^{\phi}  \\

        &\\
         \le &T^2 \cdot  e^{-c_0 \cdot \frac{1}{\sqrt{\ln^{(1-\alpha)}(T) \cdot T^{ (1-\alpha)}  } } \cdot \phi} + 2\\
         \le & 1\quad,
    \end{array}
\end{equation}
where inequality (a) uses the fact for each round $t = \tau_s^{(1)}+1, \dotsc, \tau_s^{(1)}+\phi$, the optimal arm $1$ is using a  fresh posterior sample. So, event $\overline{\mathcal{T}_1(t)}$ cannot occur. The $\eta$ term in (b) reuses the analysis of 
(\ref{Apple 2}).

Putting together the upper bounds for $I_1$ and $I_2$ into $\omega_1$, we have $\omega_1 \le O \left(\frac{\ln^{1+\alpha}\ln(T)}{\Delta_i^2} \right)$. Then, we have
\begin{equation}
    \begin{array}{lllll}
         \mathbb{E}\left[n_i(T) \right] & \le & L_i + \omega_1+\omega_2+\omega_3 &\le& O \left(\frac{\ln^{1+\alpha}\ln(T)}{\Delta_i^2} \right) \quad,  
    \end{array}
\end{equation}
which concludes the proof.


               \begin{proof}[Proof of Lemma~\ref{martingale lemma}]
        
   
        For any $F_{t-1}$, we have
    \begin{equation}
        \begin{array}{ll}
             & \mathbb{E} \left[   \bm{1} \left\{i_t = i,\mathcal{T}_1(t),  \mathcal{E}_i^{\theta}(t) \right\} \mid \mathcal{F}_{t-1} = F_{t-1}\right] \\
             \le & \mathbb{E} \left[   \bm{1} \left\{\theta_i(t) \le \bar{\mu}_i(t-1), \theta_1(t) \le \bar{\mu}_i(t-1), \theta_j(t) \le \bar{\mu}_i(t-1), \forall j \in \mathcal{A} \setminus \{i,1\} ,\mathcal{T}_1(t) \right\}\mid \mathcal{F}_{t-1} = F_{t-1}\right]\\

  = &  \bm{1} \left\{\mathcal{T}_1(t)  \right\} \cdot \mathbb{E}\left[   \bm{1} \left\{\theta_1(t) \le \bar{\mu}_i(t-1), \theta_j(t) \le \bar{\mu}_i(t-1), \forall j \in \mathcal{A} \setminus \{1\} \right\} \mid \mathcal{F}_{t-1} = F_{t-1} \right] \\
    = &  \bm{1} \left\{\mathcal{T}_1(t)   \right\} \cdot \mathbb{E} \left[   \bm{1} \left\{\theta_1(t) \le \bar{\mu}_i(t-1)\right\}\mid \mathcal{F}_{t-1} = F_{t-1} \right] \cdot \mathbb{E} \left[\bm{1}  \left\{\theta_j(t) \le \bar{\mu}_i(t-1), \forall j \in \mathcal{A} \setminus \{1\} \right\} \mid \mathcal{F}_{t-1} = F_{t-1} \right] .  
        \end{array}
        \label{Sun 1}
    \end{equation}
    The first inequality uses the fact that whether event $\mathcal{T}_1(t)$ is true is determined by the history information. Note that if $h_1(t-1) = \phi-1$, then $\bm{1} \left\{\mathcal{T}_1(t)\right\}$ = 1. If $h_1(t-1) \in \left\{0,1,\dotsc,\phi-1\right\}$, then $\bm{1} \left\{\mathcal{T}_1(t)\right\}$ = 0.

 For any $F_{t-1}$, we also have
    \begin{equation}
        \begin{array}{ll}
             &  \mathbb{E} \left[   \bm{1} \left\{i_t = 1,\mathcal{T}_1(t),  \mathcal{E}_i^{\theta}(t) \right\} \mid \mathcal{F}_{t-1} = F_{t-1}\right] \\
           \ge  & \mathbb{E} \left[   \bm{1} \left\{\theta_1(t) > \bar{\mu}_i(t-1) \ge \theta_j(t), \forall j \in \mathcal{A} \setminus \{1\},  \mathcal{T}_1(t) \right\} \mid \mathcal{F}_{t-1} = F_{t-1} \right] \\

           = &  \bm{1} \left\{\mathcal{T}_1(t)  \right\} \cdot \mathbb{E} \left[   \bm{1} \left\{\theta_1(t) > \bar{\mu}_i(t-1),   \theta_j(t) \le \bar{\mu}_i(t-1), \forall j \in \mathcal{A}\setminus \{1\}\right\} \mid \mathcal{F}_{t-1} = F_{t-1} \right]\\
       = &  \bm{1} \left\{\mathcal{T}_1(t)\right\} \cdot \mathbb{E} \left[   \bm{1} \left\{\theta_1(t) > \bar{\mu}_i(t-1)\right\} \mid \mathcal{F}_{t-1} = F_{t-1} \right] \cdot \mathbb{E} \left[\bm{1} \left\{\theta_j(t) \le \bar{\mu}_i(t-1), \forall j \in \mathcal{A} \setminus \{1\}\right\} \mid \mathcal{F}_{t-1} = F_{t-1} \right].
        \end{array}
        \label{Sun 2}
    \end{equation}

    We categorize all the possible $F_{t-1}$ into two groups based on whether 
        $\bm{1} \left\{\mathcal{T}_1(t) \right\} = 0$ or 
    $\bm{1} \left\{\mathcal{T}_1(t)  \right\} = 1$. 
    
    For any $F_{t-1}$ such that $\bm{1} \left\{\mathcal{T}_1(t)\right\} = 0$, 
    combining (\ref{Sun 1}) and (\ref{Sun 2}) yielding
    \begin{equation}
        \begin{array}{ll}
             &\underbrace{\mathbb{E} \left[   \bm{1} \left\{i_t = i,\mathcal{T}_1(t),  \mathcal{E}_i^{\theta}(t) \right\} \mid \mathcal{F}_{t-1} = F_{t-1}\right] }_{=0}
            \\
            & \\
            \le & 
            \underbrace{\frac{ \mathbb{P} \left\{\theta_{1, n_1(t-1)} \le \bar{\mu}_i(t-1)\mid \mathcal{F}_{t-1} = F_{t-1} \right\} }{ \mathbb{P} \left\{\theta_{1,n_1(t-1)} > \bar{\mu}_i(t-1)\mid \mathcal{F}_{t-1} = F_{t-1} \right\} }}_{>0} \cdot \underbrace{ \mathbb{E} \left[   \bm{1} \left\{i_t = 1,\mathcal{T}_1(t),  \mathcal{E}_i^{\theta}(t) \right\} \mid \mathcal{F}_{t-1} = F_{t-1}\right]}_{=0}.
        \end{array}
    \end{equation}
 Note that the coefficient $\frac{ \mathbb{P} \left\{\theta_{1, n_1(t-1)} \le \bar{\mu}_i(t-1)\mid \mathcal{F}_{t-1} = F_{t-1} \right\} }{ \mathbb{P} \left\{\theta_{1,n_1(t-1)} > \bar{\mu}_i(t-1)\mid \mathcal{F}_{t-1} = F_{t-1} \right\} } $ is always positive. 

   For any $F_{t-1}$ such that $\bm{1} \left\{\mathcal{T}_1(t) \right\} = 1$, from (\ref{Sun 1}) and (\ref{Sun 2}), we have
   \begin{equation}
       \begin{array}{ll}
      &   \mathbb{E} \left[   \bm{1} \left\{i_t = i,\mathcal{T}_1(t),  \mathcal{E}_i^{\theta}(t) \right\} \mid \mathcal{F}_{t-1} = F_{t-1}\right]  \\
      \le &   \mathbb{P}     \left\{\theta_{1,n_1(t-1)} \le \bar{\mu}_i(t-1) \mid \mathcal{F}_{t-1} = F_{t-1} \right\} \cdot \mathbb{P}  \left\{\theta_j(t) \le \bar{\mu}_i(t-1), \forall j \in \mathcal{A} \setminus \{1\} \mid \mathcal{F}_{t-1} = F_{t-1} \right\}    
       \end{array}
       \label{Sun 3}
   \end{equation}
   and
      \begin{equation}
       \begin{array}{ll}
         & \mathbb{E} \left[   \bm{1} \left\{i_t = 1,\mathcal{T}_1(t),  \mathcal{E}_i^{\theta}(t) \right\} \mid \mathcal{F}_{t-1} = F_{t-1}\right] \\
          \ge &  \underbrace{ \mathbb{P}     \left\{\theta_{1,n_1(t-1)} > \bar{\mu}_i(t-1) \mid \mathcal{F}_{t-1} = F_{t-1}\right\}}_{>0} \cdot \mathbb{P}  \left\{\theta_j(t) \le \bar{\mu}_i(t-1), \forall j \in \mathcal{A} \setminus \{1\} \mid \mathcal{F}_{t-1} = F_{t-1} \right\}.
            
       \end{array}
       \label{Sun 4}
   \end{equation}
   Combining (\ref{Sun 3}) and (\ref{Sun 4}) concludes the proof. Note that $ \mathbb{P}     \left\{\theta_{1,n_1(t-1)} \le \bar{\mu}_i(t-1) 
 \mid \mathcal{F}_{t-1} = F_{t-1}\right\}$ and $ \mathbb{P}_{t-1}     \left\{\theta_{1,n_1(t-1)} > \bar{\mu}_i(t-1) \mid \mathcal{F}_{t-1} = F_{t-1}\right\}$ are both positive. 
\end{proof}

\begin{proof}[Proof of Lemma~\ref{Banana 2}]
    The proof will be very similar to the proofs for Lemma~\ref{UBC 22}. Still, we let $\mathcal{G}_{1,s}$ be a geometric random variable denoting the number of consecutive independent trials until event  $\theta_{1,s} > \mu_1 - \frac{\Delta_i}{2}$  occurs.
Then, we have  
$\mathbb{E}\left[\frac{1}{\mathbb{P} \left\{ \theta_{1,s} > \mu_1 - \frac{\Delta_i}{2}\mid \mathcal{F}_{\tau_{s}^{(1)}}  \right\} }-1   \right] 
             =    \mathbb{E} \left[\mathcal{G}_{1,s} \right]-1$. 
  For a fixed $s \ge 1$, we use the definition of expectation and have
  \begin{equation}
    \begin{array}{lllll}
            \mathbb{E} \left[\mathcal{G}_{1,s} \right] & = & 
        \sum\limits_{r = 0}^{\infty} \mathbb{P} \left\{\mathcal{G}_{1,s} > r  \right\} &= & \sum\limits_{r = 0}^{\infty} \mathbb{E} \left[\mathbb{P} \left\{\mathcal{G}_{1,s} > r \mid \hat{\mu}_{1,s} \right\}  \right] \quad. 
         \end{array}
\end{equation}
Now, we claim the following results.
For any $s \ge 1$, we have
\begin{equation}
\begin{array}{lll}   
   \mathbb{E} \left[\mathbb{P} \left\{\mathcal{G}_{1,s} > r \mid \hat{\mu}_{1,s} \right\}  \right] & \le
&
\begin{cases}
   1,& r \in [0,12] \quad,\\
   e^{- \frac{\ln(13)}{\ln(13)+2} \sqrt{\frac{r}{\pi}} }  + r^{-1}, &\text{$r \in [13, 100]$}  \quad , \\               
e^{- \sqrt{\frac{r}{3\pi}} }  + r^{-\frac{4}{3}}   ,              & \text{$r \ge 101$}\quad.
\end{cases}
\end{array}
\label{UBC 33}
\end{equation}
 To prove the results shown in (\ref{UBC 33}), we express the LHS in (\ref{UBC 33}) as 
\begin{equation}
            \begin{array}{lllll}
         \mathbb{E} \left[ \mathbb{P} \left\{\mathcal{G}_{1,s} > r \mid \hat{\mu}_{1,s}\right\} \right] &
          =& 1- \mathbb{E} \left[\mathbb{P} \left\{\mathcal{G}_{1,s} \le r \mid \hat{\mu}_{1,s} \right\}\right]    

          &=&1- \underbrace{\mathbb{E} \left[\mathbb{E} \left[ \bm{1} \left\{ \mathcal{G}_{1,s} \le r  \right\} \mid \hat{\mu}_{1,s} \right]\right] }_{=:\gamma} \quad.
          \end{array}
          \end{equation}
Now, we construct lower bounds for $\gamma$ for three cases. 
          
\textbf{When integer $r \in [0,12]$}, the proof is trivial as $\gamma \ge 0 $, which implies $\mathbb{E} \left[ \mathbb{P} \left\{\mathcal{G}_{1,s} > r \mid \hat{\mu}_{1,s}\right\} \right] \le 1$.

\textbf{When integer $r \in [13,100]$}, we introduce $z = \sqrt{0.5 \ln(r)}$ and have
\begin{equation}
        \begin{array}{lll}
 \gamma &
  = & \mathbb{E} \left[\mathbb{E} \left[ \bm{1} \left\{ \mathcal{G}_{1,s} \le r  \right\} \mid \hat{\mu}_{1,s} \right]\right]  \\
& \ge & 
\mathbb{E} \left[\mathbb{E} \left[ \bm{1}\left\{ \mathop{\max}\limits_{h \in [r]} \theta_{1,s}^{h} > \mu_1 - \frac{\Delta_i}{2}  \right\} \mid \hat{\mu}_{1,s}  \right]\right] \\
& \ge &
 \mathbb{E} \left[
\mathbb{E} \left[\bm{1} \left\{\hat{\mu}_{1,s} + z\sqrt{\frac{\ln^{\alpha}(T)}{s}} \ge \mu_1 - \frac{\Delta_i}{2} \right\} \bm{1}\left\{ \mathop{\max}\limits_{h \in [r]} \theta_{1,s}^{h} >\hat{\mu}_{1,s} +z \sqrt{\frac{\ln^{\alpha}(T)}{s}} \right\} \mid \hat{\mu}_{1,s}  \right]\right] \\
& = &
 \mathbb{E} \left[\bm{1} \left\{\hat{\mu}_{1,s} +z \sqrt{\frac{\ln^{\alpha}(T)}{s}} \ge \mu_1 - \frac{\Delta_i}{2} \right\} \cdot
\mathbb{P} \left\{ \mathop{\max}\limits_{h \in [r]} \theta_{1,s}^{h} >\hat{\mu}_{1,s} +z\sqrt{\frac{\ln^{\alpha}(T)}{s}}  \mid \hat{\mu}_{1,s} \right\} \right] \\
   &          \ge^{(a)} & \mathbb{E} \left[\bm{1} \left\{\hat{\mu}_{1,s} +\sqrt{0.5\ln(r)} \sqrt{\frac{\ln^{\alpha}(T)}{s}} \ge \mu_1 \right\} \right] \cdot   \left(1- \left( 1- \frac{1}{\sqrt{2\pi}} \frac{\sqrt{\frac{1}{2}\ln(r)}}{\frac{1}{2}\ln(r)+1} e^{- \frac{1}{4}\ln(r)} \right)^r \right) \\
                  & \ge^{(b)} & \left(1-e^{-\ln(r)} \right) \cdot  \left(1- e^{-r \cdot \frac{1}{\sqrt{2\pi}} \frac{\sqrt{\frac{1}{2}\ln(r)}}{\frac{1}{2}\ln(r)+1} \cdot r^{-\frac{1}{4}}}\right) \\
                         & \ge^{(c)} & \left(1-r^{-1} \right)   \cdot \left(1- e^{-r^{\frac{1}{2}} \cdot  \frac{1}{\sqrt{2\pi}} \frac{\sqrt{\frac{1}{2}\ln(r)}}{\frac{1}{2}\ln(r)+\frac{1}{\ln(13)}\ln(r)} \cdot r^{\frac{1}{4}} }\right) \\
                            & = & \left(1-r^{-1} \right)   \cdot \left(1- e^{-r^{\frac{1}{2}} \cdot \frac{1}{\sqrt{\pi}} \cdot \frac{\ln(13)}{\ln(13)+2}\frac{\sqrt{r^{\frac{1}{2}}}}{\sqrt{\ln(r)}} } \right)\\
                            & \ge^{(d)} & \left(1-r^{-1} \right)   \cdot \left(1-e^{-r^{\frac{1}{2}} \cdot \frac{1}{\sqrt{\pi}} \cdot \frac{\ln(13)}{\ln(13)+2} }\right)\\
                            & \ge & 1- r^{-1} - e^{-r^{\frac{1}{2}} \cdot \frac{1}{\sqrt{\pi}} \cdot \frac{\ln(13)}{\ln(13)+2} }\quad,
\end{array}
\end{equation}
which implies $\mathbb{E} \left[ \mathbb{P} \left\{\mathcal{G}_{1,s} > r \mid \hat{\mu}_{1,s}\right\} \right] =1-\gamma \le r^{-1} + e^{-r^{\frac{1}{2}} \cdot \frac{1}{\sqrt{\pi}} \cdot \frac{\ln(13)}{\ln(13)+2} }$.

Inequality (a) uses anti-concentration bounds of Gaussian distributions and inequality (b) uses Hoeffding's inequality and $e^{-x}\ge 1-x$. Inequalities (c) and (d) use $1 \le \ln(r)/\ln(13)$ and $\sqrt{r^{0.5}/\ln(r)}\ge 1$, when $r\ge 13$.

\textbf{When  integer $r \ge 101$},  we introduce  $ z = \sqrt{\frac{2}{3} \ln r}$ and  have
 \begin{equation}
        \begin{array}{lll}
 \gamma &
  = & \mathbb{E} \left[\mathbb{E} \left[ \bm{1} \left\{ \mathcal{G}_{1,s} \le r  \right\} \mid \hat{\mu}_{1,s} \right]\right]  \\
& \ge & 
\mathbb{E} \left[\mathbb{E} \left[ \bm{1}\left\{ \mathop{\max}\limits_{h \in [r]} \theta_{1,s}^{h} > \mu_1 - \frac{\Delta_i}{2}  \right\} \mid \hat{\mu}_{1,s}  \right]\right] \\
& \ge &
 \mathbb{E} \left[\bm{1} \left\{\hat{\mu}_{1,s} +z \sqrt{\frac{\ln^{\alpha}(T)}{s}} \ge \mu_1 - \frac{\Delta_i}{2} \right\} \cdot
\underbrace{\mathbb{P} \left\{ \mathop{\max}\limits_{h \in [r]} \theta_{1,s}^{h} >\hat{\mu}_{1,s} +z\sqrt{\frac{\ln^{\alpha}(T)}{s}}  \mid \hat{\mu}_{1,s} \right\} }_{=:\beta}\right] \quad.
\end{array}
\end{equation}

 We construct a lower bound for $\beta$. We have     \begin{equation}
            \begin{array}{lll}
            \beta & = &  
               1- \prod\limits_{h \in [r]} \left( 1- \mathbb{P} \left\{ \theta_{1,s}^{h} > \hat{\mu}_{1,s} +z \sqrt{\frac{ \ln^{\alpha}(T)}{s}}  \mid \hat{\mu}_{1,s} \right\}   \right)  \\
                &           = & 1- \left( 1- \mathbb{P} \left\{ \theta_{1,s} > \hat{\mu}_{1,s} +z\sqrt{\frac{ \ln^{\alpha}(T)}{s}}  \mid \hat{\mu}_{1,s} \right\}\right)^r  \\
                 &          \ge &    1- \left( 1- \frac{1}{\sqrt{2\pi}} \frac{z}{z^2+1} e^{-0.5  z^2} \right)^r  \\
                  &         \ge & 1- e^{-r \cdot \frac{1}{\sqrt{2\pi}} \frac{z}{z^2+1} e^{-0.5  z^2}}\\
                     &         = & 1- e^{-r \cdot \frac{1}{\sqrt{2\pi}} \frac{\sqrt{\frac{2}{3}\ln(r)}}{\frac{2}{3}\ln(r)+1} e^{-\frac{1}{2} \cdot \frac{2}{3}\ln(r)}}\\
                     & \ge & 1-e^{-r^{\frac{1}{2}} \cdot \frac{1}{\sqrt{3\pi}} \sqrt{\frac{r^{\frac{1}{3}} }{\ln r}} } \\
                     & \ge & 1- e^{-\sqrt{\frac{r}{3\pi} }} \quad.
            \end{array}
        \end{equation}

Then, we have
\begin{equation}
\begin{array}{lll}
\gamma & \ge &
 \left(1-e^{-\sqrt{\frac{r}{3\pi} }}\right) \cdot \mathbb{P}  \left\{\hat{\mu}_{1,s} +\sqrt{\frac{\ln r}{1.5}} \sqrt{\frac{  \ln^{\alpha}(T)}{s}} \ge \mu_1 \right\}
  \\
   &\ge &  \left(1-e^{-\sqrt{\frac{r}{3\pi} }} \right) \cdot \left(1- r^{-\frac{4}{3}}\right)\\
 &  \ge & 1- e^{-\sqrt{\frac{r}{3\pi} }}- r^{-\frac{4}{3}}\quad,
\end{array}
\end{equation}
which implies
$\mathbb{E} \left[ \mathbb{P} \left\{\mathcal{G}_{1,s} > r \mid \hat{\mu}_{1,s}\right\} \right] =1-\gamma \le   e^{-\sqrt{\frac{r}{3\pi} }} + r^{-\frac{4}{3}}$.

Finally, we have
 
    \begin{equation}
         \begin{array}{ll}
&\mathbb{E}\left[\frac{1}{\mathbb{P} \left\{ \theta_{1,s} > \mu_1 - \frac{\Delta_i}{2}\mid \mathcal{F}_{\tau_{s}^{(1)}}  \right\} }-1   \right] \\
=& 
\mathbb{E} \left[\mathcal{G}_{1,s} \right]-1 \\
         = & \sum\limits_{r = 0}^{\infty} \mathbb{E} \left[\mathbb{P} \left\{\mathcal{G}_{1,s} > r \mid \hat{\mu}_{1,s} \right\}  \right] -1\\
               = & \sum\limits_{r = 0}^{12} \mathbb{E} \left[\mathbb{P} \left\{\mathcal{G}_{1,s} > r \mid \hat{\mu}_{1,s} \right\}  \right] + \sum\limits_{r = 13}^{100} \mathbb{E} \left[\mathbb{P} \left\{\mathcal{G}_{1,s} > r \mid \hat{\mu}_{1,s} \right\}  \right] + \sum\limits_{r = 101}^{\infty} \mathbb{E} \left[\mathbb{P} \left\{\mathcal{G}_{1,s} > r \mid \hat{\mu}_{1,s} \right\}  \right]-1 \\
            \le & 13 + \sum\limits_{r=13}^{100} \left(e^{-\sqrt{\frac{r  }{\pi}} \cdot \frac{\ln(13)}{\ln(13)+2} } + \frac{1}{r} \right) + \sum\limits_{r=101}^{\infty} \left(e^{- \sqrt{\frac{r}{3\pi}} } + \frac{1}{r^{\frac{4}{3}}} \right) -1\\
             \le & 12 + \int_{12}^{100} \left(e^{-\sqrt{\frac{r  }{\pi}} \cdot \frac{\ln(13)}{\ln(13)+2} } + \frac{1}{r}\right) dr+ \int_{100}^{\infty} \left(e^{- \sqrt{\frac{r}{3\pi}} }  + \frac{1}{r^{\frac{4}{3}}} \right)dr \\

\le & 12 + 10.44 +2.13 + 3.1 + 0.65 \\
\le & 29 \quad,    
        \end{array} 
        \end{equation}
which concludes the proof.\end{proof}

\begin{lemma}
    We have 
    \begin{equation}
        \begin{array}{lll}
              \sum\limits_{t=1}^{T} \mathbb{E} \left[\bm{1} \left\{i_t = i,  \overline{\mathcal{E}_i^{\theta}(t)}, n_i(t-1) \ge L_i \right\} \right] & \le & \frac{1}{c_0}\quad.
        \end{array}
    \end{equation}
    \label{Friday 2}
\end{lemma}

\begin{lemma}
    We have 
    \begin{equation}
        \begin{array}{lll}
    \sum\limits_{t=1}^{T} \mathbb{E} \left[\bm{1} \left\{i_t = i,  \overline{\mathcal{E}_i^{\mu}(t-1)}, n_i(t-1) \ge L_i \right\} \right] & \le & 1\quad. 
        \end{array}
    \end{equation}
       \label{Friday 22}
     
\end{lemma}
\begin{lemma}
    We have
    \begin{equation}
        \begin{array}{lll}
               \sum\limits_{t=1}^{T} \mathbb{E} \left[\bm{1} \left\{i_t = i, \overline{\mathcal{E}_1^{\mu}(t-1)}, n_i(t-1) \ge L_i \right\} \right] & \le & 1\quad.
        \end{array}
    \end{equation}
    \label{Eric 99}
\end{lemma}

\begin{proof}[Proof of Lemma~\ref{Friday 2}]
     Let $\tau_s^{(i)}$ be the round when the $s$-th pull of the optimal arm $i$ occurs.  
     We have     \begin{equation}
               \begin{array}{ll}
                   &  \sum\limits_{t=1}^{T} \mathbb{E} \left[\bm{1} \left\{i_t = i,  \overline{\mathcal{E}_i^{\theta}(t)}, n_i(t-1) \ge L_i \right\} \right] \\
                  \le    & \sum\limits_{s=L_i}^{T} \mathbb{E} \left[\sum\limits_{t = \tau_s^{(i)}+1}^{\tau_{s+1}^{(i)}}\bm{1} \left\{i_t = i, \overline{\mathcal{E}_i^{\theta}(t)}\right\} \right] \\
                      =    & \sum\limits_{s=L_i}^{T} \mathbb{E} \left[\sum\limits_{t = \tau_s^{(i)}+1}^{\tau_{s+1}^{(i)}}\bm{1} \left\{i_t = i,  \theta_i(t) > \hat{\mu}_{i,n_i(t-1)}  + \sqrt{\frac{\left(5-\alpha \right)\ln^{1+\alpha}(T)}{n_i(t-1)}}  \right\} \right] \\
   \le  & \sum\limits_{s=L_i}^{T} \mathbb{E} \left[\bm{1} \left\{ \theta_i\left(\tau_{s+1}^{(i)}\right) > \hat{\mu}_{i,n_i\left(\tau_{s+1}^{(i)}-1\right)}  + \sqrt{\frac{\left(5-\alpha \right)\ln^{1+\alpha}(T)}{n_i\left(\tau_{s+1}^{(i)}-1\right)}}  \right\} \right] \\
 =  & \sum\limits_{s=L_i}^{T} \mathbb{E} \left[\bm{1} \left\{ \theta_i\left(\tau_{s+1}^{(i)}\right) > \hat{\mu}_{i,s}  + \sqrt{\frac{\left(5-\alpha \right)\ln^{1+\alpha}(T)}{s}}  \right\} \right] \\
 \le & \sum\limits_{s=L_i}^{T} \mathbb{E} \left[\bm{1} \left\{\mathop{\max}\limits_{h \in [\phi]} \theta_{i,s}^{h} > \hat{\mu}_{i,s}  + \sqrt{\frac{\left(5-\alpha \right)\ln^{1+\alpha}(T)}{s}}  \right\} \right] \\
 \le & \sum\limits_{s=L_i}^{T} \sum\limits_{h \in [\phi]} \mathbb{P}  \left\{ \theta_{i,s}^{h} > \hat{\mu}_{i,s}  + \sqrt{\frac{\left(5-\alpha \right) \ln^{1+\alpha}(T)}{s}}  \right\} \\
 = & \sum\limits_{s=L_i}^{T} \sum\limits_{h \in [\phi]} \mathbb{E} \left[\mathbb{P}  \left\{ \theta_{i,s}^{h} > \hat{\mu}_{i,s}  + \sqrt{\frac{\left(5-\alpha \right)\ln^{1+\alpha}(T)}{s}} \mid \hat{\mu}_{i,s} \right\} \right] \\
 \le & T \cdot \phi \cdot \frac{0.5}{T^{2.5-0.5\alpha}} \\
 &\\
 = & T \cdot \frac{2}{c_0} T^{0.5-0.5\alpha} \ln^{1.5-0.5\alpha}(T) \cdot \frac{0.5}{T^{2.5-0.5\alpha}}  \\
  &\\
             \le & \frac{1}{c_0}  \cdot \frac{\ln^{1.5}(T)}{T} \\
            &\\
       \le & \frac{1}{c_0} \quad,
               \end{array}
           \end{equation}
           which concludes the proof. \end{proof}
\begin{proof}[Proof of Lemma~\ref{Friday 22}]
Recall event $\mathcal{E}_i^{\mu}(t-1) = \left\{\left| \hat{\mu}_{i,n_i(t-1)} - \mu_i \right| \le \sqrt{\frac{2\ln(T)}{n_i(t-1)}} \right\}$. Let $\tau_s^{(i)}$ be the round when arm $i$ is pulled for the $s$-th time. We have
\begin{equation}
        \begin{array}{ll}
            &   \sum\limits_{t=1}^{T} \mathbb{E} \left[\bm{1} \left\{i_t = i, \overline{\mathcal{E}_i^{\mu}(t-1)}, n_i(t-1) \ge L_i \right\} \right] \\
             \le &  \sum\limits_{s=L_i}^{T} \mathbb{E} \left[\sum\limits_{t= \tau_s^{(i)}+1}^{\tau_{s+1}^{(i)}}\bm{1} \left\{i_t = i, \overline{\mathcal{E}_i^{\mu}(t-1)} \right\} \right] \\
            \le &  \sum\limits_{s=L_i}^{T} \mathbb{E} \left[\bm{1} \left\{ \overline{\mathcal{E}_i^{\mu}\left(\tau_{s+1}^{(i)}-1 \right)} \right\} \right] \\ 
            = &  \sum\limits_{s=L_i}^{T} \underbrace{ \mathbb{P} \left\{ \left|\hat{\mu}_{i, s} -\mu_i \right| > \sqrt{\frac{2 \ln \left(T  \right)}{s}} \right\}}_{\text{Hoeffding's inequality}}  \\
            \le &  \sum\limits_{s=L_i}^{T}  \frac{2}{T^4} \\
            \le & 1\quad,
        \end{array}
    \end{equation}
    which concludes the proof.
    
\end{proof}
             \begin{proof}[Proof of Lemma~\ref{Eric 99}]We have
    \begin{equation}
        \begin{array}{ll}
               & \sum\limits_{t=1}^{T} \mathbb{E} \left[\bm{1} \left\{i_t = i, \overline{\mathcal{E}_1^{\mu}(t-1)},n_i(t-1) \ge L_i \right\} \right] \\
               \le & \sum\limits_{t=1}^{T} \sum\limits_{s=1}^{T} \underbrace{\mathbb{P}\left\{\left| \hat{\mu}_{1,s}- \mu_1\right| > \sqrt{\frac{2\ln(T)}{s}} \right\}}_{\text{Hoeffding's inequality}} \\
               \le & \sum\limits_{t=1}^{T} \sum\limits_{s=1}^{T}  e^{-4\ln(T)} \\
               \le & 1\quad,
        \end{array}
    \end{equation}
    which concludes the proof.
        
    \end{proof}

 \section{Worst-case regret bounds proofs}\label{app: worst-case}
Let $\Delta_*:= \sqrt{\frac{K\ln^{1+\alpha} (T)}{T}}$ be the critical gap. The total regret of pulling any sub-optimal arm $i$ such that such that  $\Delta_i < \Delta_*$ is at most $T \cdot \Delta_* =  \sqrt{KT{\ln^{1+\alpha}(T)}}$.
Now, we consider all the remaining sub-optimal arms $i$ with  $\Delta_i > \Delta_*$. For  such $i$, the regret  is upper bounded by 
\[\Delta_i \mathbb{E} \left[n_i(T) \right] \leq O \left(\frac{\ln^{1+\alpha}(T)}{\Delta_i} \right)\]
which is decreasing in $\Delta_i \in (0,1]$. Therefore, for every such
$i$, the regret is bounded by
$O\left(\sqrt{\frac{T\ln^{1+\alpha}(T)}{K}}\right)$. Taking a sum over all sub-optimal arms concludes the proofs.

\end{document}